\documentclass[runningheads,orivec]{llncs}

\usepackage[T1]{fontenc}

\usepackage{graphicx}

\usepackage{mathtools}
\usepackage{amsfonts}       
\usepackage{amsmath}
\usepackage{amssymb}
\usepackage{amsthm}
\usepackage{stmaryrd}
\usepackage{accents}

\usepackage{etoolbox} 

\usepackage{microtype}      
\usepackage[prologue,dvipsnames]{xcolor}         

\usepackage{booktabs}       
\usepackage{array}
\usepackage{multirow} 
\usepackage{graphicx}
\usepackage[subrefformat=parens]{subcaption}

\usepackage{tikz}
\usetikzlibrary{arrows}
\usetikzlibrary{arrows.meta}
\usetikzlibrary{backgrounds}
\usetikzlibrary{calc}
\usetikzlibrary{decorations.pathreplacing}
\usetikzlibrary{intersections}
\usetikzlibrary{fit}
\usetikzlibrary{patterns}
\usetikzlibrary{positioning}
\usetikzlibrary{shapes}
\usetikzlibrary{shadows}
\usetikzlibrary{shapes.geometric}
\usetikzlibrary{tikzmark}

\usepackage{paralist}
\usepackage{eqparbox}
\usepackage[inline]{enumitem}

\usepackage{listings}
\usepackage{algorithm}
\usepackage{algorithmic}

\usepackage[square,sort,comma,numbers]{natbib}
\usepackage{url}      
\usepackage[bookmarksnumbered,unicode,colorlinks=true]{hyperref}
\usepackage[capitalize,noabbrev]{cleveref}
\usepackage{orcidlink}

\usepackage[misc,geometry]{ifsym}
\usepackage[disable]{todonotes}
\usepackage{marginnote}
\usepackage{etoc}
\usepackage{comment}
\usepackage{adjustbox}
\usepackage[shortcuts]{extdash} 

\newcommand*{\myhfill}[1]{%
    {
     \hfill#1%
     \parfillskip=0pt \finalhyphendemerits=0 \par}
}

\newcommand*{\myqedd}{\myhfill{$\blacksquare$}}

\newtheorem{condition}{Condition}

\usepackage{letltxmacro}
\LetLtxMacro\oldttfamily\ttfamily
\DeclareRobustCommand{\ttfamily}{\oldttfamily\csname ttsize\endcsname}
\newcommand{\setttsize}[1]{\def\ttsize{#1}}%
\setttsize{\small}

\definecolor[named]{ACMBlue}{cmyk}{1,0.1,0,0.1}
\definecolor[named]{ACMYellow}{cmyk}{0,0.16,1,0}
\definecolor[named]{ACMOrange}{cmyk}{0,0.42,1,0.01}
\definecolor[named]{ACMRed}{cmyk}{0,0.90,0.86,0}
\definecolor[named]{ACMLightBlue}{cmyk}{0.49,0.01,0,0}
\definecolor[named]{ACMGreen}{cmyk}{0.20,0,1,0.19}
\definecolor[named]{ACMPurple}{cmyk}{0.55,1,0,0.15}
\definecolor[named]{ACMDarkBlue}{cmyk}{1,0.58,0,0.21}
\definecolor{CColor}{rgb}{0.01,0.31,0.59}
\definecolor{GGray}{rgb}{0.80,0.90,1}
\definecolor{Shady}{rgb}{0.9,0.9,0.9}
\definecolor{kaistblue}{RGB}{20,135,200}
\definecolor{kaistdarkblue}{RGB}{0,65,145}
\definecolor{urbanablue}{RGB}{19,41,75}
\definecolor{urbanaorange}{RGB}{232,74,39}
\definecolor{drp}{rgb}{0.53,0.15,0.34}

\setlist[itemize,1]{label=$\bullet$}

\newcommand{\listingsttfamily}{\small\oldttfamily}
\lstdefinestyle{prettycode}{
  basicstyle=\listingsttfamily,
  keywordstyle=,
  morekeywords={prog, func, if, else, return},
  keepspaces=true,
  mathescape=true,
}
\hypersetup{
  colorlinks,
  linkcolor=blue, 
  citecolor=blue, 
  urlcolor =blue, 
  filecolor=blue, 
}

\crefname{assumption}{Assumption}{Assumptions}
\crefname{figure}{Fig{.}}{Figs{.}}
\crefname{table}{Table}{Tables}
\crefname{definition}{Definition}{Definitions}
\crefname{theorem}{Theorem}{Theorems}
\crefname{lemma}{Lemma}{Lemmas}
\crefname{proposition}{Proposition}{Propositions}
\crefname{corollary}{Corollary}{Corollaries}
\crefname{problem}{Problem}{Problems}
\crefname{example}{Example}{Examples}
\crefname{fact}{Fact}{Facts}
\crefname{conjecture}{Conjecture}{Conjectures}
\crefname{remark}{Remark}{Remarks}
\crefname{condition}{Condition}{Conditions}
\crefname{requirement}{Requirement}{Requirements}
\crefname{enumi}{}{}
\crefname{equation}{Eq{.}}{Eqs{.}}

\newcommand{\crefrangeconjunction}{--}
 \AtBeginDocument{%
    \crefname{section}{Section}{Sections}%
    \crefname{appendix}{Appendix}{Appendices}%
    \crefname{subsection}{Section}{Sections}%
    \crefname{figure}{Figure}{Figures}%
}

 \crefformat{section}{#2\S#1#3}
 \Crefformat{section}{#2\S#1#3}
 \crefmultiformat{section}{#2\S#1#3}{ and~#2\S#1#3}{, #2\S#1#3}{, and~#2\S#1#3}
 \Crefmultiformat{section}{#2\S#1#3}{ and~#2\S#1#3}{, #2\S#1#3}{, and~#2\S#1#3}
 \crefrangeformat{section}{#3\S#1#4\crefrangeconjunction{}#5\S#2#6}

 \crefformat{appendix}{#2\S#1#3}
 \Crefformat{appendix}{#2\S#1#3}
 \crefmultiformat{appendix}{#2\S#1#3}{ and~#2\S#1#3}{, #2\S#1#3}{, and~#2\S#1#3}
 \Crefmultiformat{appendix}{#2\S#1#3}{ and~#2\S#1#3}{, #2\S#1#3}{, and~#2\S#1#3}
 \crefrangeformat{appendix}{#3\S#1#4\crefrangeconjunction{}#5\S#2#6}

\newcommand{\mathboldcommand}[1]{\mathbb{#1}}

\newcommand{\bbF}{\mathboldcommand{F}}

\newcommand{\bbI}{\mathboldcommand{I}}

\newcommand{\bbN}{\mathboldcommand{N}}

\newcommand{\bbR}{\mathboldcommand{R}}

\newcommand{\bbZ}{\mathboldcommand{Z}}

\newcommand{\bfb}{\mathbf{b}}

\newcommand{\bfx}{\mathbf{x}}
\newcommand{\bfy}{\mathbf{y}}

\usepackage{euscript}
\newcommand{\mathcalcommand}[1]{\mathcal{#1}}

\newcommand{\mcB}{\mathcalcommand{B}}
\newcommand{\mcC}{\mathcalcommand{C}}
\newcommand{\mcD}{\mathcalcommand{D}}

\newcommand{\mcF}{\mathcalcommand{F}}
\newcommand{\mcG}{\mathcalcommand{G}}

\newcommand{\mcI}{\mathcalcommand{I}}

\newcommand{\mcK}{\mathcalcommand{K}}

\newcommand{\mcN}{\mathcalcommand{N}}

\newcommand{\mcR}{\mathcalcommand{R}}
\newcommand{\mcS}{\mathcalcommand{S}}
\newcommand{\mcT}{\mathcalcommand{T}}

\newcommand{\mcX}{\mathcalcommand{X}}

\DeclareMathAlphabet{\mathpzc}{T1}{pzc}{m}{it}

\allowdisplaybreaks

\newcommand{\tcr}[1]{#1} 

\newcommand{\cmt}[3]{{{\color{#3}{\bf [}{\bf #1:} {\it #2}{\bf ]}}}}

\newcommand{\wl}[1]{\cmt{WL}{#1}{purple}}

\newcommand*{\commentout}[1]{}

\let\svthefootnote\thefootnote
\newcommand\freefootnote[1]{
  \let\thefootnote\relax%
  \footnotetext{#1}%
  \let\thefootnote\svthefootnote%
}

\definecolor{lred}{rgb}{1.0, 0.5, 0.5}
\definecolor{lorange}{rgb}{1.00, 0.90, 0.20}
\definecolor{lgreen}{rgb}{0.35, 0.95, 0.35}
\definecolor{lime}{rgb}{0.9, 1.0, 0.6}
\definecolor{lblue}{rgb}{1.0, 0.85, 0.75}

\makeatletter

\newcommand*\wthelper[2]{%
        \hbox{\dimen@\accentfontxheight#1%
                \accentfontxheight#11.1\dimen@
                $\m@th#1\widetilde{#2}$%
                \accentfontxheight#1\dimen@
        }%
}
\newcommand*\accentfontxheight[1]{%
        \fontdimen5\ifx#1\displaystyle
                \textfont
        \else\ifx#1\textstyle
                \textfont
        \else\ifx#1\scriptstyle
                \scriptfont
        \else
                \scriptscriptfont
        \fi\fi\fi3
}

\newcommand*\whhelper[2]{%
        \hbox{\dimen@\accentfontxheight#1%
                \accentfontxheight#11.2\dimen@
                $\m@th#1\widehat{#2}$%
                \accentfontxheight#1\dimen@
        }%
}
\makeatother

\makeatletter
\newcommand{\oset}[3][0ex]{%
  \mathrel{\mathop{#3}\limits^{
    \vbox to#1{\kern-3\ex@
    \hbox{$\scriptstyle#2$}\vss}}}}
\makeatother

\newcommand{\newhat}{\scalebox{1.3}[.90]{\trimbox{0pt 1.2ex}{\normalfont\textasciicircum}}}
\renewcommand{\widehat}[1]{\accentset{\newhat}{#1}}

\newcommand*{\defeq}{\coloneq}

\newcommand*{\relu}{\mathrm{ReLU}}

\newcommand*{\lrelu}{\mathrm{LeakyReLU}}

\newcommand*{\SoftPlus}{\mathrm{softplus}}
\newcommand*{\mish}{\mathrm{Mish}}
\newcommand*{\gelu}{\mathrm{GELU}}

\newcommand*{\Sigmoid}{\mathrm{sigmoid}}

\newcommand*{\elu}{\mathrm{ELU}}

\newcommand*{\class}{\mathrm{class}}

\DeclareMathOperator*{\argmax}{arg\,max}
\DeclareMathOperator*{\argmin}{arg\,min}
\newcommand*{\fpq}{\mathbb{F}}
\newcommand*{\efpq}{{\overline{\mathbb{F}}}}
\newcommand*{\fmin}{\omega}
\newcommand*{\fmax}{\Omega}
\newcommand*{\feps}{\varepsilon}

\newcommand*{\round}[1]{ \mathrm{rnd}{\ifstrempty{#1} {} {({#1})}} }
\newcommand*{\aff}{{\mathrm{aff}}}

\newcommand*{\ceilZ}[1]{ \left\lceil{#1}\right\rceil_{\mathbb{Z}} }
\newcommand*{\floorZ}[1]{ \left\lfloor{#1}\right\rfloor_{\mathbb{Z}} }

\makeatletter
\def\moverlay{\mathpalette\mov@rlay}
\def\mov@rlay#1#2{\leavevmode\vtop{%
   \baselineskip\z@skip \lineskiplimit-\maxdimen
   \ialign{\hfil$\m@th#1##$\hfil\cr#2\crcr}}}
\newcommand{\charfusion}[3][\mathord]{
    #1{\ifx#1\mathop\vphantom{#2}\fi
        \mathpalette\mov@rlay{#2\cr#3}
      }
    \ifx#1\mathop\expandafter\displaylimits\fi}
\makeatother

\newcommand{\lrp}[1]{\left({#1}\right)}

\newcommand{\emin}{\mathfrak{e}_{\min}}
\newcommand{\emax}{\mathfrak{e}_{\max}}

\newcommand*{\mant}[1]{\mathfrak{s}_{#1}}
\newcommand*{\expo}[1]{\mathfrak{e}_{#1}}
\newcommand*{\mantd}[2]{\mathfrak{s}_{#1,#2}}
\newcommand{\nan}{\bot} 

\newcommand*{\intv}[1]{\langle {#1} \rangle}

\newcommand*{\conc}[1]{\gamma\left(#1\right)}

\newcommand{\mbit}{M}
\newcommand{\ebit}{E}
\newcommand{\lips}{\lambda}

\newcommand{\Mod}[1]{\ (\mathrm{mod}\ #1)}
\DeclarePairedDelimiter\set{\lbrace}{\rbrace}
\DeclarePairedDelimiter\abs{|}{|}
\DeclarePairedDelimiter\paren{(}{)}

\DeclarePairedDelimiter\ceil{\lceil}{\rceil}
\DeclarePairedDelimiter\floor{\lfloor}{\rfloor}

\makeatletter
\newcommand*{\sumcirc}{%
  \DOTSB
  \mathop{
    \mathchoice
      {\rlap{\kern.1em\raisebox{-.25em}{\scalebox{2.0}{$\circ$}}}{\sum}}
      {\vcenter{\rlap{\kern.0em\raisebox{-.25em}{\scalebox{2.0}{$\scriptstyle\circ$}}}}{\sum}}
      {\sum}{\sum}
  }\slimits@
}
\renewcommand*{\bigoplus}{\sumcirc}
\makeatother

\usepackage{color}

\begin{document}

\title{Floating-Point Neural Networks Are\\ Provably Robust Universal Approximators}
\titlerunning{Floating-Point Networks Are Provably Robust Universal Approximators}

\def\corresp{\unskip$^{\mbox{\tiny(\Letter)}}$}

\author{%
Geonho Hwang\inst{1}${}^{\star}$\orcidlink{0000-0001-7137-426X}
\and
Wonyeol Lee\inst{2}${}^{\star}$\orcidlink{0000-0003-0301-0872}
\and
Yeachan Park\inst{3}\orcidlink{0000-0002-4211-6226}
\and
\\
Sejun Park\inst{4}\corresp\orcidlink{0000-0003-1580-5664}
\and
Feras Saad\inst{5}\orcidlink{0000-0002-0505-795X}
\protect\freefootnote{${}^{\star}$Equal contribution.~ {\corresp}Corresponding author.}
\protect\freefootnote{This manuscript is the full version of a conference paper appearing in CAV 2025.}
}
\authorrunning{G. Hwang et al.}

\institute{%
GIST, Gwangju, Republic of Korea
\\ \email{hgh2134@gist.ac.kr}
\and
POSTECH, Pohang, Republic of Korea
\\ \email{wonyeol.lee@postech.ac.kr}
\and
Sejong University, Seoul, Republic of Korea
\\ \email{ychpark@sejong.ac.kr}
\and
Korea University, Seoul, Republic of Korea
\\ \email{sejun.park000@gmail.com}
\\
\and
Carnegie Mellon University, Pittsburgh, PA, USA
\\ \email{fsaad@cmu.edu}
}

\maketitle              

\begin{abstract}
The classical universal approximation (UA) theorem for neural networks
establishes mild conditions under which a feedforward neural network
can approximate a continuous function $f$ with arbitrary accuracy.
A recent result shows that neural networks also enjoy a more general
\textit{interval} universal approximation (IUA) theorem, in the sense that
the abstract interpretation semantics of the network using the
interval domain can approximate the direct image map of $f$
(i.e., the result of applying $f$ to a set of inputs)
with arbitrary accuracy.
These theorems, however, rest on the unrealistic assumption that the neural
network computes over infinitely precise real numbers, whereas their
software implementations in practice compute over finite-precision
floating-point numbers.
An open question is \mbox{whether the IUA theorem still holds in the floating-point
setting.}

This paper introduces the first IUA theorem for \textit{floating-point} neural networks
that proves their remarkable ability to \textit{perfectly capture}
the direct image map of any rounded target function $f$,
showing no limits exist on their expressiveness.
Our IUA theorem in the floating-point setting exhibits material differences
from the real-valued setting, which reflects the fundamental distinctions
between these two computational models.
This theorem also implies surprising corollaries, which include
\begin{enumerate*}[label=(\roman*)]
\item the existence of \textit{provably robust} floating-point neural networks; and
\item the \textit{computational completeness} of the class of straight-line programs
that use only floating-point additions and multiplications for the class of
all floating-point programs that halt.
\end{enumerate*}
\keywords{Neural networks \and Robust machine learning \and Floating point \and Universal approximation \and Abstract interpretation.}
\end{abstract}

\setcounter{tocdepth}{2}
\AddToHook{cmd/appendix/before}{\gdef\theHsection{\Alph{section}}}


\section{Introduction}
\label{sec:intro}

\paragraph{\bf Background.}

Despite the remarkable success of neural networks on diverse tasks, these
models often lack \emph{robustness} and are subject to adversarial attacks.
Slight perturbations to the network inputs can cause the network to
produce significantly different outputs~\citep{SzegedyZS13,
GoodfellowSS14}, raising serious concerns in safety-critical domains such
as healthcare~\citep{FinlaysonBI19}, cybersecurity~\citep{Rosenberg2021},
{and autonomous driving~\citep{EykholtEF18}.}

These issues have brought about significant advances in new algorithms for
\textit{robustness verification}~\citep{Katz2017,Albarghouth2021,Liu2021},
which prove the robustness of a given network;
and \textit{robust training}~\citep{Raghunathan2018,Wong2018,Mirman2018,GowalDSBQUAMK19},
which train a network to be {provably} robust.
%
But despite these advances, provably robust networks do not yet achieve state-of-the-art accuracy~\citep{Li2023}.
For example, on the CIFAR-10 image classification benchmark,
non-robust networks achieve over 99\% accuracy, whereas
the best provably robust networks achieve less than 63\%~\citep{Li2023b}.
This performance gap has prompted researchers to explore whether there exists
fundamental limits on the \emph{expressiveness} of provably robust
networks that restrict their accuracy~\citep{baader24thesis}.

Surprisingly, it has been proven that no such fundamental limit exists.
Informally, for any continuous function $f: \bbR^d \to \bbR$ and compact set
$\mcK \subset \bbR^d$, there exists a neural network
$g: \bbR^d \to \bbR$ whose robustness properties are ``sufficiently close'' to
those of $f$ over $\mcK$ and easily provable using abstract interpretation~\citep{Cousot1977}
over the interval domain.
This result, known as the \emph{interval universal approximation} (IUA)
theorem~\citep{baader20,wang2022interval}, generalizes the classical universal approximation (UA)
theorem~\citep{cybenko89, hornik89} from pointwise-values to intervals, 
and confirms that provably robust networks
{do not suffer from a fundamental loss of expressive power.}

\paragraph{\bf Key challenges.}

The IUA theorem in \cite{baader20,wang2022interval}
overlooks a critical aspect of real-world computation, which is
the use of \emph{floating-point arithmetic} instead of real arithmetic.
It assumes that neural networks and interval analyses
operate on arbitrary real numbers with exact operations.
In reality, numerical implementations of neural networks use {floating-point numbers} and
operations~\citep[\S4.1]{GoodfellowBC16}, sometimes with extremely
low-precision to speed-up performance~\citep{Hubara2017,dettmers2023}.
This discrepancy means that the existing IUA theorem does not directly
apply to neural networks that are implemented in software and {actually used in practice.}


\begin{figure}[t]

\makeatletter
\define@key{Template}{curve}{\def\TemplateCurve{#1}}
\define@key{Template}{curveF}{\def\TemplateCurveF{#1}}
\define@key{Template}{curveS}{\def\TemplateCurveS{#1}}
\def\TemplateDefaults{%
  \setkeys{Template}{%
    curve=black,
    curveF=$f$,
    curveS=solid,
  }}
\makeatother

\newenvironment{Template}[1][]{
  \TemplateDefaults
  \setkeys{Template}{#1}
  \begin{tikzpicture}
  \def\xscale{5}
  \def\yscale{2.5}
  \def\yoff{0.15}

  \tikzset{point/.style={circle,inner sep=1pt}}

  \draw[-latex, line width=.25mm] (0,0) -- (1.1*\xscale,0);
  \draw[-latex, line width=.25mm] (-\yoff,0) -- (-\yoff,1.1*\yscale);

  \foreach \i in {0, 1/8, 1/4, 1/2, 1} {
    \draw[thick] ([xshift=-0.1cm]-\yoff,\i*\yscale) -- ([xshift=+0.1cm]-\yoff,\i*\yscale);
    \draw[thick] ([yshift=-0.15cm]\i*\xscale,0) -- ([yshift=+0.15cm]\i*\xscale,0);
  }

  \foreach \i in {1/32, 2/32, 3/32, 5/32, 6/32, 7/32,  5/16, 6/16, 7/16, 5/8, 6/8, 7/8} {
      \draw[thick] ([xshift=-0.05cm]-\yoff,\i*\yscale) -- ([xshift=+0.05cm]-\yoff,\i*\yscale);
      \draw[thick] (\i*\xscale,-0.1) -- (\i*\xscale, +0.1);
  }

  \node[name=p1,at={(0,\yscale)},coordinate]{};
  \node[name=p2,at={(0.5*\xscale,0.5*\yscale)},coordinate]{};
  \node[name=p3,at={(0.9*\xscale,0.5*\yscale)},coordinate]{};
  \node[name=p11,at={(0.25*\xscale,1.1*\yscale)},coordinate]{};
  \node[name=p12,at={(0.3*\xscale,-0.17*\yscale)},coordinate]{};
  \node[name=p21,at={(0.7*\xscale,1.25*\yscale)},coordinate]{};
  \node[name=p22,at={(0.7*\xscale,0.5*\yscale)},coordinate]{};

  \draw[name path=Seg1,thick,draw=\TemplateCurve,\TemplateCurveS]
    (p1)
    .. controls (p11) and (p12) ..
    (p2)
    .. controls (p21) and (p22) ..
    (p3)
    node[pos=1,below]{\TemplateCurveF};

  \def\xL{7/32 * \xscale}
  \def\xR{5/16 * \xscale}

  \draw[thick, decoration={brace, raise=0.15cm}, decorate]
    (\xL,0) -- (\xR,0)
    node[anchor=south,pos=0.5,yshift=.2cm]{$\mcB$};

  \draw[name path=xL,draw=none] (\xL, 0) -- (\xL, \yscale);
  \draw[name path=xR,draw=none] (\xR, 0) -- (\xR, \yscale);
}{\end{tikzpicture}}

\begin{subfigure}{.48\linewidth}
\centering
\begin{Template}

  \path[name intersections={of=Seg1 and xL,by=iL}] node[at=(iL),point,fill=black]{};
  \path[name intersections={of=Seg1 and xR,by=iR}] node[at=(iR),point,fill=black]{};

  \draw[color=black,ultra thick] (\xR,0) -- (\xL,0)
    node[pos=0,point,fill=black]{} node[pos=1,point,fill=black]{};

  \node[name=yL, at={(-\yoff,0 |- iL)},point,fill=blue]{};
  \node[name=yR, at={(-\yoff,0 |- iR)},point,fill=blue]{};
  \draw[color=blue,ultra thick] (yL.center) -- (yR.center);

  \node[name=yLa, xshift=-\yoff cm, yshift=-.25cm, point,fill=red, at={(-\yoff,0 |- iL)}]{};
  \node[name=yRa, xshift=-\yoff cm, yshift=-0.175cm, point,fill=red, at={(-\yoff,0 |- iR)}]{};
  \draw[color=red,ultra thick] (yLa.center) -- (yRa.center);

  \foreach \n in {yR, yL}{
  \draw[decorate, decoration={brace, amplitude=.075cm, raise=0.1cm}]
    (\n.center) -- (\n a.center -| \n.center)
    node[pos=0.5,right,font=\scriptsize, xshift=0.1cm]{${\le} \delta$};
  }

  \node[at={(\xscale,0.975*\yscale)},anchor=north east,draw=none, inner sep=0pt] {
    \begin{tabular}{ll}
    \tikz{\draw[color=blue,thick](0,0)--(0.25,0) node[pos=1,right]{$f(\mcB)$}} \\
    \tikz{\draw[color=red,thick](0,0)--(0.25,0) node[pos=1,right]{$\nu^{\sharp}(\mcB)$}}
    \end{tabular}
  };
\end{Template}
\captionsetup{skip=0pt}
\caption{$f : \bbR^d \to \bbR$ is a continuous target function; $\nu : \bbR^d \to \bbR$ is a neural network.}
\label{fig:iua-statement-real}
\end{subfigure}
\hfill
%
\begin{subfigure}{.48\linewidth}
\centering
\begin{Template}[curve=black!30!white, curveF=$\widehat{f}$, curveS=densely dashed]

  \begin{scope}[on background layer]
  \foreach \x [count = \i] in {
      0,    1/32, 2/32, 3/32,
      1/8,  5/32, 6/32, 7/32,
      1/4,  5/16, 6/16, 7/16,
      1/2,  5/8,  6/8,  7/8,
      1} {
      \draw[name path global/.expanded=guidex-\i,draw=gray!20!white] (\x*\xscale, 0) -- (\x*\xscale, \yscale);
      \draw[name path global/.expanded=guidey-\i,draw=gray!20!white] (-\yoff,\x*\yscale) -- (\xscale, \x*\yscale);
    }
   \foreach \x [count = \i] in {
      0,    1/32, 2/32, 3/32,
      1/8,  5/32, 6/32, 7/32,
      1/4,  5/16, 6/16, 7/16,
      1/2,  5/8,  6/8,  7/8,
      1} {
   }
  \end{scope}

  \begin{scope}
    \node[point,fill=black,at={(0/32*\xscale, 8/8*\yscale)}] {};
    \node[point,fill=black,at={(1/32*\xscale, 8/8*\yscale)}] {};
    \node[point,fill=black,at={(2/32*\xscale, 8/8*\yscale)}] {};
    \node[point,fill=black,at={(3/32*\xscale, 8/8*\yscale)}] {};
    \node[point,fill=black,at={(4/32*\xscale, 7/8*\yscale)}] {};
    \node[point,fill=black,at={(5/32*\xscale, 7/8*\yscale)}] {};
    \node[point,fill=black,at={(6/32*\xscale, 6/8*\yscale)}] {};
    \node[point,fill=black,at={(7/32*\xscale, 6/8*\yscale)}] {};

    \node[point,fill=black,at={(4/16*\xscale, 5/8*\yscale)}] {};
    \node[point,fill=black,at={(5/16*\xscale, 7/16*\yscale)}] {};
    \node[point,fill=black,at={(6/16*\xscale, 5/16*\yscale)}] {};
    \node[point,fill=black,at={(7/16*\xscale, 5/16*\yscale)}] {};

    \node[point,fill=black,at={(4/8*\xscale, 4/8*\yscale)}] {};
    \node[point,fill=black,at={(5/8*\xscale, 7/8*\yscale)}] {};
    \node[point,fill=black,at={(6/8*\xscale, 5/8*\yscale)}] {};
    \node[point,fill=black,at={(7/8*\xscale, 4/8*\yscale)}] {};
  \end{scope}

  \foreach \x in {\xL, 1/4 * \xscale, \xR} {
    \node[point,fill=black,at={(\x,0)}]{};
  }

  \node[name=yLa,xshift=-1.25*\yoff cm, at={(-\yoff,0 7/16 * \yscale)}, point,fill=red]{};
  \node[name=yLb,xshift=-1.25*\yoff cm, at={(-\yoff,0 6/8 * \yscale)}, point,fill=red]{};
  \draw[color=red,ultra thick,densely dotted] (yLa.center) -- (yLb.center);
  \foreach \shift/\color in {0/blue}{
    \node[xshift=-\shift cm, at={(-\yoff,0 7/16 * \yscale)}, point,fill=\color]{};
    \node[xshift=-\shift cm, at={(-\yoff,0 5/8 * \yscale)}, point,fill=\color]{};
    \node[xshift=-\shift cm, at={(-\yoff,0 6/8 * \yscale)}, point,fill=\color]{};
  }

  \node[at={(\xscale,0.975*\yscale)},anchor=north east,draw=none, inner sep=0pt] {
    \begin{tabular}{ll}
    \tikz{\draw[color=blue,thick](0,0)--(0.25,0) node[pos=1,right]{$\widehat{f}(\mcB)$}} \\
    \tikz{\draw[color=red,thick](0,0)--(0.25,0) node[pos=1,right]{$\nu^{\sharp}(\mcB)$}}
    \end{tabular}
  };
\end{Template}
\captionsetup{skip=0pt}
\caption{$\widehat{f} : \efpq{}^d \to \efpq$ is a rounded target function; $\nu : \efpq{}^d \to \efpq$ is a neural network.}
\label{fig:iua-statement-float}
\end{subfigure}

\caption{%
  Illustration and comparison of the IUA theorems.
  \subref{fig:iua-statement-real}
  In the real-valued setting, the neural network abstract interpretation $\nu^{\sharp}$ forms a $\delta$-approximation
  to the image map of $f$.
  \subref{fig:iua-statement-float}
  In the floating-point setting,
  $\nu^{\sharp}$ exactly computes the upper and lower points of the image map of $f$:
  $\nu^{\sharp}(\mcB) = [\min \protect\widehat{f}(\mcB), \max \protect\widehat{f}(\mcB)] \cap \efpq$.
}
\label{fig:iua-statement}

\end{figure}
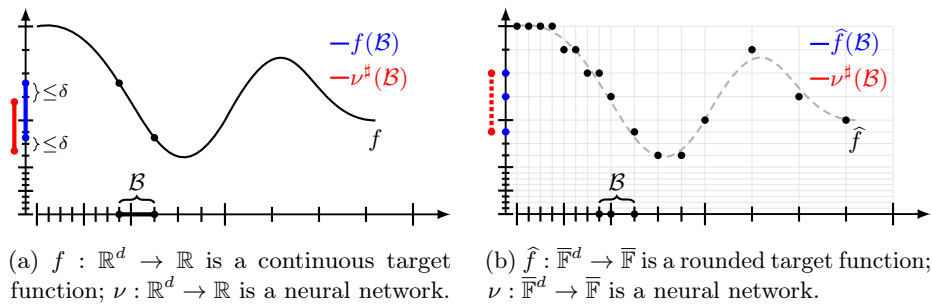

To our knowledge, no prior work has studied the robustness and
expressiveness properties of floating-point neural networks or established
an IUA theorem for them.
The unique complexities of floating-point arithmetic introduce daunting
challenges to any such theoretical study.
For example, floating-point numbers are discretized and bounded, and their
operations have rounding errors that become infinite in cases of overflow.
Whereas the IUA proof over reals requires very large real numbers for
network weights or intermediate computations, these values cannot be
represented as floats.
Naively rounding reals to floats causes approximation errors that
invalidate many steps of the IUA proofs in \citep{baader20,wang2022interval}.

\paragraph{\bf This work.}

We formally study the IUA theorem over floating point,
as a step toward bridging the theory and practice of provably robust neural networks.

We first formulate a floating-point analog of the IUA theorem, considering the details of floating point.
Let $f : \bbR^d \to \bbR$ be a target function to approximate.
Since all floating-point neural networks are functions between floating-point values,
they can at-best approximate the rounded version $\widehat{f} : \efpq{}^d \to \efpq$
of $f$ over floats, where $\efpq{}$ denotes the set of all floats.
The floating-point version of the IUA theorem asks the following: is there a
floating-point neural network $\nu: \efpq{}^d \to \efpq$ whose \emph{interval semantics} is
arbitrarily close to the \emph{direct image map} of the rounded target
$\widehat{f}$ over $[-1,1]^d$?
More formally, this property means that
for any $\delta > 0$, there exists a neural network $\nu$ such that
for all boxes $\mcB \subseteq [-1, 1]^d \cap \efpq{}^d$,
\begin{align}
  & \abs[\big]{\min \nu^\sharp(\mcB) - \min \widehat{f}(\mcB)} \leq \delta,
  &
  & \abs[\big]{\max \nu^\sharp(\mcB) - \max \widehat{f}(\mcB)} \leq \delta.
  \label{eq:approximate-direct-image}
\end{align}
In \cref{eq:approximate-direct-image},
$\nu^\sharp(\mcB)$ is the result of abstract interpretation
of $\mcB$ under $\nu$ (using the interval domain),
and ${\widehat{f}}(\mcB) \defeq \set{\widehat{f}(\bfx) \mid \bfx \in \mcB } \subset \bbR$
is the image of $\mcB$ under $\widehat{f}$.

We prove that the IUA theorem holds for floating-point networks,
despite all their numerical complexities.
In particular, we show that for \emph{any} target function $f$ and a
\emph{large} class of activation functions $\sigma$, including most
practical ones (e.g., ReLU, GELU, sigmoid), it is possible to find a
floating-point network $\nu$ with $\sigma$ whose interval semantics
\emph{exactly} matches the direct image map of the rounded target
$\widehat{f}$ over $[-1,1]^d \cap \efpq{}^d$ (\cref{fig:iua-statement}).
This result implies that no fundamental limit exists on the expressiveness of
provably robust floating-point neural networks.

Our result is considerably different from the previous IUA theorem over the
reals in three key aspects.
The previous theorem considers continuous target functions; requires
a restricted class of so-called squashable activation functions;
and finds networks that are arbitrarily close to target functions.
In contrast, our result considers arbitrary target functions;
allows almost all activation functions used in practice;
and find networks that are precisely equal to (rounded) target functions.
Our IUA theorem even holds for the \emph{identity} activation function,
which is not the case for the traditional IUA or UA theorems over real numbers,
because any network that uses the identity activation is affine over the reals.

As a corollary of our main theorem,
we prove the following existence of provably robust floating-point neural networks:
given an ideal floating-point classifier $\widehat{f}$ (not necessarily a neural network)
that is robust (not necessarily provably robust),
we can find a floating-point neural network $\nu$
that is \emph{identical} to $\widehat{f}$ and is \emph{provably} robust with interval analysis.
We also prove a nontrivial result about ``floating-point completeness'',
as an unexpected byproduct of the main theorem.
Specifically, we show that the class of straight-line floating-point
programs that use only floating-point $+$ and $\times$ operations is
\emph{floating-point interval-complete}: it can simulate \emph{any}
terminating floating-point program that takes finite floats as input and
returns arbitrary floats as output.
The same statement holds under the interval semantics.
To our knowledge, no prior work has identified such a small yet powerful
class of floating-point programs, suggesting that this corollary is of
\mbox{significant independent interest to the extensive floating-point literature.}

\paragraph{\bf Contributions.}

This article makes the following contributions:

\begin{itemize}
\item
  We formalize a \emph{floating-point} analog of the \emph{interval universal approximation} (IUA) theorem,
  to bridge the theory and practice of \emph{provably robust} neural networks
  (\cref{sec:prelim}, \cref{sec:iua}).
  It asks if there is a floating-point network
  whose interval semantics is close to the direct image map
  of a given target function.

\item
  We prove the floating-point version of the IUA theorem does hold,
  for \emph{all} target functions and a \emph{broad} class of activation functions
  that includes most of the activations used in practice
  (\cref{sec:iua-conditions}, \cref{sec:iua-main-result}, \cref{sec:iua-proof}).
  This shows no fundamental limit exists on the expressiveness of provably robust networks over floats.

\item
  We rigorously analyze the essential differences
  between the previous IUA theorem over reals and our IUA theorem over floats (\cref{sec:comparison}).
  Unlike real-valued networks, floating-point networks can \emph{perfectly} capture
  the behavior of \emph{any} rounded target function,
  even with the \emph{identity} activation function.

\item
  We prove that if there exists an ideal robust floating-point classifier,
  then one can always find a \emph{provably} robust floating-point network
  that makes \emph{exactly} the same prediction as the classifier
  (\cref{sec:provable-robustness}).

\item
  We prove that the set of straight-line floating-point programs with only
  $(+,\times)$ is \emph{floating-point interval-complete}: it can simulate
  \emph{any} terminating floating-point programs that take finite inputs
  and return finite/infinite outputs, under the usual floating-point
  semantics and interval semantics (\cref{sec:completeness}).
\end{itemize}


\section{Preliminaries}
\label{sec:prelim}

This section introduces floating-point arithmetic (\cref{sec:floating-point}),
neural networks that compute over floating-point numbers (\cref{sec:neural-networks}),
and interval analysis for neural networks (\cref{sec:interval-semantics}).
Throughout the paper, we define $\bbN$ to be the set of positive integers
and let $[n] \defeq \set{1,\dots,n}$ for each $n \in \bbN$.

\subsection{Floating Point}
\label{sec:floating-point}

\paragraph{\bf Floating-point numbers.}
Let $E, M \in \bbN$.
The set of \emph{finite} floating-point numbers with $E$-bit exponent and \tcr{$(M+1)$-bit significand} 
is typically defined by
\begin{align}
\label{eq:fpq}
\fpq^E_M
\defeq \set[\big]{
  (-1)^b \times (s_0.s_1 \ldots s_\mbit)_2 \times 2^e
  \,\big|\, b, s_i \in\{0, 1\}, e \in \{\emin, ..., \emax\}
},
\end{align}
where $\emin \defeq -2^{\ebit-1}+2$ and $\emax \defeq 2^{\ebit-1}-1$~\citep{Muller2018}.
The set of \emph{all} floating-point numbers, including non-finite ones, is then defined by
$\efpq{}^E_M \defeq \fpq^E_M \cup \set{-\infty, +\infty, \nan}$,
where $\nan$ denotes NaN (i.e., not-a-number).
For brevity, we call a floating-point number simply a \emph{float},
and write $\fpq^E_M$ and $\efpq{}^E_M$ simply as $\fpq$ and $\efpq$.
In this paper, we assume {$\ebit \ge 5$} 
and {$2^{\ebit-1} \geq M \geq 3$},
which hold for nearly all practical floating-point formats,
including bfloat16~\citep{tensorflow16} and all the formats defined in the IEEE-754 standard~\citep{ieee754}
such as float16, float32, and float64.

We introduce several notations and terms related to finite floats.
First, we define three key constants:
the \emph{smallest} positive float $\fmin \defeq 2^{\emin-\mbit}$,
the \emph{largest} positive float $\fmax \defeq 2^{\emax} (2-2^{-\mbit})$,
and the \emph{machine epsilon} $\feps\defeq2^{-\mbit-1}$.
%
Next, consider a finite float $x\in\fpq$.
We call $x$ a \emph{subnormal} number if ${0< \abs{x} < 2^{\emin}}$,
and a \emph{normal} number otherwise.
The \emph{exponent} and \emph{significand} of $x$ are defined by
$\expo{x} \defeq \max\set{\floor{\log_2 \abs{x}}, \emin} \in [\emin, \emax]$ and
$\mant{x} \defeq \abs{x} / 2^{\expo{x}} \in [0, 2)$.
We use $\mantd{x}{0},\ldots,\mantd{x}{\mbit}$
to denote the binary expansion of $\mant{x}$, i.e.,
$(\mantd{x}{0}.\,\mantd{x}{1}\dots\mantd{x}{\mbit})_2 = \mant{x}$
with $\mantd{x}{i} \in \set{0,1}$.
The \emph{predecessor} and \emph{successor} of $x$ in $\efpq$ are written as
$x^- \defeq \max\set{y \in \efpq \setminus \set{\nan} \mid x > y}$ and
$x^+ \defeq \min\set{y \in \efpq \setminus \set{\nan} \mid x < y}$.

\paragraph{\bf Floating-point operations.}
We define the \emph{rounding function}
$\round{} : \bbR \cup \{-\infty, +\infty\} \to \efpq$ as follows:
$\round{x} \defeq -\infty$ if $x \in [-\infty, -\fmax-c]$,
$\round{x} \defeq \argmin_{y\in\fpq}|y-x|$ if $x\in(-\fmax-c, \fmax+c)$, and
$\round{x} \defeq +\infty$ if $x \in [\fmax+c, +\infty]$,
where $c \defeq 2^{\emax}\feps$ and
$\argmin$ breaks ties by choosing a float $y$ with $\mantd{y}{\mbit} = 0$.
This function corresponds to the rounding mode
``round to nearest (ties to even)'',
which is the default rounding mode in the IEEE-754 standard~\citep{ieee754}.

The floating-point \emph{arithmetic operations}
$\oplus, \ominus, \otimes : \efpq \times \efpq \to \efpq$
are defined via the rounding function:
for finite floats $x, y \in \fpq$,
$x \oplus  y \defeq \round{x+y}$,
$x \ominus y \defeq \round{x-y}$, and
$x \otimes y \defeq \round{x \times y}$.
We omit the definition for non-finite operands
because they are unimportant in this paper,
except that $x \oplus 0 = x \ominus 0 = x$ for all $x \in \set{-\infty, +\infty}$.
{For the full definition, refer to the IEEE-754 standard~\citep{ieee754}.}

We introduce two more floating-point operations: $\aff_{W, \bfb}$ and $\round{f}$.
First, we define the floating-point \emph{affine transformation}:
for a matrix $W = (w_{i,j})_{{i\in [m],j\in [n]}}\in {\fpq{}^{m \times n}}$
and a vector $\bfb = \lrp{b_1,\dots, b_{m}}\in {\fpq{}^{m}}$,
$\aff_{W,\bfb} : \efpq{}^n \to \efpq{}^m$ is defined by
\begin{align}
  \label{eq:aff}
  \aff_{W,\bfb}(x_1, \allowbreak \ldots, \allowbreak x_n) \defeq \!
  \paren[\Bigg]{\!
    \paren[\Bigg]{
      \bigoplus_{j=1}^{n} x_j\otimes w_{1,j} \!
    } \oplus b_1,\,
    \allowbreak
    \ldots, \,
    \allowbreak
    \paren[\Bigg]{
      \bigoplus_{j=1}^{n} x_j\otimes w_{m,j} \!
    } \oplus b_m \!
  }.
\end{align}
Here, $\bigoplus$ denotes the floating-point summation defined in the left-associative way:
\(
\bigoplus_{i=1}^{n} y_i \defeq (\cdots((y_1 \oplus y_2) \oplus y_3) \cdots) \oplus y_n,
\)
where the order of $\oplus$ is important because $\oplus$ is not associative.
Next, we define the \emph{correctly rounded version} of a real-valued function.
For $f:\bbR \to \bbR$, the function $\round{f}:\efpq \to \efpq$ {is defined by}
\begin{align}
  \label{eq:round-f}
  \round{f}(x) \defeq
  \begin{cases}
    \round{f(x)}
    & \text{if } x \in (-\infty, +\infty)
    \\
    \mathrm{rnd}\left({\displaystyle \lim_{t \to x} f(t)}\right)
    & \text{if }
    x \in \set{-\infty, +\infty} \land 
    {\displaystyle \lim_{t \to x} f(t) \in \bbR \cup \set{-\infty, +\infty}}
    \\[-1pt]
    \nan
    & \text{otherwise}.
  \end{cases}
  \raisetag{12pt}
\end{align}

\subsection{Neural Networks}
\label{sec:neural-networks}


\begin{figure}[t]
\centering
\begin{adjustbox}{max width=\linewidth}
\begin{tikzpicture}[thick]
\node[name=x1,draw,minimum width=6.1cm,inner xsep=0pt]{$\bfx$};
\node[name=y1,draw,minimum width=6.1cm,inner xsep=0pt,above = .4cm of x1,fill=red!15!white]{
  $\begin{aligned}
  \bfy_{L-1} &\defeq (\tilde{\sigma}_{d_{L-1}} \circ \aff_{W_{L-1},\bfb_{L-1}}) (\bfy_{L-2})\\[-10pt]
  &\;\;\vdots\\[-5pt]
  \bfy_{1} &\defeq (\tilde{\sigma}_{d_{1}} \circ \aff_{W_1,\bfb_1}) (\bfx)
  \end{aligned}$
};
\node[name=z1,draw,minimum width=6.1cm,inner xsep=0pt,above = 2.75cm of y1]{
  $\begin{aligned}
  \bfy &\defeq \textcolor{black}{\aff_{W_{L},\bfb_{L}}(\bfy_{L-1}) \equiv \bfy_{L-1}}
  \end{aligned}$
};
\draw[-latex] (x1) -- (y1);
\draw[-latex] (y1) -- (z1);

\node[name=x2,draw,minimum width=6.1cm,inner xsep=0pt, right=3.25 of x1.center]{$\bfx$};
\node[name=y2,draw,minimum width=6.1cm,inner xsep=0pt,above = 3.3cm of x2.north, anchor=north,fill=cyan!15!white]{
  $\begin{aligned}
  \bfy_1 &\defeq \textcolor{black}{(\tilde{\sigma}_{d'_{1}} \circ \aff_{W'_{1},\bfb'_{1}}) (\bfx)}\\[-2.5pt]
         &\equiv \textcolor{black}{(\tilde{\sigma}_{d'_{0}}(\bfx), \sigma(b'_{1, d'_{0}+1}), \dots, \sigma(b'_{1,d'_1}))}
  \end{aligned}$
};
\node[name=z2,draw,minimum width=6.1cm,inner xsep=0pt,anchor=north,at={(z1.north-|y2)},fill=olive!15!white]{
  $\begin{aligned}
  \bfy_{L'} &\defeq \aff_{W'_{L'},\bfb'_{L'}}(\bfy_{L'-1})\\[-10pt]
  &\;\;\vdots\\[-5pt]
  \bfy_{2} &\defeq (\tilde{\sigma}_{d'_2}\circ \aff_{W'_2,\bfb'_2})(\bfy_1)
  \end{aligned}$
};
\draw[-latex] (x2) -- (y2);
\draw[-latex] (y2) -- (z2);

\node[name=x3,draw,minimum width=6.1cm,inner xsep=0pt,right=3.25 of x2.center]{$\bfx$};
\node[name=y3,draw,minimum width=6.1cm,inner xsep=0pt,above = .4cm of x3,fill=red!15!white]{
  $\begin{aligned}
  \bfy_{L-1} &\defeq (\tilde{\sigma}_{d_{L-1}} \circ \aff_{W_{L-1},\bfb_{L-1}}) (\bfy_{L-2})\\[-10pt]
  &\;\;\vdots\\[-5pt]
  \bfy_{1} &\defeq (\tilde{\sigma}_{d_{1}} \circ \aff_{W_1,\bfb_1}) (\bfx)
  \end{aligned}$
};
\node[name=z3,draw,minimum width=6.1cm,inner xsep=0pt,above = 3.3cm of x3.north,anchor=north,fill=cyan!15!white]{
  $\begin{aligned}
  \bfy_L &\defeq \textcolor{black}{(\tilde{\sigma}_{d'_{1}} \circ \aff_{W'_{1},\bfb'_{1}}) (\bfy_{L-1})}\\[-2.5pt]
         &\equiv \textcolor{black}{(\tilde{\sigma}_{d'_{0}}(\bfy_{L-1}), \sigma(b'_{1,d'_{0}+1}), \dots, \sigma(b'_{1,d'_1}))}
  \end{aligned}$
};
\node[name=w3,draw,minimum width=6.1cm,inner xsep=0pt,anchor=north, at=(z1.north-|y3),fill=olive!15!white]{
  $\begin{aligned}
  \bfy_{L+L'-1} &\defeq \aff_{W'_{L'},\bfb'_{L'}}(\bfy_{L+L'-2})\\[-10pt]
  &\;\;\vdots\\[-5pt]
  \bfy_{L+1} &\defeq (\tilde{\sigma}_{d'_2}\circ \aff_{W'_2,\bfb'_2})(\bfy_{L})
  \end{aligned}$
};
\draw[-latex] (x3) -- (y3);
\draw[-latex] (y3) -- (z3);
\draw[-latex] (z3) -- (w3);

\node[name=cap1,minimum width=6.1cm,inner xsep=0pt,below=0.2 of x1,anchor=north,draw=none,font=\Large]{$\nu_1$};
\node[name=cap2,minimum width=6.1cm,inner xsep=0pt,below=0.2 of x2,anchor=north,draw=none,font=\Large]{$\nu_2$};
\node[name=cap2,minimum width=6.1cm,inner xsep=0pt,below=0.2 of x3,anchor=north,draw=none,font=\Large]{$\nu_2\circ\nu_1$};
\end{tikzpicture}
\end{adjustbox}
\caption{%
  Illustrations of a network $\nu_1$ without the last affine layer (left),
  a network $\nu_2$ without the first affine layer (middle),
  and their composition $\nu_2\circ\nu_1$ (right).
  Note that $\aff_{W'_1,{\mathbf b}'_1}\circ\aff_{W_L,\mathbf b_L}=\aff_{W'_1,{\mathbf b}'_1}$
  is a floating-point affine transformation.
}
\label{fig:network}
\end{figure}
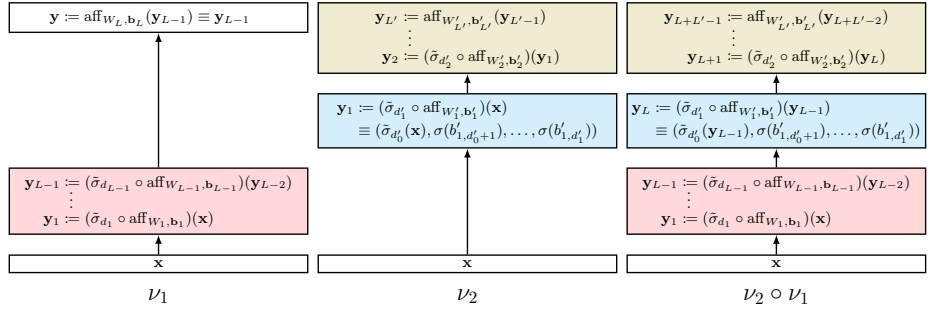

A neural network typically refers to a composition of affine transformations and activation functions.
Formally, for $L \in \bbN$ and $\sigma : \efpq \to \efpq$,
we call a function $\nu$ a \emph{depth-$L$ $\sigma$-neural network} (or a \emph{neural network})
if $\nu$ is defined by
\begin{align}
  \label{eq:nn}
  \nu &: \efpq{}^{d_0} \to \efpq{}^{d_L},
  &
  \nu
  & \defeq
  \aff_{W_L,\bfb_L} \circ
  \tilde{\sigma}_{d_{L-1}} \circ \aff_{W_{L-1},\bfb_{L-1}} \circ  \cdots \circ
  \tilde{\sigma}_{d_1}   \circ \aff_{W_1,\bfb_1}
\end{align}
for some $d_\ell \in \bbN$, $W_\ell \in \fpq{}^{d_\ell \times d_{\ell-1}}$, and $\bfb_\ell \in \fpq{}^{d_\ell}$,
where $\tilde{\sigma}_n : \efpq{}^{n} \to \efpq{}^{n}$ is the coordinatewise application of $\sigma$.
Here, $L$ denotes the number of layers,
$\sigma$ the floating-point activation function,
$d_0$ and $d_L$ the input and output dimensions,
$d_\ell$ the number of hidden neurons in the $\ell$-th layer ($\ell \in [L-1]$),
and $W_\ell$ and $\bfb_\ell$ the parameters of the floating-point affine transformation
in the $\ell$-th layer ($\ell \in [L]$).
We emphasize that a neural network in this paper is a function over \emph{floating-point} values,
defined in terms of \emph{floating-point} activation function and arithmetic.
For instance, a depth-1 neural network is a floating-point affine transformation.

Let $\nu$ be a neural network defined as \cref{eq:nn}.
We say $\nu$ is \emph{without the last affine layer}
if $d_L=d_{L-1}$, $W_L$ is the identity matrix, and $\bfb_L = \mathbf{0}$.
Similarly, we say $\nu$ is \emph{without the first affine layer}
if $d_1 \geq d_0$, $W_1$ is a rectangular diagonal matrix whose diagonal entries are all $1$,
and $b_{1,i} = 0$ for all $i \in [d_0]$.
The two definitions are not perfectly symmetric
due to some technical details arising in our proofs.
We note that a neural network can be constructed by composing
networks without the first/last affine layer(s) and arbitrary networks (\cref{fig:network}).
For example, consider arbitrary networks $\nu_1 : \efpq{}^{n_0} \to \efpq{}^{n_1}$
and $\nu_4 : \efpq{}^{n_3} \to \efpq{}^{n_4}$,
a network without the first affine layer $\nu_2 : \efpq{}^{n_1} \to \efpq{}^{n_2}$, and
a network without the first and last affine layers $\nu_3 : \efpq{}^{n_2} \to \efpq{}^{n_3}$.
It is easily verified that the function $\nu : \efpq{}^{n_0} \to \efpq{}^{n_4}$
specified by $\nu(\bfx) = (\nu_4 \circ \cdots \circ \nu_1)(\bfx)$ denotes a network,
whose definition in the form of \cref{eq:nn}
can be obtained by {``merging'' the last layer of $\nu_1$ and the first layer of $\nu_2$, etc.}

\subsection{Interval Semantics}
\label{sec:interval-semantics}

Interval analysis~\citep{CousotC77,Moore2009} is a technique
for analyzing the behavior of numerical programs soundly and efficiently,
based on abstract interpretation~\citep{Cousot1977}.
It uses intervals to overapproximate the ranges of inputs and expressions,
and propagates them through a program to overapproximate the output range.
Interval analysis has been used to establish the robustness of practical neural
networks~\citep{GehrMDTCV18,GowalDSBQUAMK19,jovanovic22,mao24}.
It can overapproximate the output range of a network over perturbed inputs,
which is required to prove robustness; and it runs efficiently by performing only
simple computations, which is required to analyze large-scale networks.

\paragraph{\bf Interval domain and operations.}

We formalize interval analysis for neural networks as follows.
We first define the \emph{interval domain}
\begin{align}
  \label{eq:interval_domain}
  \bbI &\defeq \set[\big]{\intv{a,b} \,\big|\, a,b \in \efpq \setminus \set{\bot} \text{ with } a \leq b} \cup \set{\top},
\end{align}
on which interval analysis operates.
Here, $\intv{a,b}$ abstracts the floating-point interval $[a,b] \cap \efpq$,
and $\top$ abstracts the entire floating-point set $\efpq$ including $\bot$.
The concrete semantics of an \emph{abstract interval} $\mcI \in \bbI$ and
an \emph{abstract box} $\mcB = (\mcI_1, \ldots, \mcI_d) \in \bbI^d$
are defined through the \emph{concretization function} $\gamma$, where
\begin{align}
  \label{eq:concretization_function}
  \gamma: \cup_{d=1}^{\infty}\bbI^d &\to \cup_{d=1}^{\infty} 2^{\efpq{}^{d}},
  &
  \conc{\mcI} & \defeq
  \begin{cases}
    [a,b] \cap {\efpq} & \text{if } \mcI = \intv{a,b}
    \\
    \efpq & \text{if } \mcI = \top
  \end{cases},
  &
  \conc{\mcB} & \defeq \prod_{i = 1}^d \conc{\mcI_i}.
\end{align}
We say that an abstract box $\mcB \in \bbI^d$ \emph{is in} a set $\mcS \subseteq \bbR^d$ if $\gamma(\mcB) \subseteq \mcS$.

For any function $\phi : \efpq{}^d \to \efpq$ over floats
(which is not a neural network or a floating-point affine transformation),
the \emph{interval operation} $\phi^\sharp : \bbI^d \to \bbI$
extends $\phi$ to the interval domain as follows:
\begin{align}
  \label{eq:interval_operation}
  \phi^\sharp(\mcB)
  & \defeq
  \begin{cases}
    \intv{ \min\mcS, \max\mcS }
    & \text{if } \bot \notin \mcS
    \\
    \top
    & \text{if } \bot \in \mcS
  \end{cases},
  \quad\text{where } \mcS \defeq \phi(\conc{\mcB}).
\end{align}
In the special case that $\phi = \odot \in \set{\oplus, \allowbreak \ominus, \allowbreak \otimes}$
is a floating-point arithmetic operation, the above definition
(using infix notation) is equivalent to the following:
\begin{align}
  \!
  \intv{a,b} \odot^\sharp \intv{c,d}
  & \defeq
  \begin{cases}
    \intv{ \min\mcS, \max\mcS } \!\!\! & \text{if } \bot \notin \mcS
    \\
    \top & \text{if }\bot \in \mcS
  \end{cases},
  \;\;\text{where } \mcS \defeq \set*{\!\begin{aligned} &a \odot c, a \odot d,\\&b \odot c, b \odot d\end{aligned}\!},
\end{align}
and $\odot^\sharp$ returns $\top$ if at least one of its operands is $\top$.%
\footnote{
This definition of $\odot^\sharp$ differs slightly from the standard definition,
as $\odot^\sharp$ uses ``round to nearest'' mode (implicit in $\odot$),
whereas the more common mode is ``round downward/upward''
(e.g., $\intv{a,b} \oplus^\sharp \intv{c,d} \defeq \intv{a \oplus_\downarrow c, b \oplus_\uparrow d}$)~\citep[Section 5]{HickeyJE01}.
This choice is due to different goals to achieve:
our definition overapproximates floating-point operations (e.g., $\oplus$),
while the usual one overapproximates exact operations (e.g., $+$).
}
We remark that $\odot^\sharp$ can be efficiently computed,
and so can $\phi^\sharp$ when $\phi : \efpq \to \efpq$ is piecewise-monotone with finitely many pieces,
which holds for the correctly rounded versions of widely-used activation functions (e.g., ReLU, GELU, sigmoid).
We then define the \emph{interval affine transformation}
$\smash{\aff^\sharp_{W,\bfb}} : \bbI^n \to \bbI^m$,
which extends its floating-point counterpart $\aff_{W,\bfb} : \efpq{}^n \to \efpq{}^m$:
\(
\smash{\aff^\sharp_{W,\bfb}}(\mcI_1, \allowbreak \ldots, \allowbreak \mcI_n) \defeq
\paren[\big]{
  \paren{
    \bigoplus_{j=1}^{n} \! {}^\sharp \, \mcI_j\otimes^\sharp \intv{w_{i,j}, w_{i,j}}
  } \oplus^\sharp \intv{b_i, b_i}
}_{i = 1}^m,
\)
where $\bigoplus \! \smash{{}^\sharp}$ is the interval summation which uses $\smash{\oplus^\sharp}$ instead of $\oplus$.

\paragraph{\bf Interval semantics.}

The \emph{interval semantics} $\nu^\sharp : \bbI^{d_0} \to \bbI^{d_L}$
of a neural network $\nu : \efpq{}^{d_0} \to \efpq{}^{d_L}$
is defined as the result of interval analysis on~$\nu$:
\begin{align}
  \label{eq:interval_semantics}
  \nu^\sharp
  & \defeq
  \aff^\sharp_{W_L,\bfb_L} \circ
  \tilde{\sigma}^\sharp_{d_{L-1}} \circ \aff^\sharp_{W_{L-1},\bfb_{L-1}} \circ  \cdots \circ
  \tilde{\sigma}^\sharp_{d_1}     \circ \aff^\sharp_{W_1,\bfb_1},
\end{align}
where $\nu$ is assumed to be defined as \cref{eq:nn}
and $\tilde{\sigma}^\sharp_n : \bbI^n \to \bbI^n$ is the coordinatewise application of $\sigma^\sharp : \bbI \to \bbI$.
It is easily verified that the interval semantics is sound with respect to the floating-point semantics:
\begin{align}
  \nu\paren[\big]{ \conc{\mcB} } \subseteq \conc{ \nu^\sharp(\mcB) }
  &&
  (\text{$\mcB \in \bbI^{d_0}$}).
\end{align}
That is, the result of interval analysis ${\smash{\nu^\sharp(\mcB)}} \in \bbI{}^{d_L}$
subsumes the set of all possible outputs of the network $\nu$
when the input is in the concrete box $\conc{\mcB} \subseteq \efpq{}^{d_0}$.


\section{Interval Universal Approximation Over Floats}
\label{sec:iua}

This section presents our main result on interval universal
approximation (IUA) for floating-point neural networks.
We first introduce conditions on activation functions for our result
(\cref{sec:iua-conditions}),
and then formally describe our result under these conditions
(\cref{sec:iua-main-result}).
We then compare our IUA theorem over floats with existing IUA theorems over
reals, highlighting several nontrivial differences (\cref{sec:comparison}).

\subsection{Conditions on Activation Functions}
\label{sec:iua-conditions}

Our IUA theorem is for floating-point neural networks that use activation functions
satisfying the following conditions (\cref{fig:condition}).


\begin{figure}[t]
  \begin{subfigure}{0.32\textwidth}
    \centering
    \includegraphics[width=1.0\linewidth]{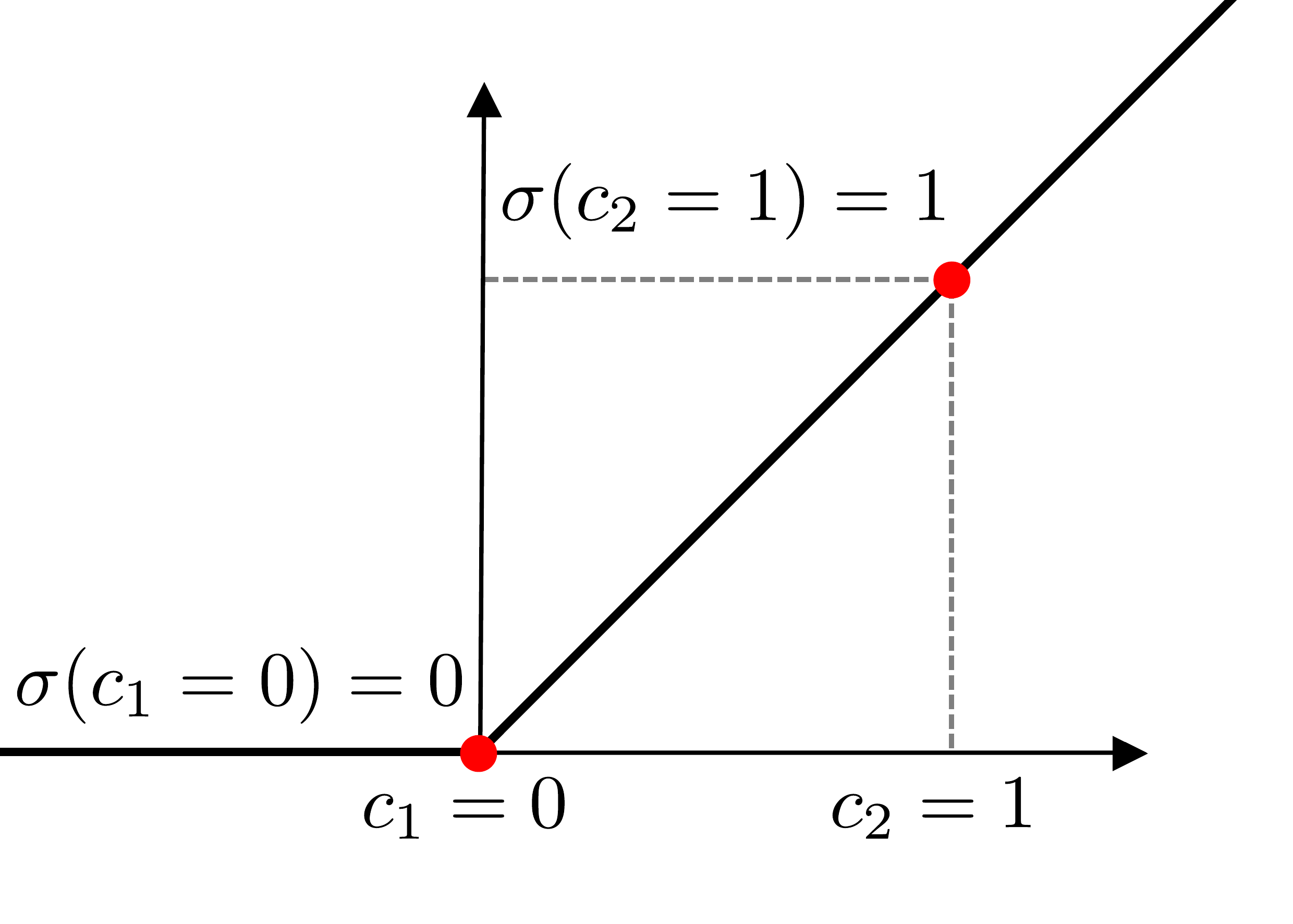}
  \end{subfigure}
  \begin{subfigure}{0.32\textwidth}
    \centering
    \includegraphics[width=1.0\linewidth]{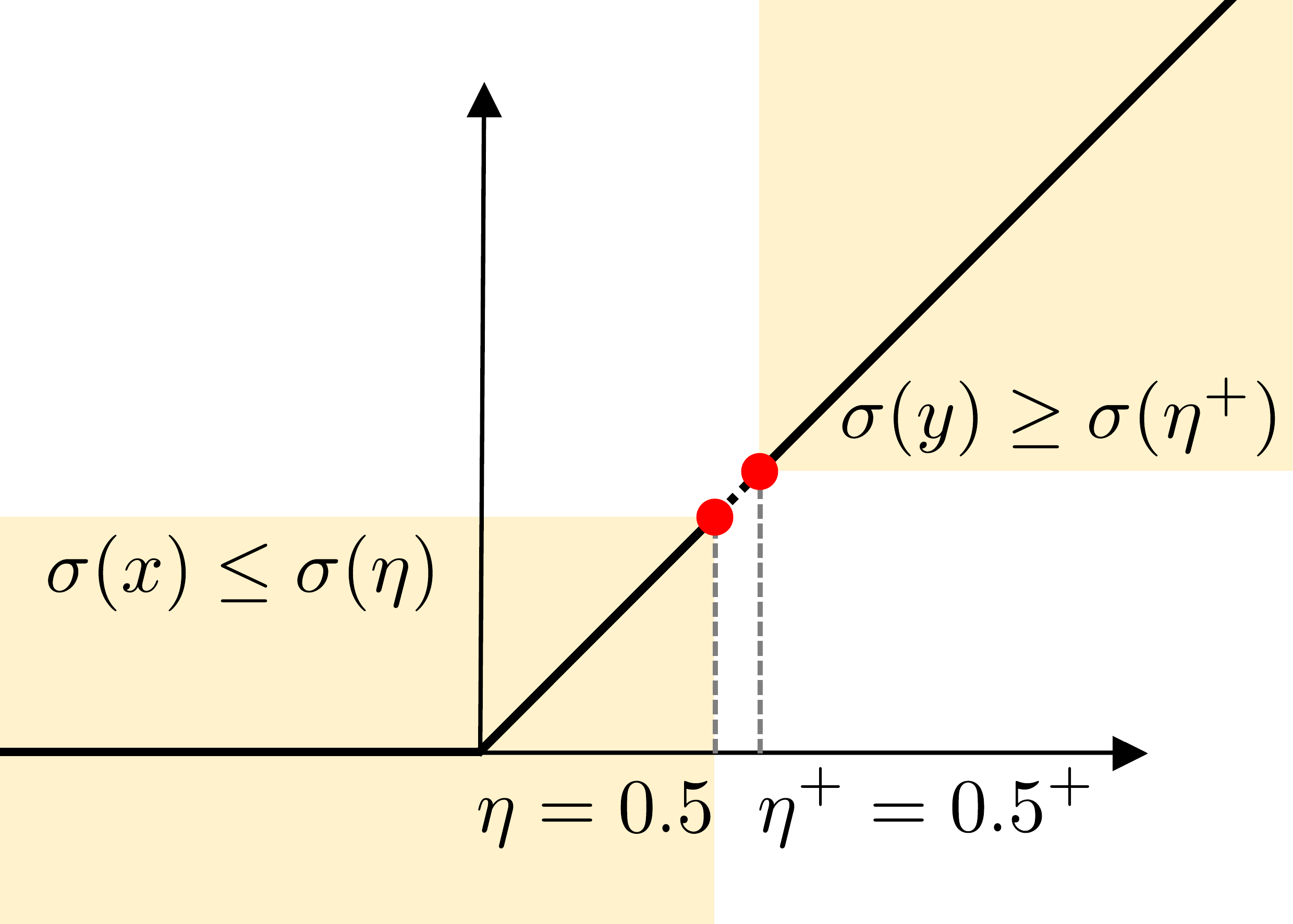}
  \end{subfigure}
  \begin{subfigure}{0.32\textwidth}
    \centering
    \includegraphics[width=1.0\linewidth]{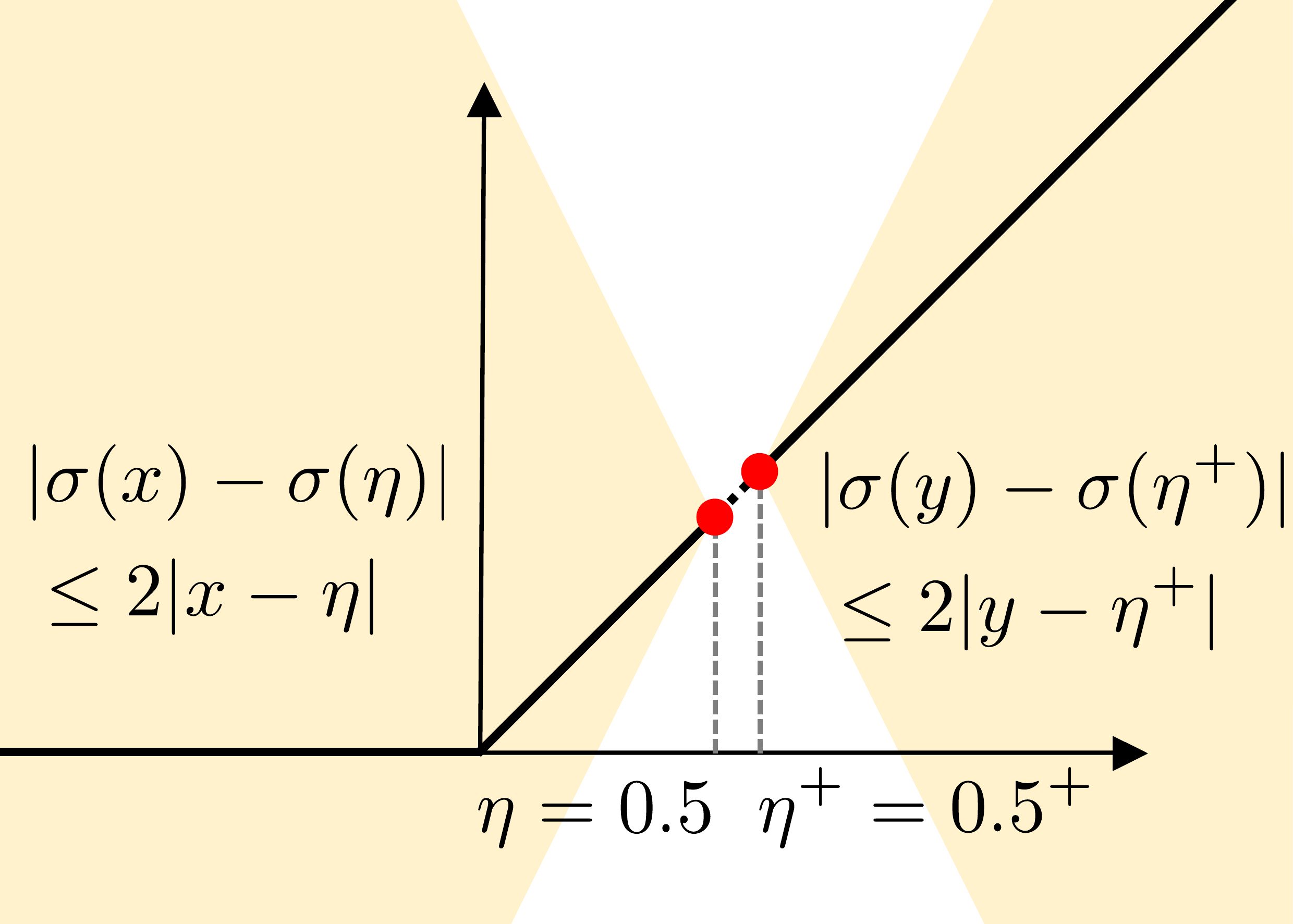}
  \end{subfigure}
  \caption{%
    Illustration of the first (left), second (middle), and third (right) conditions in \cref{cond:activation2}
    for the ReLU activation function: $\sigma(x) \defeq \max\{x,0\}$ for $x \in \efpq$.
  }
  \label{fig:condition}
\end{figure}

\begin{condition}
  \label{cond:activation2}
  An activation function $\sigma:\efpq\to\efpq$ satisfies the following conditions:
  \begin{enumerate}[leftmargin=3.0em, label={\rm{(C\arabic*)}}, ref={\rm{(C\arabic*)}}]
  \item \label{cond:orig-1}
    There exist $c_1, c_2\in \fpq$ such that
    $\sigma(c_1)=0$,
    $|\sigma(c_2)|\in[\tfrac{\feps}{2}+2\feps^2, \tfrac{5}{4}-2\feps]$,
    $\max\{|c_1|,|c_2|\}\ge2^{\emin+1}$,
    and $\sigma(x)$ lies between $\sigma(c_1)$ and $\sigma(c_2)$ for all $x$ between $c_1$ and $c_2$,
    \tcr{where $\feps$ is the machine epsilon (see \cref{sec:floating-point}).}

  \item \label{cond:orig-2}
    There exists $\eta\in\bbF$ with
    $|\eta|\in [2^{\emin+5},4-8\feps]$ and
    $|\sigma(\eta)|,|\sigma(\eta^+)|\in[2^{\emin+5}, \allowbreak 2^{\emax-6} \cdot |\eta|]$
    such that for any $x,y\in \fpq$ with $x\le \eta<\eta^+ \le y$,
    \begin{align}
      \sigma(x)\le \sigma(\eta)<\sigma(\eta^+)\le \sigma(y)
      \quad\text{or}\quad
      \sigma(x)\ge \sigma(\eta)>\sigma(\eta^+)\ge \sigma(y).
    \end{align}

  \item \label{cond:orig-3}
    There exists $\lips \in[0, 2^{\emax-7} \cdot \min\{|\sigma(\eta)|,2^{\mbit+3}\}]$
    such that for any $x,y\in \fpq$ with $x\le \eta<\eta^+ \le y$,
    \begin{align}
      &|\sigma(x) - \sigma(\eta)| \le \lips |x-\eta|
      \quad\text{and}\quad
      |\sigma(y) - \sigma(\eta^+)| \le \lips |y - \eta^+|.
    \end{align}
  \end{enumerate}
\end{condition}

The condition \cref{cond:orig-1} states that the activation function $\sigma $ can output the exact
zero (i.e., $\sigma(c_1)$) and some value whose magnitude is
approximately in $[\frac{\feps}{2}, \frac{5}{4}]$ (i.e., $\sigma(c_2)$);
and its output is within $\sigma(c_1)$ and $\sigma(c_2)$ for all inputs between $c_1$ and $c_2$.
The condition \cref{cond:orig-2} states that there exists some \emph{threshold} $\eta$
such that $\sigma(x)$ is either smaller or greater than $\sigma(\eta)$ or $\sigma(\eta^+)$,
depending on whether $x$ is on the left or right side of $\eta$.
This condition holds automatically for all monotone activation functions
that are non-constant on either $[2^{\emin+5},4-8\varepsilon]\cap\fpq$ or $[-4+8\varepsilon,-2^{\emin+5}]\cap\fpq$.
The condition \cref{cond:orig-3} states that $\sigma$ does not increase or decrease
too rapidly from $\eta$ and $\eta^+$, which implies that \mbox{$\sigma(x)$ is finite
for all finite floats $x\in \fpq$.}

While \cref{cond:activation2} is mild, verifying whether practical
activation functions over floats satisfy \cref{cond:activation2} can be
cumbersome.
Floating-point activation functions are typically
implemented in complicated ways~\citep{Markstein2000, Muller2016, Beebe17}
(e.g., by intermixing floating-point operations with
integer/bit-level operations and if-else branches), which makes it
challenging to rigorously analyze such implementations~\citep{lee2018,Faissole24}.
To bypass this issue, we focus on the correctly rounded version
$\sigma : \efpq \to \efpq$ of a real-valued activation function
$\rho : \bbR \to \bbR$ (i.e., $\sigma(x) \defeq \round{\rho(x)}$),
when verifying \cref{cond:activation2}.
Correctly rounded versions of elementary mathematical functions have been
actively developed in several software
libraries~\citep{SibidanovZG22,LimN21,LimN22,ziv2001,daramyloirat2006}.

Under the correct rounding assumption, we provide an easily verifiable
sufficient condition for activation functions on reals that can be used to
verify \cref{cond:activation2} for their rounded versions.
The proof of \cref{lem:activation} is in \cref{sec:pflem:activation}.

\begin{lemma}
  \label{lem:activation}
  For any activation function $\rho:\bbR\to\bbR$, the correctly rounded activation
  $\round{\rho} : \efpq \to \efpq$ satisfies \cref{cond:activation2} if the following conditions hold:
  \begin{enumerate}[leftmargin=3.3em, label={\rm{(C\arabic*${}^\prime$)}}, ref={\rm{(C\arabic*${}^\prime$)}}]
  \item \label{cond:suff-1}
    There exist $c_1',c_2'\in\fpq$ such that
    $|\rho(c_1')|\le\tfrac{\fmin}{2}$,
    $|\rho(c_2')|\in[\tfrac{\feps}{2}+2\feps^2, \tfrac{5}{4}-2\feps]$,
    $\max\{|c_1'|,|c_2'|\}\ge2^{\emin+1}$,
    and $\rho(x)$ lies between $\rho(c_1')$ and $\rho(c_2')$ for all $x$ between $c_1'$ and $c_2'$,
    \tcr{where $\fmin$ is the smallest positive float (see \cref{sec:floating-point}).}

  \item \label{cond:suff-2}
    There exists $\delta \in \bbR$ with $|\delta|\in[\frac3{8},\frac{7}{8}]$ such that
    \begin{itemize}
    \item for all $x,y\in\bbR$ satisfying $x\le \delta-\tfrac18<\delta+\tfrac18 \le y$,
      \begin{equation*}
        \rho(x) \,{\le}\, \rho(\delta-\tfrac18) \,{<}\, \rho({\delta+\tfrac18}) \,{\le}\, \rho(y)~~\text{or}~~
        \rho(x) \,{\ge}\, \rho(\delta-\tfrac18) \,{>}\, \rho({\delta+\tfrac18}) \,{\ge}\, \rho(y),
      \end{equation*}
    \item $|\rho(x)|\in[\frac14,1]$ and $|\rho(x)-\rho(y)|>\frac18|x-y|$
      for all $x,y\in[\delta-\tfrac1{8},\delta+\tfrac1{8}]$.
    \end{itemize}

  \item \label{cond:suff-3}
    $\rho$ is $\lips$-Lipschitz continuous for some $\lips \in[0 , \tfrac{1}{5} \cdot 2^{\emax-9}]$.
  \end{enumerate}
\end{lemma}

The conditions \cref{cond:suff-1,cond:suff-2,cond:suff-3} in \cref{lem:activation} correspond to
the conditions \cref{cond:orig-1,cond:orig-2,cond:orig-3} in \cref{cond:activation2}.
The condition \cref{cond:suff-1}, corresponding to \cref{cond:orig-1},
can be easily satisfied since modern activation functions are piecewise-monotone
and either zero at zero (e.g., ReLU, GELU, softplus, tanh)
or close to zero at $-\fmax$ or $\fmax$ (e.g., sigmoid).
The condition \cref{cond:suff-2} roughly states the existence of $\delta \in \bbR$ satisfying the following:
(i) $\rho(\delta-\frac{1}{8})$ and $\rho(\delta+\frac{1}{8})$ are lower/upper bounds of $\rho$
on $(-\infty, \delta-\frac{1}{8})$ and $(\delta+\frac{1}{8}, \infty)$; and
(ii) $\rho$ is bounded and strictly monotone on $[\delta-\frac{1}{8},\delta+\frac{1}{8}]$.
This condition guarantees the existence of $\eta \in \bbF$ in \cref{cond:orig-2}.
The condition \cref{cond:suff-3}, corresponding to \cref{cond:orig-3},
can also be easily satisfied since $\lips <3$ for most practical activation functions.
We note that \cref{lem:activation} gives sufficient but not necessary
\mbox{conditions for a correctly rounded activation function to satisfy
\cref{cond:activation2}.}

The following corollary uses \cref{lem:activation} to show that many prominent
activation functions satisfy \cref{cond:activation2}.
Its proof is in \cref{sec:pfcor:activation}.
\begin{corollary}
  \label{cor:activation}
  The correctly rounded implementations of the $\relu$, $\lrelu$, $\gelu$,
  $\elu$, $\mish$, $\SoftPlus$, $\Sigmoid$, and $\tanh$ activations
  satisfy \cref{cond:activation2}.
\end{corollary}

\subsection{Main Result}
\label{sec:iua-main-result}

We are now ready to present our IUA theorem over floating-point arithmetic.

\begin{theorem}
  \label{thm:main}
  Let $\sigma:\efpq\to\efpq$ be an activation function satisfying \cref{cond:activation2}.%
  \footnote{\tcr{%
    \cref{cond:activation2} is sufficient for \cref{thm:main} but not necessary.
    E.g., \cref{thm:main} still holds
    under 8-bit floats (both E4M3 and E5M2 formats \citep{micikevicius2022}) for the ReLU activation function;
    this corresponds to the case where \cref{cond:orig-1,cond:orig-2} hold but \cref{cond:orig-3} is violated.
    %
  }}
  Then, for any target function $f : \bbR^d\to\bbR$,
  there exists a $\sigma$-neural network $\nu : \efpq{}^d \to \efpq$ such that
  \begin{align}
    \label{eq:thm:main}
    &\gamma\paren[\big]{ \nu^\sharp(\mcB) }
    =\Big[\min \widehat{f} \paren[\big]{ \gamma(\mcB) },\,\max\widehat{f} \paren[\big]{ \gamma(\mcB) }\Big] \cap \efpq
  \end{align}
  for $\widehat{f} = \round{f} : \efpq{}^d \to \efpq{}^d$ and for all abstract boxes $\mcB$ in $[-1,1]^d$.%
  \footnote{\tcr{%
    In the literature on universal approximation theorems, it is typically assumed that
    the inputs are in $[0,1]$ or in a compact subset of $\bbR$
    (e.g., \citep{cybenko89,yarotsky18,baader20,wang2022interval}).
    Since the inputs are often normalized to $[-1,1]$,
    we focus the theoretical analysis on $[-1,1]^d$.
    %
  }}
\end{theorem}

\cref{thm:main} states that
for any activation function $\sigma:\efpq \to \efpq$ satisfying \cref{cond:activation2}
and any target function $f : \bbR^d \to \bbR$,
there exists a $\sigma$-network $\nu$ whose interval semantics \emph{exactly computes} the
upper and lower points of the direct image map
of the rounded target $\widehat{f} : \efpq{}^d \to \efpq{}$ on $[-1,1]^d \cap \fpq^d$.
A special case of our IUA \cref{thm:main} is the following
universal approximation (UA) theorem for floating-point neural networks:
\begin{align}
  \label{eq:ua}
  \nu(\bfx)=\widehat f(\bfx) && (\bfx \in [-1,1]^d \cap \fpq^d).
\end{align}
That is, floating-point neural networks using an activation function satisfying \cref{cond:activation2}
can represent any function $\widehat{f}: [-1,1]^d \cap \fpq^d \to \fpq\cup\{-\infty,+\infty\}$;
or the rounded version of any real function $f: [-1,1]^d \to \bbR$.
Moreover, \cref{thm:main} easily extends to any target function
$f: \bbR^d \to \bbR^{d'}$ with multiple outputs.

As previous IUA results assume exact operations over reals,
they do not extend to our setting of floating-point arithmetic
(due to rounding errors, overflow, NaNs, discreteness, boundedness, etc.).
As a simple example of these issues, consider the following subnetwork,
which is used in the IUA proof of \citep{baader20}:
\begin{align*}
\mu(x,y)
= \frac{1}{2}\paren[\Big]{ \text{ReLU}(x+y)-\text{ReLU}(-x-y)-\text{ReLU}(x-y)-\text{ReLU}(y-x) }.
\end{align*}
This subnetwork returns $\min\{x,y\}$ if all operations are exact.
However, it does not under floating-point arithmetic due to the
rounding error: if $(+,\times)$ is replaced by $(\oplus,\otimes)$,
then $\mu(x,y) = 0 \ne \feps = \min\set{x,y}$ for $x=1$ and $y=\feps$.

In addition, the network construction in \cite[Theorem 4.10]{wang2022interval}
requires multiplying a large number $z$ that depends on
the target error and the activation function, to the output of some neuron.
However, because $\bbF$ is bounded and floating-point operations are subject to overflow,
the number $z$ and the result of the multiplication are not guaranteed to be within $\bbF$ when using a
small target error (e.g., less than $\fmin$) or when using common
activations functions (e.g., ReLU, softplus).
To bypass these issues, we carefully analyze rounding errors
and design a network without infinities in the intermediate layers, when proving \cref{thm:main}.

\tcr{We present the proof outline of \cref{thm:main} in \cref{sec:iua-proof},
  and the full proof in \cref{sec:pflem:proof_results-0,sec:pflem:proof_results,sec:techlemma_revised}.
  We implemented the proof (i.e., our network construction) in Python
  and made it available at \url{https://github.com/yechanp/floating-point-iua-theorem}.}

\subsection{Comparison With Existing Results Over Reals}
\label{sec:comparison}

\Cref{thm:main}, which gives an IUA theorem over floats, has notable differences from
previous IUA theorems over the reals \citep[Theorem 1.1]{baader20};
\citep[Theorem 3.7]{wang2022interval}.

One difference is the class of target functions and the desired property of networks.
Previous IUA theorems find a network
that \emph{sufficiently approximates} the direct image map of a \emph{continuous} target function
(i.e., $\delta>0$ in \cref{eq:approximate-direct-image}).
In contrast, our IUA theorem finds a network
that \emph{exactly computes} the direct image map of an \emph{arbitrary} rounded target function
(i.e., $\delta=0$ in \cref{eq:approximate-direct-image}).
This difference arises from the domains of the functions being approximated:
the real-valued setting considers functions $f$ over $[-1, 1]^d$ (or a compact $\mcK \subset \bbR^d$);
the floating-point setting considers functions $\widehat{f}$ over $[-1, 1]^d \cap \fpq^d$.

\begin{itemize}
\item
  Since $[-1, 1]^d$ is uncountable, exactly computing the direct image map of $f$
  requires a network to fit \emph{uncountably} many input/output pairs and related box/interval pairs.
  This task is difficult to achieve, and indeed,
  recent works~\cite{mirman22,baader24,baader24thesis}
  prove that it is theoretically unachievable even for simple target functions
  (e.g., continuous piecewise linear functions).

\item
  Since $[-1, 1]^d \cap \fpq^d$ is finite, exactly computing the direct image map of $\widehat{f}$
  requires a network to fit \emph{finitely} many input/output and box/interval pairs.
  Our result proves that, despite all the complexities of floating-point computation,
  this task can be achieved for any rounded target function.
\end{itemize}

Another key difference is the class of activation functions.
There are real-valued activation functions $\rho, \rho' : \bbR \to \bbR$ such that
previous IUA theorems \emph{cannot} hold for $\rho$ but our IUA theorem \emph{does} hold for $\round{\rho} : \efpq \to \efpq$;
\mbox{and vice versa for $\rho'$.}

\begin{itemize}
\item
  An example of $\rho$ is the \emph{identity} function: $\rho(x) = x$.
  \tcr{%
  No classical IUA or UA theorem can hold for $\rho$,
  since all \emph{real-valued} $\rho$-networks $\mu : \bbR^d \to \bbR$ are \emph{affine over the reals}
  (i.e., there exists $A \in \bbR^{1 \times d}$ and $b \in \bbR$
  such that $\mu(\bfx) = A\bfx + b$ for all $\bfx \in \bbR^d$).
  In contrast, our IUA theorem does hold for $\round{\rho}$,
  because $\round{\rho}$ satisfies all the conditions in \cref{lem:activation}
  (with constants $c_1' = 0$, $c_2' = 1$, $\delta = 1/2$, and $\lips = 1$).
  This counterintuitive result is made possible because
  \emph{floating-point} $\round{\rho}$-networks $\nu : \efpq{}^d \to \efpq$ can be \emph{non-affine over the reals}
  (i.e., there may not exist $A \in \bbR^{1 \times d}$ and $b \in \bbR$
  such that $\nu(\bfx) = A\bfx + b$ for all $\bfx \in \fpq^d$).
  This non-affineness arises from rounding errors:
  some floating-point affine transformations $\aff_{W, \bfb}$ are not actually affine over the reals
  due to rounding errors.}
  An interesting implication of this result is discussed in \cref{sec:completeness}.

\item
  An example of $\rho'$ is any function that
  is non-decreasing on $\bbR$,
  is constant on $[-\fmax, \fmax]$,
  and satisfies $\lim_{x \to -\infty} \rho'(x) < \lim_{x \to +\infty} \rho'(x)$, where the two limits exist in $\bbR$.
  The real-valued IUA theorem holds for $\rho'$,
  because $\rho'$ satisfies the condition in \cite[Definition 2.3]{wang2022interval}.
  However, no floating-point IUA or UA theorem can hold for $\round{\rho'}$,
  because all $\round{\rho'}$-networks $\nu : \efpq \to \efpq$
  must be monotone if its depth is 1, and must satisfy $\nu(0) = \nu(\fmin)$ otherwise.
  The monotonicity holds when the depth is 1 since $\oplus,\otimes$ are monotone when an operand is a constant;
  and $\nu(0) = \nu(\fmin)$ holds otherwise since
  $x \otimes a \oplus b \in [-\fmax, \fmax]$ for all $x \in \set{0, \fmin}$ and $a, b \in \fpq$,
  and $\round{\rho'}$ is constant on $[-\fmax, \fmax] \cap \fpq$.
\end{itemize}


\section{Implications of IUA Theorem Over Floats}
\label{sec:iua-implications}

This section presents two important implications of our IUA theorem,
on provable robustness and ``floating-point completeness''.
We first prove the existence of a provably robust floating-point network,
given an ideal robust floating-point classifier (\cref{sec:provable-robustness}).
We then prove that floating-point $+$ and $\times$ are sufficient to simulate
all halting programs that \mbox{return finite/infinite floats when given finite floats (\cref{sec:completeness}).}

\subsection{Provable Robustness of Neural Networks}
\label{sec:provable-robustness}

Consider the task of classifying floating-point inputs $\bfx \in \mcX$
(e.g., images of objects) into $n \in \bbN$ classes (e.g., categories of objects),
where $\mcX \defeq [-1,1]^d \cap \fpq^d$ denotes the space of inputs throughout this subsection.
For this task, a function $f : \mcX \to \fpq^n$
is often viewed as a classifier in the following sense:
$f$ predicts $\bfx$ to be in the $i$-th class ($i \in [n]$),
where $i \defeq \class(f(\bfx))$ and $\class : \fpq^n \to [n]$ is defined by
$\class(y_1, \ldots, y_n) \defeq \argmax_{i \in [n]} y_i$
with an arbitrary tie-breaking rule.

A typical robustness property of a classifier $f$ is that
$f$ should predict the same class for all neighboring inputs
under the $\ell_\infty$ distance~\cite{Li2023}.
We formalize this notion of \emph{robust} classifiers in a way similar to \cite[Definition A.4]{wang2022interval}.

\begin{definition}
  \label{def:robust}
  Let $\delta > 0$ and $\mcD \subseteq \mcX$.
  A classifier $f : \mcX \to \fpq^n$ is called \emph{$\delta$-robust on $\mcD$}
  if for all $\bfx_0 \in \mcD$, $\bfy, \bfy' \in f\left(\mcN_\delta(\bfx_0)\right)$ implies $\class(\bfy) = \class(\bfy')$,
  where $\mcN_\delta(\bfx_0) \defeq \set{ \bfx \in \mcX \mid \| \bfx_0 - \bfx \|_\infty \leq \delta }$
  and $\|\cdot\|_\infty$ denotes the $\ell_\infty$-norm.
\end{definition}

Neural networks have been widely used as classifiers,
but establishing the robustness properties of practical networks as in \cref{def:robust} is intractable
due to the enormous number of inputs to be checked (i.e., $\abs{\mcN_\delta(\bfx_0)} \gg 1$ when $d \gg 1$).
Instead, these properties have been proven often by using interval analysis,
as mentioned in \cref{sec:interval-semantics}.
We formalize the notion of such \emph{provably robust} networks under interval analysis,
{in a way similar to \cite[Definition A.5]{wang2022interval}.}

\begin{definition}
  \label{def:provably-robust}
  Let $\delta > 0$ and $\mcD \subseteq \mcX$.
  A neural network $\nu : \efpq{}^d \to \efpq{}^n$ is called \emph{$\delta$-provably robust on $\mcD$}
  if for all $\bfx_0 \in \mcD$, $\bfy, \bfy' \in \gamma( \nu^\sharp(\mcB) )$ implies $\class(\bfy) = \class(\bfy')$,
  where $\mcB \in \bbI^d$ denotes the abstract box such that $\gamma(\mcB) = \mcN_\delta(\bfx_0)$.
\end{definition}

Under these definitions, we prove that given an ideal robust classifier $f$,
we can always find a neural network $\nu$
(i) whose robustness property is \emph{exactly} the same as that of $f$
and is easily provable using only interval analysis, and
(ii) whose predictions are \emph{precisely} equal to those of $f$.

\begin{theorem}
  \label{thm:provable-robustness}
  Let $f : \mcX \to \fpq^n$ be a classifier that is $\delta$-robust on $\mcD$,
  and $\sigma : \efpq \to \efpq$ be an activation function satisfying \cref{cond:activation2}.
  Then, there exists a $\sigma$-neural network $\nu : \efpq{}^d \to \efpq{}^n$ that is $\delta$-provably robust on $\mcD$
  and makes the same prediction as $f$ on $\mcD$ (i.e., $\class(\nu(\bfx)) = \class(f(\bfx))$ for all $\bfx \in \mcD$).
\end{theorem}
\begin{proof}[Proof sketch]
  \tcr{%
    We show this
    (i) by applying \cref{thm:main} to $n$ target functions that are constructed from $f$, and
    (ii) by using the following observation:
    the network constructed in the proof of \cref{thm:main} has depth not depending on a target function (when $d$ is fixed).
    The full proof is in~\cref{sec:pfthm:provable-robustness}.}
\end{proof}

\subsection{Floating-Point Interval-Completeness}
\label{sec:completeness}

To motivate our result, we recall the notion of Turing completeness.
A computation model is called \emph{Turing-complete} 
if for every Turing machine $T$, there exists a program in the model 
that can simulate the machine~\citep{Kozen97, MooreM11, Barak23}. 
%
Extensive research has established the Turing completeness of numerous computation models:
from untyped $\lambda$-calculus~\citep{Church33, Turing36} 
and $\mu$-recursive functions~\citep{Godel34, Church36}, 
to type systems (e.g., Haskell~\citep{Wansbrough98}, Java~\citep{Grigore17}) 
and neural networks over the rationals (e.g., RNNs~\cite{SiegelmannS95}, Transformers~\citep{PerezBM21}).
%
These results identify simpler computation models as powerful as Turing machines,
and shed light on the computational power of new models.

We ask an analogous question for \emph{floating-point} computations instead of \emph{binary} computations,
where the former is captured by floating-point programs and the latter by Turing machines.
That is, which small class of floating-point programs can simulate all (or nearly all) floating-point programs?

Formally, let $\mcF$ be the set of all terminating programs
that take finite floats and return finite or infinite floats,
where these programs can use any floating-point constants/operations (e.g., $-\infty$, $\otimes$)
and language constructs (e.g., if-else, while).
%
Then, $\mcF$ semantically denotes the set of all functions
from $\fpq^n$ to $(\fpq \cup \{-\infty, +\infty\})^m$ for all $n,m \in \bbN$,
because each such function can be expressed with if-else branches and floating-point constants.
For this class of programs, we define the notion of
\emph{(interval\=/)simulation} and \emph{floating-point (interval\=/)completeness} as follows.

\begin{definition}
  Let $P, Q \in \mcF$ be programs with arity $n$.
  We say $Q$ \emph{simulates} $P$ if $Q(\bfx) = P(\bfx)$ for all $\bfx \in \fpq^n$,
  where $P(\bfx)$ denotes the concrete semantics of $P$ on $\bfx$.
  We say $Q$ \emph{interval-simulates} $P$
  if $\gamma({Q}^\sharp(\mcB)) = [\min P(\gamma(\mcB)), \max P(\gamma(\mcB))] \cap \efpq$
  for all abstract boxes $\mcB$ in $\fpq^n$,
  where \mbox{${Q}^\sharp(\mcB)$ denotes the interval semantics of $Q$ on $\mcB$.}
\end{definition}

\begin{definition}
  We say a class of programs $\mcG \subseteq \mcF$ is \emph{floating-point (interval\=/)complete}
  if for every $P \in \mcF$, there exists $Q \in \mcG$ such that $Q$ (interval\=/)simulates~$P$.
\end{definition}

We prove that a surprisingly small class of programs is floating-point interval-complete (so floating-point complete).
In particular, we show that only floating-point addition, multiplication, and constants
are sufficient to interval\=/simulate \emph{all} halting programs that output finite/infinite floats when given finite floats.

\begin{theorem}
  \label{thm:turing}
  $\mcF_{\oplus, \otimes} \subset \mcF$ is floating-point interval-complete,
  where $\mcF_{\oplus, \otimes}$ denotes the class of straight-line programs
  that use only $\oplus$, $\otimes$, and floating-point constants.
\end{theorem}
\begin{proof}[Proof sketch]
  \tcr{%
    We show this by extending the key lemma used in the proof of \cref{thm:main}:
    there exist $\sigma$-networks that capture the direct image maps of indicator functions over $[-1,1]^n \cap \fpq^n$
    (\cref{lem:indc}).
    In particular, we prove that $[-1,1]^n \cap \fpq^n$ can be extended to $\fpq^n$
    if $\sigma$ is the identity function.
    The full proof is in \cref{sec:pfthm:turing}.}
\end{proof}

To our knowledge, this is the first non-trivial result on floating-point (interval\=/)completeness.
This result is an extension of our IUA theorem (\cref{thm:main}) for the identity activation function $\sigma_\mathrm{id}$,
in that floating-point interval-completeness considers the input domain $\fpq^n$ (not $[-1,1]^n \cap \fpq^n$)
and $\mcF_{\oplus, \otimes}$ includes all $\sigma_\mathrm{id}$-networks (but no other $\sigma$-networks).
\cref{thm:turing}, however, \emph{cannot} be extended to
the input domain $(\fpq \cup \set{-\infty, +\infty})^n$ (instead of $\fpq^n$),
since no program in $\mcF_{\oplus, \otimes}$ can represent a non-constant function
that maps an infinite float to a finite float%
---this is because $\oplus$ and $\otimes$ do not return finite floats when applied to $\pm\infty$.


\section{Proof of IUA Theorem Over Floats}
\label{sec:iua-proof}

We now prove \cref{thm:main} by constructing a $\sigma$-neural network
that computes the upper and lower points of the direct image map of a
rounded target function $\widehat{f}$.
For $a, b \in \bbR$,
we let $[a, b]_\fpq \defeq [a, b] \cap \fpq$ and
$\bbI_{[a,b]} \defeq \set{ \mcI \in \bbI \mid \gamma(\mcI) \subseteq [a,b]}$.
With this notation, $\smash{(\bbI_{[a,b]})^d}$
is the set of all abstract boxes in $\smash{[a,b]^d}$.

We start with defining indicator functions for a set of floating-point values and for an abstract box,
which play a key role in our proof.

\begin{definition}
  \label{def:indicator}
  Let $d \in \bbN$.
  For $\mcS \subseteq \efpq{}^d$,
  we define $\iota_{\mcS} : \efpq{}^d\to\efpq$ as
  $\iota_{\mcS}(\bfx) \defeq 1$ if $\bfx\in\mcS$, and $\iota_{\mcS}(\bfx) \defeq 0$ otherwise.
  For $a\in\fpq$, we define $\iota_{> a} : \efpq \to \efpq$ by $\smash{\iota_{\{x> a \,\mid\, x\in\fpq\}}}$,
  and define $\iota_{\ge a}$, $\iota_{< a}$, $\iota_{\le a}$ analogously.
  For $\mcC \in \bbI^d$, we define $\iota_{\mcC} : \efpq{}^d \to \efpq$ by $\iota_{\conc{\mcC}}$.
\end{definition}

Our proof of \cref{thm:main} consists of two parts.
We first show the existence of $\sigma$-networks
that precisely compute indicator functions under the interval semantics.
We then construct a $\sigma$-network stated in \cref{thm:main}
by composing the $\sigma$-networks for indicator functions and using the properties of indicator functions.

Both parts of our proof are centered around a new property of activation functions,
which we call ``$([a,b]_{\fpq}, \eta, K, L_\phi, L_\psi)$-separability'' and define as follows.

\begin{definition}
  \label{def:separability}
  We say that $\sigma:\efpq\to\efpq$ is \emph{$([a,b]_{\fpq}, \eta, K, L_\phi, L_\psi)$-separable}
  for $a,b, \eta, K \in \fpq$ and $L_\phi, L_\psi \in \bbN$ if the following hold:
  \begin{itemize}[leftmargin=15pt]
  \item
    For every $z\in[a,b]_{\fpq}$, there exist depth-$L_\phi$ $\sigma$-networks
    $\phi_{\le z}, \phi_{\ge z} : \efpq \to \efpq$  without the last affine layer
    \mbox{such that $\smash{\phi_{\le z}^\sharp} = \smash{(K\iota_{\le z})^\sharp}$
      and $\smash{\phi_{\ge z}^\sharp} = \smash{(K\iota_{\ge z})^\sharp}$ on $\bbI_{[a,b]}$.}
  \item
    There exists a depth-$L_\psi$ $\sigma$-network $\psi_{>\eta} : \efpq \to \efpq$
    without the first and last affine layers
    such that $\smash{\psi_{>\eta}^\sharp} = \smash{(K\iota_{>\eta})^\sharp}$ on $\bbI_{[a,b]}$.
  \end{itemize}
\end{definition}

The first condition in \cref{def:separability} ensures the existence of $\sigma$-networks
that perfectly implement scaled indicator functions $K \iota_{\leq z}$ and $K \iota_{\geq z}$
under the interval semantics, for all $z \in [a,b]_\fpq$.
Since these networks should have the same depth $L_\phi$ without the last affine layer,
a function $\nu : \efpq{}^n \to \efpq$ defined, e.g., by
\begin{align}
  \label{eq:pfthm:main1}
  \nu(x_1,\dots,x_n)
  & =
  \paren[\Bigg]{ \bigoplus_{i=1}^n \alpha \otimes \phi_{\le z_i}(x_i) } \oplus \beta
\end{align}
denotes a depth-$L_\phi$ $\sigma$-network for any $z_i \in [a,b]_\fpq$ and $\alpha, \beta \in \fpq$.
The second condition in \cref{def:separability} guarantees that
another scaled indicator function $K \iota_{>\eta}$ can be precisely implemented
by a depth-$L_\psi$ $\sigma$-network $\psi_{>\eta}$ without the first and last affine layers.
This implies, e.g., that $\psi_{>\eta} \circ \nu$ denotes a depth-$(L_\phi+L_\psi-1)$ $\sigma$-network,
where $\nu$ denotes the network presented in \cref{eq:pfthm:main1}.

Using the separability property, we can formally state the two parts of our proof as \cref{lem:indc,lem:indc-to-iua}.
\cref{thm:main} is a direct corollary of the two lemmas.
We present the proofs of \cref{lem:indc,lem:indc-to-iua} in the next subsections
(\cref{sec:pflem:indc,sec:pflem:indc-to-iua}).

\begin{lemma}
  \label{lem:indc}
  Suppose that $\sigma : \efpq \to \efpq$ satisfies \cref{cond:activation2} with constants $c_2, \eta \in \fpq$.
  Then, $\sigma$ is $([-1,1]_{\fpq}, \eta, K, L_\phi, L_\psi)$-separable for some $L_\phi, L_\psi \in \bbN$,
  where $\eta$ and $K \defeq \sigma(c_2)$ satisfy $|\eta| \in [2^{\emin+5}, 4-8\feps]$
  and $|K| \in [\tfrac{\feps}{2}+2\feps^2, \tfrac{5}{4}-2\feps]$.
\end{lemma}

\begin{lemma}
  \label{lem:indc-to-iua}
  Suppose that $\sigma : \efpq \to \efpq$ is $([a,b]_{\fpq},\eta,K,L_\phi,L_\psi)$-separable
  for some $a, b, \eta, K \in \fpq$ and $L_\phi, L_\psi \in \bbN$
  with $|\eta| \in [2^{\emin+5}, 4-8\feps]$ and $|K| \in [\tfrac{\feps}{2}+2\feps^2, \tfrac{5}{4}-2\feps]$.
  Then, for every $d \in \bbN$ and
  function $h : \smash{\efpq{}^d} \to \smash{\efpq} \setminus \set{\bot}$,
  there exists a $\sigma$-neural network $\nu : \efpq{}^d \to \efpq$ such that
  $\smash{\nu^\sharp}(\mcB) = \smash{h^\sharp}(\mcB)$ for all abstract boxes $\mcB$ in $[a,b]^d$.
\end{lemma}

\tcr{%
To prove \cref{lem:indc}, we construct a $\sigma$-network for the scaled indicator function $K \iota_{\geq z}$ in two steps.
We first construct a $\sigma$-network that maps all inputs smaller than $z$ to some point $x_1$,
and all other inputs to another point $x_2 \neq x_1$ (\cref{lem:sigmaindicator-2,lem:eta_etaplus-2}),
where we exploit round-off errors to obtain such ``contraction''
(\cref{lem:contraction2} in \cref{sec:techlemma_revised}).
We then map $x_1$ to $c_1$ and $x_2$ to $c_2$, and apply $\sigma$ to the result so that
the final network maps all inputs smaller than $z$ to $\sigma(c_1) = 0$
and all other inputs to $\sigma(c_2) = K$ (\cref{lem:sigmaetatoc1-2}).
We construct $\sigma$-networks for $K \iota_{\leq z}$ and $K \iota_{>\eta}$ analogously.

To prove \cref{lem:indc-to-iua}, we construct $\sigma$-networks for the scaled indicator functions of
every box in $([a,b]_\fpq)^d$ (\cref{lem:indc-box}) and
every subset of $([a,b]_\fpq)^d$ (\cref{lem:indc-set}),
using the indicator functions constructed in \cref{lem:indc}.
We construct the final $\sigma$-network (i.e., universal interval approximator)
as a floating-point linear combination of the $\sigma$-networks
that represent the scaled indicator functions of the level sets of the target function (\cref{lem:iua}).}

\subsection{Proof of Lemma~\ref{lem:indc}}
\label{sec:pflem:indc}

To prove \cref{lem:indc}, we assume that the activation function $\sigma : \efpq \to \efpq$
satisfies \cref{cond:activation2} with some constants $c_1, c_2, \eta \in \fpq$.
By \cref{cond:activation2}, the constants $\eta$ and $K \defeq \sigma(c_2)$ clearly satisfy
the range condition in \cref{lem:indc}.
Hence, it remains to show the $([-1, 1]_\fpq, \eta, K, L_\phi, L_\psi)$-separability of $\sigma$
for some $L_\phi, L_\psi \in \bbN$.
This requires us to construct $\sigma$-networks
$\psi_{>\eta}$ and $\phi_{\leq z}, \phi_{\geq z}$ for every $z \in [-1, 1]_\fpq$ such that
$\smash{\psi_{> \eta}^\sharp} = \smash{(K \iota_{> \eta})^\sharp}$,
$\smash{\phi_{\leq z}^\sharp} = \smash{(K \iota_{\leq z})^\sharp}$, and
$\smash{\phi_{\geq z}^\sharp} = \smash{(K \iota_{\geq z})^\sharp}$ on $\bbI_{[-1, 1]}$ (\cref{def:separability}).

We first construct $\psi_{>\eta}$ using \cref{lem:sigmaindicator-2,lem:sigmaetatoc1-2} (\cref{fig:indicator}).
The proofs of these lemmas, presented in \cref{sec:pflem:sigmaindicator-2,sec:pflem:sigmaetatoc1-2},
rely heavily on \cref{cond:orig-1,cond:orig-2,cond:orig-3} of \cref{cond:activation2}.

\begin{lemma}
  \label{lem:sigmaindicator-2}
  There exists a $\sigma$-network $\mu : \efpq \to \efpq$ without the first affine layer
  such that
  $\smash{\mu^\sharp} \left(\intv{-\fmax , \eta} \right) = \intv{\eta,\eta}$,
  $\smash{\mu^\sharp}( \intv{\eta^+ , \fmax}) = \intv{\eta^+,\eta^+}$,
  and $\smash{\mu^\sharp}(\intv{-\fmax,\fmax})=\intv{\eta,\eta^+}$.
\end{lemma}

\begin{lemma}
  \label{lem:sigmaetatoc1-2}
  Let $(\theta,\theta')$ be either $(c_1,c_2)$ or $(c_2,c_1)$.
  Then, there exists a depth\=/2 $\sigma$\=/network $\tau_{\theta,\theta'} : \efpq \to \efpq$
  without the first affine layer such that
  $\smash{\tau_{\theta,\theta'}^\sharp}(\intv{\eta,\eta}) = \intv{\theta,\theta}$,
  $\smash{\tau_{\theta,\theta'}^\sharp}(\intv{\eta^+,\eta^+}) = \intv{\theta',\theta'}$,
  and $\smash{\tau_{\theta,\theta'}^\sharp}(\intv{\eta,\eta^+})
  = \intv{\min\set{\theta,\theta'}, \max\set{\theta,\theta'}}$.
\end{lemma}


\begin{figure}[t]
  \centering
  \includegraphics[width=0.9\linewidth]{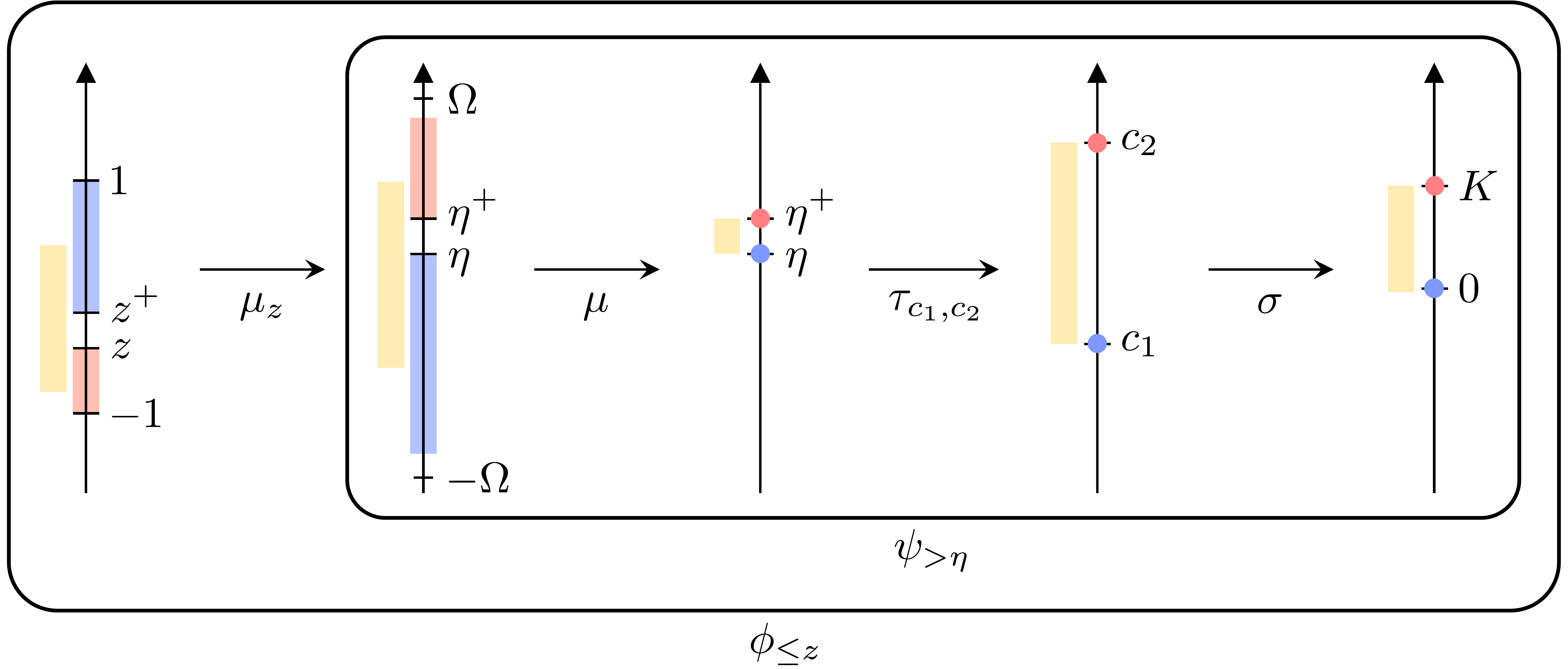}
  \caption{%
    Illustration of networks $\mu$, $\tau_{c_1,c_2}$, $\mu_z$
    (\cref{lem:sigmaindicator-2,lem:sigmaetatoc1-2,lem:eta_etaplus-2})
    and $\psi_{>\eta}$, $\phi_{\leq z}$ (\cref{eq:psi-construct,eq:phi-construct}),
    assuming (b) in \cref{lem:eta_etaplus-2}.
    A box/dot denotes an abstract interval.
  }
  \label{fig:indicator}
\end{figure}

\cref{lem:sigmaindicator-2} states that
we can construct a $\sigma$-network $\mu$ without the first affine layer,
whose interval semantics maps
all finite (abstract) intervals left of $\eta$ to the singleton interval $\intv{\eta, \eta}$,
all finite intervals right of $\eta^+$ to $\intv{\eta^+, \eta^+}$, and
all the remaining finite intervals to $\intv{\eta, \eta^+}$.
Similarly, \cref{lem:sigmaetatoc1-2} shows that
there exists a $\sigma$-network $\tau_{\theta,\theta'}$ without the first affine layer,
whose interval semantics maps $\intv{\eta, \eta}$ to $\intv{\theta, \theta}$,
$\intv{\eta^+, \eta^+}$ to $\intv{\theta', \theta'}$,
and $\intv{\eta, \eta^+}$ to the interval between $\theta$ and $\theta'$.
By composing these networks with $\sigma$, we construct $\psi_{>\eta}$ as
\begin{align}
  \label{eq:psi-construct}
  \psi_{>\eta} \defeq \sigma\circ\tau_{c_1,c_2}\circ\mu.
\end{align}
This function $\psi_{> \eta}$ is a $\sigma$-network without the first and last affine layers,
since $\tau_{c_1, c_2}$ are $\mu$ are without the first affine layer.
Moreover, $\smash{\psi_{>\eta}^\sharp} = \smash{(K\iota_{>\eta})^\sharp}$ on $\bbI_{[-1,1]}$
by the aforementioned properties of $\tau_{c_1, c_2}$ and $\mu$,
and by the next properties of $\sigma$ from \cref{cond:orig-1} of \cref{cond:activation2}:
$\sigma(c_1) = 0$, $\sigma(c_2) = K$, and $\sigma(x)$ lies between them for all $x$ between $c_1$ and $c_2$.
Lastly, we choose $L_\psi$ as the depth of $\psi_{>\eta}$.

We next construct $\phi_{\le z}$ and $\phi_{\ge z}$ using \cref{lem:eta_etaplus-2} (\cref{fig:indicator}).
The proof of this lemma is provided in \cref{sec:pflem:eta_etaplus-2}.

\begin{lemma}
  \label{lem:eta_etaplus-2}
  Let $z \in \fpq$ with $|z| \leq 1^+$.
  Then, there exists a depth-1 $\sigma$-network ${\mu_z} : \efpq \to \efpq$
  such that one of the following holds.
  \begin{itemize}
  \item[(a)]
    $\gamma\paren[\big]{ \smash{\mu_z^\sharp}(\intv{-1,z})  } \subset [-\fmax,\eta]$ and
    $\gamma\paren[\big]{ \smash{\mu_z^\sharp}(\intv{z^+,1}) } \subset [\eta^+,\fmax]$.
  \item[(b)]
    $\gamma\paren[\big]{ \smash{\mu_z^\sharp}(\intv{-1,z})  } \subset [\eta^+,\fmax]$ and
    $\gamma\paren[\big]{ \smash{\mu_z^\sharp}(\intv{z^+,1}) } \subset [-\fmax,\eta]$.
  \end{itemize}
\end{lemma}

\cref{lem:eta_etaplus-2} ensures the existence of a depth-1 $\sigma$-network $\mu_z$,
whose interval semantics
maps $\intv{-1, z}$ and $\intv{z^+, 1}$ to an interval left of $\eta$ and an interval right of $\eta^+$.
By composing $\mu_z$ with the previous networks \mbox{$\tau_{\theta, \theta'}$ and $\mu$,
we construct $\phi_{\le z}$ as}
\begin{align}
  \label{eq:phi-construct}
  \phi_{\le z} \defeq
  \begin{cases}
    \sigma\circ\tau_{c_2,c_1}\circ\mu\circ\mu_z & \text{if (a) holds in \cref{lem:eta_etaplus-2}}
    \\
    \sigma\circ\tau_{c_1,c_2}\circ\mu\circ\mu_z & \text{if (b) holds in \cref{lem:eta_etaplus-2}}.
  \end{cases}
\end{align}
By a similar argument used above,
the function $\phi_{\le z}$ is a $\sigma$-network without the last affine layer,
and it satisfies the desired equation:
$\smash{\phi_{\le z}^\sharp} = \smash{(K\iota_{\leq z})^\sharp}$ on $\bbI_{[-1,1]}$.
We construct $\phi_{\ge z}$ analogously, but using $\mu_{z^-}$ instead of $\mu_z$.
Since the depths of $\phi_{\le z}$ and $\phi_{\ge z}$ are identical for all $z$,
we denote this depth by $L_\phi$.
This completes the construction of $\psi_{>\eta}$, $\phi_{\leq z}$, and $\phi_{\geq z}$,
finishing the proof of \cref{lem:indc}.

\subsection{Proof of Lemma~\ref{lem:indc-to-iua}}
\label{sec:pflem:indc-to-iua}

To prove \cref{lem:indc-to-iua}, we assume that the activation function $\sigma$
is $([a,b]_{\fpq}, \allowbreak \eta, \allowbreak K, \allowbreak L_\phi, \allowbreak L_\psi)$-separable
for some $\eta, K \in \fpq$ with $|\eta|\in[2^{\emin+5}, \allowbreak 4-8\feps]$
and $|K|\in[\tfrac{\feps}{2}+2\feps^2, \allowbreak \tfrac{5}{4}-2\feps]$.
Given this, we construct a $\sigma$-network
whose interval semantics exactly computes that of the target function
$h : \smash{\efpq{}^d} \to \smash{\efpq} \setminus \set{\bot}$ for all abstract boxes in $[a,b]^d$.
In our construction, we progressively implement the following functions using $\sigma$-networks:
(i) scaled indicator functions of arbitrary boxes,
(ii) scaled indicator functions of arbitrary sets,
and (iii) the target function.

We first construct a $\sigma$-network $\tilde{\nu}_\mcB$, for any abstract box $\mcB$ in $[a,b]^d$,
that implements the scaled indicator function $K \iota_\mcB$ under the interval semantics.

\begin{lemma}
  \label{lem:indc-box}
  For any $\mcB \in \smash{(\bbI_{[a,b]})^d}$,
  there exists a depth-$L$ $\sigma$-network $\tilde{\nu}_{\mcB} : \efpq{}^d \to \efpq$ without the last affine layer
  such that $\smash{\tilde{\nu}_{\mcB}^\sharp} = \smash{(K\iota_{\mcB})^\sharp}$ on $\smash{(\bbI_{[a,b]})^d}$,
  where $L \defeq L_\phi + (L_\psi-1)(\lceil\log_{2^M}d\rceil+1)$.
\end{lemma}

In the proof of \cref{lem:indc-box}, we design $\tilde\nu_{\mcB}$
using the networks $\psi_{>\eta}$, $\phi_{\le z}$, and $\phi_{\ge z}$ constructed in \cref{sec:pflem:indc}.
Specifically, for an abstract box $\mcB=(\intv{a_1,b_1}, \allowbreak \dots, \allowbreak \intv{a_d,b_d})$,
we define a $\sigma$-network $\tilde{\nu_i} : \efpq \to \efpq$ as
\begin{align}
  \tilde\nu_i(x)
  \defeq \psi_{>\eta} \paren[\Big]{
    \paren[\Big]{ \alpha \otimes \phi_{\ge a_i}(x) } \oplus
    \paren[\Big]{ \alpha \otimes \phi_{\le b_i}(x) } \oplus \beta
  },
\end{align}
where $\alpha, \beta \in \fpq$ are constants such that
$\beta \leq \eta$,
$(\alpha \otimes K) \oplus \beta \leq \eta$, and
$(\alpha \otimes K) \oplus (\alpha \otimes K) \oplus \beta > \eta$.
Then, we can show that
$\smash{\tilde{\nu}_i^\sharp} = \smash{(K\iota_{\intv{a_i,b_i}})^\sharp}$ on $\bbI_{[a,b]}$.
When $d$ is small (e.g., $d \leq 2^{M+1}$),
we construct $\tilde{\nu}_{\mcB}$ using $\tilde{\nu}_i$ and $\psi_{> \eta}$, as follows:
\begin{align}
  \tilde{\nu}_{\mcB}(x_1,\dots,x_d)
  \defeq \psi_{>\eta} \paren[\Bigg]{
    \paren[\Bigg]{ \bigoplus_{i=1}^d \alpha' \otimes \tilde{\nu}_i(x_i) } \oplus \beta'
  },
\end{align}
where $\alpha',\beta'\in\fpq$ are suitably chosen so that
$\smash{\tilde{\nu}_{\mcB}^\sharp} = \smash{(K\iota_{\mcB})^\sharp}$ on $\smash{(\bbI_{[a,b]})^d}$.
When $d$ is large (e.g., $d>2^{M+1}$),
this construction does not work
since $\smash{\bigoplus_{i=1}^d} \alpha'\otimes\tilde\nu_i(x_i)$ may not be computed as we want
due to rounding errors (e.g., $\bigoplus_{i=1}^{n} 1 = 2^{M+1} < n$ for all $n > 2^{M+1}$).
In such a case, we construct $\tilde{\nu}_\mcB$ hierarchically using more layers, but based on a similar idea.
A rigorous proof of \cref{lem:indc-box},
including the proof that appropriate $\alpha, \alpha', \beta, \beta' \in \fpq$ exist,
is presented in \cref{sec:pflem:indc-box}.

Using $\tilde\nu_{\mcB}$,
we next construct a $\sigma$-network $\tilde{\nu}_{\mcS}$, for any set $\mcS$ in $([a, b]_\fpq)^d$,
whose interval semantics computes that of the scaled indicator function $K \iota_\mcS$.

\begin{lemma}
  \label{lem:indc-set}
  Suppose that for any $\mcB \in \smash{(\bbI_{[a,b]})^d}$,
  there exists a depth-$L$ $\sigma$-network $\tilde{\nu}_{\mcB}$ without the last affine layer
  such that $\smash{\tilde{\nu}_{\mcB}^\sharp} = \smash{(K\iota_{\mcB})^\sharp}$ on $\smash{(\bbI_{[a,b]})^d}$.
  Then, for any $\mcS \subseteq \smash{([a,b]_{\fpq})^d}$,
  there exists a depth-$(L+L_\psi-1)$ $\sigma$-network $\smash{\tilde{\nu}_{\mcS}} : \efpq{}^d \to \efpq$
  without the last affine layer
  such that $\smash{\tilde{\nu}_\mcS^\sharp} = \smash{(K\iota_{\mcS})^\sharp}$ on $\smash{(\bbI_{[a,b]})^d}$.
\end{lemma}

In the proof of \cref{lem:indc-set},
we construct ${\tilde{\nu}_{\mcS}}$ using $\tilde{\nu}_\mcB$ and $\psi_{>\eta}$, as follows:
\begin{align}
  \tilde{\nu}_{\mcS}(\bfx)
  \defeq \psi_{>\eta} \paren[\Bigg]{
    \paren[\Bigg]{ \bigoplus_{\mcB\in\mcT} \alpha'' \otimes \tilde\nu_{\mcB}(\bfx) } \oplus \eta
  },
\end{align}
where $\mcT$ denotes the collection of all abstract boxes in $\mcS$,
and $\alpha'' \in \fpq$ is a constant such that
$\eta < \paren[\big]{ \smash{\bigoplus_{i=1}^n} \alpha'' \otimes K } \oplus \eta < \infty$ for all $n \geq 1$.
We remark that it is possible to make the summation not overflow even for a large $n$,
by cleverly exploiting the rounding errors from $\oplus$. 
With a proper choice of $\alpha''$, we can further show that
$\smash{\tilde{\nu}_{\mcS}^\sharp} = \smash{(K\iota_{\mcS})^\sharp}$ on $\smash{(\bbI_{[a,b]})^d}$.
A formal proof of \cref{lem:indc-set} is given in \cref{sec:pflem:indc-set}.

Using $\tilde\nu_{\mcS}$, we finally construct a $\sigma$-network
that coincides, under the interval semantics, with the target function $h$ over $([a,b]_\fpq)^d$.
This result (\cref{lem:iua}) and the above results (\cref{lem:indc-box,lem:indc-set})
directly imply \cref{lem:indc-to-iua}.

\begin{lemma}
  \label{lem:iua}
  Assume that for any $\mcS \subseteq \smash{([a,b]_{\fpq})^d}$,
  there exists a depth-$L'$ $\sigma$-network $\smash{\tilde{\nu}_{\mcS}}$ without the last affine layer
  such that $\smash{\tilde{\nu}_\mcS^\sharp} = \smash{(K\iota_{\mcS})^\sharp}$ on $\smash{(\bbI_{[a,b]})^d}$.
  Then, for any $h : \efpq{}^d \to \efpq \setminus \set{\bot}$, 
  there exists a $\sigma$-network $\nu : \efpq{}^d \to \efpq$ such that
  \mbox{$\smash{\nu^\sharp} = \smash{h^\sharp}$ on $\smash{(\bbI_{[a,b]})^d}$.}
\end{lemma}

We now illustrate the main idea of the proof of \cref{lem:iua}.
For a simpler argument, we assume that $h$ is non-negative;
the proof for the general case is similar (see \cref{sec:pflem:iua}).
Let $0 = z_0 < z_1 < \cdots < z_{n} = +\infty$ be all non-negative floats (except $\bot$) in increasing order,
and let $\mcS_i \defeq \set{\bfx \in \smash{([a,b]_{\fpq})^d} \mid h(\bfx)\ge z_i}$
be the level set of $h$ for $z_i$.
Under this setup, we construct $\nu$ using $\tilde{\nu}_{\mcS_i}$, as follows:
\begin{align}
  \nu(\bfx) \defeq \bigoplus_{i=1}^{m} \alpha_i \otimes \tilde{\nu}_{\mcS_i}(\bfx),
\end{align}
where $m \in \bbN \cup \set{0}$ and $\alpha_i \in \fpq$ are chosen so that
$z_m = \max\set{ h(\bfx) \mid \bfx \in \smash{([a,b]_\fpq)^d} }$ and
$\alpha_i \otimes K \approx z_i-z_{i-1}$ for all $i \in [m]$.
If $\alpha_i\otimes K$ is close enough to $z_i-z_{i-1}$,
then the floating-point summation $\smash{\bigoplus_{i=1}^k} \alpha_i \otimes K$ is exactly equal to
the exact summation $\smash{\sum_{i=1}^k} z_{i}-z_{i-1} = z_k$ for all $k \in [m]$,
by the rounding errors of $\oplus$.
Using this observation,
we can show that $\nu(\bfx)=h(\bfx)$ for all $\bfx \in \smash{([a,b]_\fpq)^d}$,
and more importantly, $\smash{\nu^\sharp} = \smash{h^\sharp}$ on $\smash{(\bbI_{[a,b]})^d}$.
The full proof of \cref{lem:iua} is in \cref{sec:pflem:iua}.


\section{Related Work}
\label{sec:related-work}

\paragraph{\bf Universal approximation.}

Universal approximation theorems for neural networks are widely studied in
the literature, which include results for
feedforward networks~\citep{cybenko89,hornik89,hornik90,pinkus99},
convolutional networks~\citep{zhou20},
residual networks~\citep{lin18},
and transformers~\citep{yun21}.
With the advent of low-precision computing for neural networks
(e.g., 8-bit E5M2, 8-bit E4M3~\citep{Wang2018,micikevicius2022};
float16~\citep{micikevicius2018}; bfloat16~\citep{tensorflow16}),
there has been growing interest among researchers in characterizing
their expressiveness power in this setting.
New UA theorems for ``quantized'' neural networks, which use
finite-precision network parameters with \textit{exact} real arithmetic,
have been studied in \citep{gonon2023,ding2018}.
These networks differ from the floating-point networks considered in this work,
because our networks use \textit{inexact} floating-point arithmetic.
%

\tcr{%
To the best of current knowledge, \citep{park24,hwang25b} are the only works that
study UA theorems for floating-point neural networks.
\citep{park24} proves UA theorems for ReLU and step activation functions.
Our IUA \cref{thm:main}, by virtue of \cref{eq:ua},
is a strict generalization of \citep{park24} in two senses:
(i) it applies to a much broader class of activations that satisfy \cref{cond:activation2},
which subsumes ReLU and step functions; and
(ii) it provides a result for abstract interpretation via interval
analysis, of which the pointwise approximation considered in \citep{park24} is a special case.
Concurrent with this article,
\citep{hwang25b} generalizes \citep{park24} to support a wider range of activation functions and larger input domains.
Our \cref{thm:main} partially subsumes \citep{hwang25b} in that
it is a result for interval approximation, whereas \citep{hwang25b} considers only pointwise approximation.
Conversely, a special case of our \cref{thm:main} for pointwise approximation (i.e., \cref{eq:ua})
is subsumed by \citep{hwang25b}
in that it applies to smaller classes of activation functions and input domains.%
}

\paragraph{\bf Interval universal approximation.}

%
The first work to establish an IUA theorem for neural networks used
interval analysis with the ReLU activation~\citep{baader20}, which was
later extended to the more general class of so-called ``squashable''
activation functions~\citep{wang2022interval}.
Whereas these previous IUA theorems assume the neural network can compute over
arbitrary real numbers with infinitely precise real arithmetic, the IUA
result (\cref{thm:main}) in this work applies to ``machine-implementable''
neural networks that use floating-point numbers and operations.
To the best our knowledge, no previous work has established
an IUA theorem for floating-point neural networks.
These different computational models lead to substantial differences in
both the proof methods (cf.~\cref{sec:iua-main-result,sec:iua-proof}) and
the specific technical results---\cref{sec:comparison} gives a detailed
discussion of how \cref{thm:main} differs from previous IUA and robustness
results~\citep[Theorem 1.1]{baader20}; \citep[Theorem 3.7]{wang2022interval}.

\paragraph{\bf \tcr{Provable robustness.}}
There is an extensive literature on robustness verification and robust training
for neural networks, which is surveyed in, e.g.,
\citep[Chapter~1]{baader20}; \citep{Li2023,Singh2023}.
Notable methods among these works are~\citep{Singh2018,Singh2019}, which
verify the robustness of a neural network using abstract interpretation
with the zonotope and polyhedra domains for a restricted class of activations,
and are sound with respect to floating-point arithmetic.
Compared to these methods, our contribution is a theoretical result on the
inherent expressiveness of provably robust floating-point
networks under the interval domain for a broad class of activation
functions, rather than new verification algorithms or abstract domains.
Indeed, our existence result directly applies to the zonotope and polyhedra
domains, as they are more precise than the interval domain.
More specific IUA theorems tailored to these domains may yield more
compact constructions that witness the existence of a provably
robust floating-point neural network.
Recently, \citep{Jin24} shows that
even if a neural network is provably robust over real arithmetic,
it can be non-robust over floating-point arithmetic and remain vulnerable to adversarial attacks.
This highlights the importance of establishing robustness in the floating-point setting.

\begin{credits}
\subsubsection{\ackname}

G.~Hwang and Y.~Park were supported by
Korea Institute for Advanced Study (KIAS)
Individual Grants AP092801 and AP090301,
via the Center for AI and Natural Sciences at KIAS.
G.~Hwang was also supported by
National Research Foundation of Korea (NRF) Grants
RS-2025-00515264 and RS-2024-00406127, funded by the
Korea Ministry of Science and ICT (MSIT);
and the Gwangju Institute of Science and Technology (GIST)
Global University Project in 2025.
Y.~Park was also supported by the Sejong
University faculty research fund
in 2025.
S.~Park was supported by
the Korea
Institute of Information \& Communications Technology Planning \& Evaluation (IITP)
Grant RS-2019-II190079,
funded by the Korea MSIT;
the Information Technology Research Center (IITP-ITRC) Grant
IITP-2025-RS-2024-00436857,
funded by the Korea MSIT;
and the Culture, Sports, and Tourism R\&D Program through the Korea
Creative Content Agency (KOCCA) Grants RS-2024-00348469 and RS-2024-00345025,
funded by the Korea Ministry of Culture, Sports and
Tourism (MCST) in 2024.
%
W.~Lee and F.~Saad were supported by the United States
National Science Foundation (NSF) under Grant No.~2311983 and funds from
the Computer Science Department at Carnegie Mellon University.
Any opinions, findings, and conclusions or recommendations expressed
in this material are those of the authors and do not necessarily
reflect the views of the funding agencies.

\subsubsection{\discintname}
The authors have no competing interests to declare that are relevant to
the content of this article.
\end{credits}

\newpage
\bibliographystyle{splncs04}
\bibliography{paper}
\addcontentsline{toc}{section}{References}

\newpage
\appendix

\section{Preliminaries for the Appendix}

\subsection{Notation}

We introduce additional conventions and notations that are used throughout the appendix.
First, we interpret any number in the form of $b = b_0.b_1 b_2 \ldots$ ($b_i \in \set{0,1}$)
as a binary expansion, unless otherwise specified:
\begin{equation*}
   b =  b_0 + b_1 \times 2^{-1} + b_2\times 2^{-2} +\cdots.
\end{equation*}
Second, we interpret floating-point addition and summation operators ($\oplus$ and $\bigoplus$)
in the left-associative way, even when they are mixed together. For instance,
\begin{equation*}
    z \oplus \bigoplus_{i=1}^{n} x_i
    \defeq
    (\cdots ((z \oplus x_1) \oplus x_2) \cdots) \oplus x_n.
\end{equation*}
That is, we first expand out all the floating-point addition operations
and then perform each operation from left to right.
\todo{WL. For the next version: Use a different notation for
  $a \oplus \bigoplus_i b_i$, $a \oplus \bigoplus_i b_i \oplus \bigoplus_j c_j$, etc.}
Lastly, in the interval semantics, we abuse notation so that $c \in \fpq$ denotes the abstract interval $\intv{c,c} \in \bbI$. For instance,
\begin{equation*}
    \langle a,b \rangle \oplus^\sharp c = \langle a,b\rangle \oplus^\sharp \langle c,c\rangle.
\end{equation*}

\subsection{Relaxed Version of Condition~\ref{cond:activation2}}

In the appendix, we use a relaxed version of \cref{cond:activation2} to simplify the proof.

\begin{condition}
\label{cond:activation_2r}
An activation function $\sigma:\efpq\to\efpq$ satisfies the following conditions:

\begin{enumerate}[leftmargin=3em, label={\rm{(C\arabic*${}^+$)}}, ref={\rm{(C\arabic*${}^+$)}}]
\item \label{cond:relx-1}
There exist $c_1,c_2 \in \fpq$ such that $\sigma(c_1) = 0$,
\begin{align*}
  \sigma(c_2)\in \mcR
  & \defeq [-\tfrac{5}{4}+2\feps, -\tfrac{\feps}{2}-2\feps^2)]_{\fpq}
  \cup [\tfrac{\feps}{2}+2\feps^2, \tfrac{5}{4}-2\feps]_{\fpq}
  \\
  & \; = \left[(1+2^{-\mbit+1})\times2^{-\mbit-2},1+2^{-2}-2^{-\mbit}\right]_{\fpq} \\
  & \qquad \cup \left[-(1+2^{-2}-2^{-\mbit}),-(1+2^{-\mbit+1})\times2^{-\mbit-2}\right]_{\fpq},
\end{align*}
and $\sigma(x)$ lies between $\sigma(c_1)$ and $\sigma(c_2)$ for all $x$ \mbox{between $c_1$ and $c_2$.}

\item \label{cond:relx-2}
  There exists $\eta \in \fpq$ with  $\eta\in(-(2^2)^-,(2^2)^- )_{\fpq}$ such that
  \begin{itemize}
  \item
    $5+ \emin \le \expo{\eta} \le 1$,
  \item
    $\emin+ 5 \le  \max \{ \expo{\sigma(\eta)}, \expo{\sigma(\eta^+)}  \}  \le \expo{\eta} + \emax -5$, and
  \item
    for all $x,y\in \fpq$ with $x\le \eta<\eta^+ \le y$,
    \begin{align*}
        \sigma(x)\le \sigma(\eta)<\sigma(\eta^+)\le \sigma(y)
        \qquad\text{or}\qquad
        \sigma(x)\ge \sigma(\eta)>\sigma(\eta^+)\ge \sigma(y).
    \end{align*}
  \end{itemize}

\item \label{cond:relx-3}
   There exists $\lips \in[0, 2^{\emax-6} \cdot 2^{\min\{ \max\{ \expo{\sigma(\eta)}, \expo{\sigma(\eta^+)} \}   ,\mbit+2\} }]$
    such that
    for any $x,y\in \fpq$ with $x\le \eta<\eta^+ \le y$,
\begin{align*}
    &|\sigma(x) - \sigma(\eta)| \le \lips |x-\eta|
    \qquad\text{and}\qquad
    |\sigma(y) - \sigma(\eta^+)| \le \lips |y - \eta^+|.
\end{align*}
\end{enumerate}
\end{condition}

\begin{lemma}\label{lem:activation_2r}
    If $\sigma$ satisfies \cref{cond:activation2}, then $\sigma$ satisfies \cref{cond:activation_2r}.
\end{lemma}

\begin{proof}
We prove each of \cref{cond:relx-1,cond:relx-2,cond:relx-3} as follows.

\begin{itemize}[leftmargin=3em]
\item[\cref{cond:relx-1}]
Because $\feps= 2^{-\mbit-1}$, we have the desired result.

\item[\cref{cond:relx-2}]
Because
\begin{align*}
    | \sigma(\eta) |,| \sigma(\eta^+) | \le 2^{\emax-6} \cdot |\eta| \le 2^{\emax-5+\expo{\eta}},
\end{align*}
we have the desired result.

\item[\cref{cond:relx-3}]
Because
\begin{align*}
    2^{\emax-7} \times |\sigma(\eta)| &\le  2^{\emax -7+(\expo{\sigma(\eta)}+1) } \le 2^{\emax-6+\expo{\sigma(\eta)} } \le 2^{\emax-6+\max \{ \expo{\sigma(\eta)} ,  \expo{\sigma(\eta^+)}\} }, \\
    2^{\emax-7}  \times 2^{\mbit+3} &= 2^{\emax-6} \times 2^{\mbit+2},
\end{align*}
we have the desired result.

\end{itemize}
\end{proof}

\clearpage
\section{Proofs of the Results in \S\ref{sec:iua-conditions}}

\subsection{Proof of Lemma~\ref{lem:activation}}
\label{sec:pflem:activation}

\paragraph{\bf Proof of \cref{cond:suff-1} $\implies$ \cref{cond:orig-1}.}

Suppose $\rho$ satisfies \cref{cond:suff-1}. Since $\mbit \ge 3$, we have $\tfrac{\feps}{2}+2\feps^2$, $\tfrac{5}{4}-2\feps \in \fpq$. Hence $\round{\rho(c_2')} \in [\tfrac{\feps}{2}+2\feps^2, \tfrac{5}{4}-2\feps]$. In addition, since $ -\tfrac{\fmin}{2} \le\rho(c_1') \le \tfrac{\fmin}{2}$, we have $ \round{\rho(c_1')} = 0$.
Since $\round{\cdot}$ is order-preserving, $\round{\rho(x)}$ lies between $\round{\rho(c_1')}$ and $\round{\rho(c_2')}$ for all $x$ between $c_1'$ and $c_2'$.
Therefore $\round{\rho}$ satisfies \cref{cond:orig-1} of \cref{cond:activation2} with $c_1=c_1'$ and $c_2=c_2'$.

\paragraph{\bf Proof of \cref{cond:suff-2} $\implies$ \cref{cond:orig-2}.}

First suppose
\begin{equation*}
        \rho(x)\le \rho(\delta-\tfrac18)<\rho({\delta+\tfrac18})\le \rho(y),
\end{equation*}
for all $x,y\in\bbR$ satisfying $x\le \delta-\tfrac18<\delta+\tfrac18 \le y$.
Since $|\rho(x)-\rho(y)|>\frac18|x-y|$  for all $x,y\in[\delta-\tfrac1{8},\delta+\tfrac1{8}]$, $\rho$ is monotonically increasing on $[\delta - \tfrac{1}{8},\delta + \tfrac{1}{8}]$. Hence $\round{\rho}$ is also monotonically increasing on $[\delta - \tfrac{1}{8},\delta + \tfrac{1}{8}] \cap \fpq$.

We define $\eta \in \fpq$ as
\begin{align*}
    \eta \defeq \min \{ t \in (\delta - \tfrac{1}{8},\delta + \tfrac{1}{8}) \cap \fpq : \round{\rho(t)} < \round{\rho(t^+)} \}.
\end{align*}
Since the minimum of the distance between the floating-point numbers in $[2^{-2},1]$ or $[-1,-2^{-2}]$ is $ 2^{-\mbit-2}$ and
\begin{align*}
     \rho(\delta+\tfrac{1}{8}) - \rho(\delta-\tfrac{1}{8})   > \frac{1}{8} \times \frac{1}{4} = 2^{-5} \ge  2^{-\mbit-2},
\end{align*}
there exist $\gamma_1,\gamma_2 \in [\delta - \tfrac{1}{8},\delta + \tfrac{1}{8}] \cap \fpq$ such that $\round{\rho(\gamma_1)} \ne \round{\rho(\gamma_2)}$. Hence $\eta$ is well-defined. Note that since $\eta \in (\delta - \tfrac{1}{8},\delta + \tfrac{1}{8}) $ we have $|\eta| \in ( \tfrac{1}{4} , 1) \subset [2^{\emin+5},4-8\feps]$. Also note that since $\tfrac{1}{4},1\in\fpq$, $|\round{\rho(\eta)} |, |\round{\rho(\eta^+)} | \in [\tfrac{1}{4},1]$ leading to
\begin{align*}
   |\round{\rho(\eta)} |, |\round{\rho(\eta^+)} | &\in [\tfrac{1}{4},1] \subset [2^{-9},2^{7}] \subset [2^{\emin+5},2^{\emax-8}] \\
   & \subset  [2^{\emin+5},2^{\emax-6}\cdot \eta].
\end{align*}

Let $x,y \in \fpq$, with $x \le \eta < \eta^+ \le y$.
Since $\rho(x) \le \rho(\delta + \tfrac18) \le \rho(\eta)$, we have
\begin{align*}
    \round{\rho(x)} \le \round{\rho(\delta-\tfrac18)} \le \round{\rho(\eta)}.
\end{align*}
If $ \eta^+ \le y \le \delta + \tfrac18$, since $\round{\rho}$ is monotonically increasing on $[\delta - \tfrac{1}{8},\delta + \tfrac{1}{8}] \cap \fpq$,  we have
\begin{align*}
     \round{\rho(\eta^+)} \le  \round{\rho(y)}.
\end{align*}
If $y \ge \delta + \tfrac18$, since $\rho(y) \ge \rho(\delta + \tfrac18) \ge \rho(\eta^+)$ we have
\begin{align*}
      \round{\rho(\eta^+)} \le \round{\rho(\delta-\tfrac18)}    \le \round{\rho(y)}.
\end{align*}
Symmetrically, if
\begin{equation*}
        \rho(x)\ge \rho(\delta-\tfrac18)<\rho({\delta+\tfrac18})\ge \rho(y),
\end{equation*}
we have
\begin{align*}
    \round{\rho(x)} \ge \round{\rho(\eta)}>\round{\rho(\eta^+)}\ge \round{\rho(y)},
\end{align*}
for $x \le \eta < \eta^+ \le y$.

\paragraph{\bf Proof of \cref{cond:suff-2} and \cref{cond:suff-3} $\implies$ \cref{cond:orig-3}.}

Let $\lips_1 = \max \{ \lips , 2^{\emin+1}\}$. Since $\rho$ is $\lips$-Lipschitz, we have
\begin{align*}
    | \rho(\alpha) - \rho(\beta) | \le \lips | \alpha-\beta| \le \lips_1 | \alpha-\beta|.
\end{align*}
for $\alpha,\beta \in \bbR$.

Let $x,y \in \fpq$ where $x \le \eta < \eta^+ \le y$.

First, suppose $\round{\rho(x)}$ is normal.
Let $ C = \max \{ |x|, |\eta| \}$. Since $| \round{t} - t | \le |t| \times 2^{\mbit}$ for normal $t \in \fpq$, we have
\begin{align*}
    | \round{\rho(x)}  - \round{\rho(\eta)} | &\le  | \round{\rho(x)}  - \rho(x) | + | \rho(x) - \rho(\eta) | + | \rho(\eta) -  \round{\rho(\eta)}  | \\
    &\le \lips_1 ( |x| + |\eta|) \times 2^{-\mbit} +  \lips_1 |x-\eta| \\
    &\le 2 \lips_1 C \times 2^{-\mbit} +  \lips_1|x-\eta|  \le 5 \lips_1|x-\eta|,
\end{align*}
where we use
\begin{align*}
      | x - \eta | \ge C \times 2^{-\mbit -1}.
\end{align*}
Now, suppose $\round{\rho(x)}$ is subnormal. Since $| \round{t} - t | \le \tfrac12 \times \fmin$ for subnormal $t \in \fpq$, we have
\begin{align*}
    | \round{\rho(x)}  - \round{\rho(\eta)} | &\le
    \tfrac12 \times \fmin + \lips_1|\eta| \times 2^{-\mbit} +  \lips_1|x-\eta| \\
    &\le 2 \lips_1 C \times 2^{-\mbit} +  \lips_1|x-\eta|  \le 5 \lips_1|x-\eta|,
\end{align*}
where we use
\begin{align*}
     2^{\emin+1}  \le \lips_1, \quad  C \ge |\eta| \ge 2^{-2}, \quad \tfrac12 \fmin \le  \lips_1 C \times 2^{-\mbit}.
\end{align*}
Therefore, we have
\begin{align*}
    | \round{\rho(x)}  - \round{\rho(\eta)} | \le 5 \lips_1 | x -\eta| \le \tilde{\lips} | x -\eta|,
\end{align*}
where
\begin{align*}
    \tilde{\lips} = 5\lips_1 \in  [0, 2^{\emax-9} ] =  [0, 2^{\emax-7} \cdot \min\{|\sigma(\eta)|,2^{\mbit+3}\}].
\end{align*}

Similarly, we can show
\begin{align*}
    | \round{\rho(\eta^+)}  - \round{\rho(y)} | \le \tilde{\lips} | \eta^+ -y |.
\end{align*}
\myqedd

\subsection{Proof of Corollary~\ref{cor:activation}}
\label{sec:pfcor:activation}

If $\rho$ satisfies the conditions of \cref{lem:activation}, $\rho$ satisfies \cref{cond:activation2}.
Since $\relu$, $\lrelu$, $\gelu$, $\elu$, $\mish$, $\SoftPlus$, $\Sigmoid$, and $\tanh$ are increasing on $[\tfrac{1}{4},1]$, we have
\begin{equation*}
        \rho(x)\le \rho(\delta-\tfrac18)<\rho({\delta+\tfrac18})\le \rho(y),
\end{equation*}
for some $\delta \in [\tfrac{3}{8}, \tfrac{7}{8}]$.

To check $\rho$ satisfy the conditions of \cref{lem:activation}, we need to check the following requirements.
\begin{itemize}
    \item $ |\rho(c_1')| \le \tfrac{\fmin}{2}$.
    \item $ |\rho(c_2')| \in [\tfrac{\feps}{2}+2\feps^2, \tfrac{5}{4}-2\feps]$.
    \item $ \delta \in [\tfrac{3}{8}, \tfrac{7}{8}]$.
    \item $ \max \{ |c_1|,|c_2| \} \ge 2^{\emin+1}=2^{-2^{\ebit-1}+3} \ge 2^{-13}$.
    \item $ |\rho(x)| \in [ \tfrac{1}{4} , 1]$ for $ x \in [\delta-\tfrac{1}{8}, \delta +\tfrac{1}{8}]$.
    \item $\inf_{\tfrac{1}{4}\le x\le 1} |\rho'(x)| > \tfrac{1}{8}$.
    \item $\lips \le \frac{1}{5} \cdot 2^{\emax-9}$.
\end{itemize}
If $\mbit \ge 3$ and $\ebit \ge 5$, according to \cref{table:floating_format}, we have
\begin{align*}
     -\fmax &\le -32768, \; [0.0391,1.125]  \subset [\tfrac{\feps}{2}+2\feps^2, \tfrac{5}{4}-2\feps], \\
     \tfrac \fmin 2 &\le 7.63 \times 10^{-6}, \;  \tfrac{1}{5} \cdot 2^{\emax-9} \ge 12.8.
\end{align*}
Hence, the above requirements are satisfied by \cref{table:lip} under  the condition $\mbit \ge 3$, $\ebit \ge 5$ as well as for various floating-point formats presented in \cref{table:floating_format}.

To verify $|\rho(-\fmax)| \le \tfrac \fmin 2$ for $\Sigmoid$ and $\SoftPlus$, it is sufficient to show
\begin{align*}
    \log \left( \rho(- 2^{2^{\ebit-1}} ) \right)  \le (-2^{\ebit-1}-\mbit +1) \log 2,
\end{align*}
since $\rho$ is monotonically increasing on $x<0$ and
\begin{align*}
    \rho(-\fmax)= \rho(-(2-2^{-\mbit})\times 2^{2^{\ebit-1}} ) \le \rho(- 2^{2^{\ebit-1}} ) \le  2^{-2^{\ebit-1}-\mbit +1} = \tfrac \fmin 2.
\end{align*}
Note that
\begin{align*}
    |\Sigmoid(x)| &= |\frac{1}{1+e^{-x}}| \le e^x, \quad x < 0, \\
    |\SoftPlus(x)| &= | \log(1+e^x) | \le e^x, \quad x < -1,
\end{align*}
and
\begin{align*}
(2^{\ebit-1}+\mbit -1) \log 2  \le (2^\ebit-1)\log2 \le 2^{\ebit+1} \le 2^{2^{\ebit-1}},
\end{align*}
which is due to $n \le 2^{n-1} -1$ for $3 \le n \in \bbN$ (Note that $\ebit \ge 5$, $\mbit \le 2^{\ebit-1}$ ).
Therefore we have
\begin{align}
    \log \left( \rho(- 2^{2^{\ebit-1}} ) \right) \le -2^{2^{\ebit-1}}  \le (-2^{\ebit-1}-\mbit +1) \log 2, \label{eq:cor_sigmoid}
\end{align}
for $\rho=\Sigmoid$ or $\rho=\SoftPlus$.

Finally, since $\relu$, $\lrelu$, $\gelu$, $\elu$, $\mish$ and $\tanh$ are increasing on $[0,1]$, $\SoftPlus$ and $\Sigmoid$ are increasing on $[-\infty,1]$,
and $\rho(\cdot)$ is order-preserving, $\rho(x)$ lies between $\rho(c_1')$ and $\rho(c_2')$ for all $x$ between $c_1'$ and $c_2'$.
\myqedd

\begin{table}[h]
\caption{%
  Properties of various activation functions for verifying the conditions.
  $D$ denotes $[\delta-\tfrac{1}{8}, \delta +\tfrac{1}{8}]$ and
  $\text{Lip}(\rho)$ denotes the Lipschitz constant of $\rho$.
  For $\Sigmoid$ and $\SoftPlus$, we show $|\rho(c_1')| \le \tfrac{\fmin}{2}$ in  \cref{eq:cor_sigmoid}.
  The numbers in the table are represented in decimal form and rounded to two decimal places.
}
\vspace{0.3cm}
\setlength{\tabcolsep}{3pt}
\hspace{-8pt}
\begin{tabular}{cccccc@{\;\;\;}ccccc}
\toprule
\begin{tabular}{@{}c@{}}
  Activation \\ function
\end{tabular}
& $c_1'$ & $ |\rho(c_1')|$ & $c_2'$ & $\rho(c_2')$ & $\delta$ & $\displaystyle
\inf_{x \in D} |\rho(x)|$ & $\displaystyle \sup_{x \in D} |\rho(x)|  $ & $\displaystyle \inf_{1/4 \le x\le 1} |\rho'(x)|$   &$\text{Lip}(\rho)$ &  \\
\midrule
$\relu$     & 0        & 0                     & 1 & 1    & 0.5 & $0.37$ & $0.63$ & 1    & 1    \\
$\lrelu$    & 0        & 0                     & 1 & 1    & 0.5 & $0.37$ & $0.63$ & 1    & 1    \\
$\gelu$     & 0        & 0                     & 1 & 0.84 & 0.6 & $0.32$ & $0.57$ & 0.70 & 1.13 \\
$\elu$      & 0        & 0                     & 1 & 1    & 0.5 & $0.37$ & $0.63$ & 1    & 1    \\
$\mish$     & 0        & 0                     & 1 & 0.87 & 0.5 & $0.26$ & $0.49$ & 0.75 & 1.09 \\
$\SoftPlus$ & $-\fmax$ & \cref{eq:cor_sigmoid} & 1 & 1.31 & 0.4 & $0.84$ & $0.99$ & 0.56 & 1    \\
$\Sigmoid$  & $-\fmax$ & \cref{eq:cor_sigmoid} & 1 & 0.73 & 0.5 & $0.59$ & $0.66$ & 0.20 & 0.25 \\
$\tanh$     & 0        & 0                     & 1 & 0.76 & 0.5 & $0.35$ & $0.56$ & 0.42 & 1    \\
\bottomrule
\end{tabular}
\label{table:lip}
\vspace{0.2in}
\caption{%
  Properties of floating-point format for verifying the conditions.
The numbers in the table are represented in decimal form and rounded to two decimal places.
}
\vspace{0.3cm}
\centering
\begin{tabular}{@{}ccccccccccc@{}}
\toprule
\begin{tabular}{@{}c@{}}\end{tabular} Format name  & $\ebit$   & $ \mbit$ & $-\fmax$ & $[\tfrac{\feps}{2}+2\feps^2, \tfrac{5}{4}-2\feps]$ & $\tfrac \fmin 2 $ & $ \tfrac{1}{5} \cdot 2^{\emax-9}$\\
\midrule
9-bit format & 5  & 3  & ${<}\,{-2^{15}}$   & [0.039,1.125]                & $7.63 \times 10^{-6}$   & 12.8 \\
bfloat16     & 8  & 7  & ${<}\,{-2^{127}}$  & [$1.93\times10^{-3}$,1.24]   & $9.18 \times 10^{-41}$  & $6.65\times10^{34}$ \\
float16      & 5  & 10 & ${<}\,{-2^{15}}$   & [$2.54\times10^{-4}$ 1.25]   & $5.96 \times 10^{-8}$   & 12.8 \\
float32      & 8  & 23 & ${<}\,{-2^{127}}$  & [$2.98\times10^{-8}$,1.25]   & $1.40 \times 10^{-45}$  & $6.65\times10^{34}$ \\
float64      & 11 & 52 & ${<}\,{-2^{1023}}$ & [$5.55\times10^{-17}$, 1.25] & $5.00 \times 10^{-324}$ & $3.51\times10^{304}$\\
\bottomrule
\end{tabular}
\label{table:floating_format}
\end{table}

\clearpage
\section{Proofs of the Results in \S\ref{sec:iua-implications}}
\label{sec:aaa}

\subsection{Proof of Theorem~\ref{thm:provable-robustness}}
\label{sec:pfthm:provable-robustness}

First, consider any $g : \efpq{}^d \to \efpq{}^n$ such that
for every $\bfx_0 \in \mcX$ and $\bfx \in \mcN_\delta(\bfx_0)$,
\begin{align}
  \label{eqthm:provable-robustness-g}
  g(\bfx) = (\underbrace{0, \ldots, 0}_{\mathclap{\text{$\class(f(\bfx_0))-1$}}}, 1, 0, \ldots, 0).
\end{align}
Then, $g$ makes the same prediction as $f$ on $\mcX$ by \cref{eqthm:provable-robustness-g},
and $g$ is $\delta$-robust on $\mcX$ since $f$ does so.
For each $i \in [n]$, let $g_i : \efpq{}^d \to \efpq{}$ be the $i$-th component of $g$: $g_i(\bfx) \defeq g(\bfx)_i$.
Then, by \cref{thm:main}, there exist $\sigma$-neural networks $\nu_1, \ldots, \nu_n : \efpq{}^d \to \efpq$ such that
for every $i \in [n]$ and $\mcB \in \bbI^d$ in $[-1,1]^d$,
\begin{align}
  \label{eqthm:provable-robustness-nui}
  \gamma( \nu_i^\sharp(\mcB) ) = \big[ \min g_i(\gamma(\mcB)), \max g_i(\gamma(\mcB)) \big] \cap \efpq.
\end{align}

Next, define $\nu : \efpq{}^d \to \efpq{}^n$ by a $\sigma$-neural network that stacks up $\nu_1, \ldots, \nu_n$ such that
for every $\mcB \in \bbI^d$,
\begin{align}
  \label{eqthm:provable-robustness-nu}
  \nu^\sharp(\mcB) = (\nu_1^\sharp(\mcB), \ldots, \nu_n^\sharp(\mcB)).
\end{align}
We can construct such $\nu$ because $\nu_1, \ldots, \nu_n$ have the same depth by the proof of \cref{thm:main}.
Then, $\nu$ makes the same prediction as $g$, and thus as $f$, by \cref{eqthm:provable-robustness-nui,eqthm:provable-robustness-nu}.
Moreover, we claim that $\nu$ is $\delta$-provably robust on $\mcX$.
To prove this, let $\bfx_0 \in \mcX$ and $\mcB \in \bbI^d$ with $\gamma(\mcB) = \mcN_{\delta}(\bfx_0)$.
Then,
\begin{align}
  \gamma(\nu^\sharp(\mcB))
  &= \prod_{i=1}^n \big[ \min g_i(\gamma(\mcB)), \max g_i(\gamma(\mcB)) \big] \cap \efpq
  \\
  &= {\set{0} \times \cdots \times \set{0}}
  \times \set{1} \times \set{0} \times \cdots \times \set{0},
\end{align}
where the first equality is by \cref{eqthm:provable-robustness-nui,eqthm:provable-robustness-nu}
and the second equality is by \cref{eqthm:provable-robustness-g}.
Since $\gamma(\nu^\sharp(\mcB))$ is a singleton set,
$\bfy, \bfy' \in \gamma(\nu^\sharp(\mcB))$ clearly implies $\class(\bfy) = \class(\bfy')$, as desired.
\myqedd

\subsection{Proof of Theorem~\ref{thm:turing}}
\label{sec:pfthm:turing}

\cref{thm:turing} is a direct corollary of \cref{lem:indc-to-iua} and the following lemma.
\myqedd

\begin{lemma}\label{lem:turing-indc}
Let $\sigma:\fpq\to\fpq$ be the identity function, i.e., $\sigma(x)=x$ for all $x\in\fpq$.
Then, $\sigma$ is $(\fpq,1,1,L_\phi,L_\psi)$-separable for some $L_\phi,L_\psi\in\bbN$.
\end{lemma}
\begin{proof}

We define $f_0$ as
\begin{align*}
    f_0(x) \defeq  \sigma\left( 2^{-1} \otimes \sigma \left( x  \oplus \fmin \right) \right).
\end{align*}
Then we have
\begin{align*}
    f_0^\sharp( \langle 0 ,  0 \rangle) &=  \langle 0 , 0 \rangle, \quad f_0^\sharp( \langle \fmin ,  \fmin \rangle) = \langle \fmin, \fmin \rangle, \\
    f_0^\sharp(\langle -2\fmin , 0 \rangle ) &=  \langle 0, 0 \rangle, \quad f_0^\sharp(\langle 2\fmin , 4\fmin \rangle ) =  \langle 2\fmin, 2\fmin \rangle.
\end{align*}

\paragraph{\fbox{\bf Case 1: $x > \fmin$ with $x = n \cdot \fmin$ for $n \in \bbN$.}}~

Define $\mcI_i$ as
\begin{align}
    \mcI_0 &\defeq [2,2^{\mbit+2} -1],  \\
    \mcI_k &\defeq [2^{\mbit+1+k},2^{\mbit+2+k} -1], \quad k \in \bbN.
\end{align}

If $n \in \mcI_0$, we have
\begin{align*}
    f_0(x) &= \sigma\left( 2^{-1} \otimes \sigma \left( (n+2) \fmin \right) \right) \le  \lrp{ \tfrac{n}{2}+ \tfrac{3}{2}} \fmin.
\end{align*}
Since the solution of the recurrence relation $a_{i+1} = \tfrac{1}{2}a_i + \tfrac{3}{2}$ is $a_i = (a_0-3)(\tfrac{1}{2})^i + 3$,
pick $m_1 \in  \bbN$ such that
\begin{align*}
    m_1 \ge  \ceil*{ \frac{\log( 2^{\mbit+2} -4)}{\log(2)} }_{\bbZ}.
\end{align*}
we have
\begin{align*}
    a_m = (a_0-3)(\tfrac{1}{2})^{m_1} + 3 \le (2^{\mbit+2} -4)(\tfrac{1}{2})^{m_1} + 3\le   1+ 3 = 4,
\end{align*}
which leads to $f^{\circ (m_1)}(x) \le 4 \fmin $ and $f^{\circ (m_1+2)}(x) = \fmin$ for $ 2\fmin \le x < 2^{\mbit+2}$.

If $n \in \mcI_k$, we have
\begin{align*}
    f_0(x) &= \sigma\left( 2^{-1} \otimes \sigma \left( n \fmin \right) \right) =   \tfrac{n}{2} \fmin.
\end{align*}
Hence $f_0 (x) = n_2 \fmin $ for $n_2 \in \mcI_{k-1}$.

Therefore we have $f_0^{\circ (\emax -\emin - 1 )}(x) = n_3 \fmin $ for $n_3 \in \mcI_0$ which leads to
$f_0^{\circ (\emax -\emin +m_1 + 1 )}(x) = \fmin $
for $\fmin \le x \le \fmax$.

\paragraph{\fbox{\bf Case 2: $x \le 0 $ with $x = - n \fmin $ for $n \in \bbN$.}}~

If $n \in \mcI_0$, we have
\begin{align*}
    f_0(x) &= \sigma\left( 2^{-1} \otimes \sigma \left( (-n) \fmin \right) \right) \ge  ( -\tfrac{n}{2}- \tfrac{1}{2} ) \fmin.
\end{align*}
Since the solution of the recurrence relation $b_{i+1} = \tfrac{1}{2}b_i + \tfrac{1}{2}$ is $b_i = (a_0-1)(\tfrac{1}{2})^i + 1$,
pick $m_2 \in  \bbN$ such that
\begin{align*}
    m_2 \ge  \ceil*{ \frac{\log( 2^{\mbit+2} -2)}{\log(2)} }_{\bbZ},
\end{align*}
we have
\begin{align*}
    a_m = (a_0-1)\lrp{\tfrac{1}{2}}^{m_1} + 1 \le (2^{\mbit+2} -2)\lrp{\tfrac{1}{2}}^{m_1} + 1\le 2,
\end{align*}
which leads to $f^{\circ (m_2)}(x) \ge -  2\fmin $ and $f^{\circ (m_1+1)}(x) = 0$ for $ 2^{\mbit+2} < x \le 0$.

If $n \in \mcI_k$, we have
\begin{align*}
    f_0(x) &= \sigma\left( 2^{-1} \otimes \sigma \left( -n \fmin \right) \right) =   - \tfrac{n}{2} \fmin.
\end{align*}

Therefore we have $f_0^{\circ (\emax -\emin - 1 )}(x) = -n_4 \fmin $ for $n_4 \in \mcI_0$ which leads to
$f_0^{\circ (\emax -\emin +m_2 + 1 )}(x) = \fmin $
for $-\fmax \le x \le 0$.
We define $g_0$ as $f_0(x) \defeq f^{\circ o(\emax -\emin +\max\{m_1,m_2\} + 1)}$, and we have
\begin{align*}
    g_0^\sharp( \langle -\fmax ,  0 \rangle) &=  \langle 0 , 0 \rangle, \quad g_0^\sharp( \langle \fmin ,  \fmax \rangle) = \langle \fmin, \fmin \rangle, \quad   g_0^\sharp( \langle -\fmax ,  \fmax \rangle) = \langle 0, \fmin \rangle.
\end{align*}
For $z \in \fpq$. we define $g_z$ as
\begin{align*}
    g_z(x) \defeq \begin{cases}
     g_0\left( \sigma( x \ominus z ) \right) \quad &\text{if} \quad  |z| < 2^{\emax -\mbit - 1} ,  \\
     g_0\left( \sigma( 2^{-1} \otimes x \ominus  \frac{z}{2} \right) \quad &\text{if} \quad |z| \ge 2^{\emax -\mbit - 1} ,
    \end{cases}
\end{align*}
and we have
\begin{align*}
    g_z^\sharp( \langle -\fmax ,  z \rangle) &=  \langle 0 , 0 \rangle, \quad g_z^\sharp( \langle z^+ ,  \fmax \rangle) = \langle \fmin, \fmin \rangle,  \quad g_0^\sharp( \langle -\fmax ,  \fmax \rangle) = \langle 0, \fmin \rangle.
\end{align*}
Finally we define $\iota_{> z}$ as
\begin{align*}
     \iota_{> z}(x) &\defeq \sigma \left( 2^{-\emin} \otimes  \sigma \left( 2^{-\mbit} \otimes g_z(x) \right) \right), \\
  \iota_{ \ge z}(x) &\defeq \sigma \left( 2^{-\emin} \otimes  \sigma \left( 2^{-\mbit} \otimes g_{z^-}(x) \right) \right), \\
   \iota_{< z}(x) &\defeq \sigma \left( 2^{-\emin} \otimes  \sigma \left( 2^{-\mbit} \otimes g_{-z}(-x) \right) \right), \\
  \iota_{\le z}(x) &\defeq \sigma \left( 2^{-\emin} \otimes  \sigma \left( 2^{-\mbit} \otimes g_{(-z)^-}(-x) \right) \right).
\end{align*}
\end{proof}

\clearpage
\section{Proofs of the Results in \S\ref{sec:pflem:indc}}
\label{sec:pflem:proof_results-0}

In \cref{lem:indc}, we suppose $\sigma$ satisfies \cref{cond:activation2}. By \cref{lem:activation_2r}, we suppose $\sigma$ satisfies \cref{cond:activation_2r}, the relaxed version of \cref{cond:activation2}.

\subsection{Proof of Lemma~\ref{lem:sigmaindicator-2}}
\label{sec:pflem:sigmaindicator-2}

Since $\sigma$ also satisfies \cref{cond:activation_2r}, it is sufficient to prove the following lemma (\cref{thm:sigmaindicator}).
To prove \cref{thm:sigmaindicator}, we need following preliminary technical lemmas: \cref{lem:subnorm_inverse,lem:telescoping,lem:contraction2} (presented in \cref{subsec:techlemma_for_sigmaindicator-2}) and \cref{lem:inverse} (presented in \cref{sec:common_techlemma}).
\myqedd

\begin{lemma}\label{thm:sigmaindicator}
    Suppose that $\sigma:\fpq\to\fpq$ satisfies \cref{cond:activation_2r}.
    Then, there exists a $\sigma$-network $f$ without the first and last affine layer such that
    \begin{align}
    f^\sharp \left( \langle -\fmax , \eta  \rangle \right) = \langle \sigma(\eta) ,
    \sigma(\eta) \rangle, \quad f^\sharp \left( \langle \eta^+ , \fmax \rangle \right) = \langle \sigma(\eta^+) , \sigma(\eta^+)  \rangle.
    \end{align}
\end{lemma}
\begin{proof}

Let $e_0 \in \bbZ$ such that $ 2^{e_0} \le | \sigma(\eta^+)-\sigma(\eta)| < 2^{e_0+1}$.
Then note that
\begin{align*}
    \max \{ \expo{\sigma(\eta)}, \expo{\sigma(\eta^+)}  \} - \mbit - 1  \le e_0 \le \max \{ \expo{\sigma(\eta)}, \expo{\sigma(\eta^+)}  \} + 1  .
\end{align*}
Define $\expo{\theta} , \expo{\zeta} \in \bbZ$, $\tilde{\lips}$ as
\begin{align*}
    \expo{\theta}  &\defeq \max\lrp{\emin-\mbit,\emin- e_0-\mbit+1}, \\
    \expo{\zeta} &\defeq \begin{cases}
    \expo{\eta} - \mbit - 1 &\text{ if } \eta >0\; \text{or} \; \eta <0, \eta \neq -2^{\expo{\eta}},
    \\ \expo{\eta} - \mbit - 2 &\text{ if } \eta <0,\eta = -2^{\expo{\eta}},
\end{cases} \\
\tilde{\lips} & \defeq \lips\times 2^{\expo{\theta}+2}.
\end{align*}

To use \cref{lem:contraction2}, we need to check the assumptions of \cref{lem:contraction2}: $\expo{\theta} \le  -3$ and $e_0\le \expo{\eta}-\mbit-3-\expo{\theta}$.
Since $\max \{ \expo{\sigma(\eta)}, \expo{\sigma(\eta^+)}  \} \ge \emin+ 5$, we have
\begin{align*}
    -e_0 + \emin -\mbit + 1 \le - \max \{ \expo{\sigma(\eta)}, \expo{\sigma(\eta^+)}  \} +\emin + 2 \le -3.
\end{align*}
Therefore we have $\expo{\theta} \le -3$.
To show $e_0\le \expo{\eta}-\mbit-3-\expo{\theta}$, first suppose $e_0 \ge 1$. Then we have $\expo{\theta} = \emin-\mbit$, which leads to
\begin{align*}
    e_0 &\le \max \{ \expo{\sigma(\eta)}, \expo{\sigma(\eta^+)}  \} + 1  \le \expo{\eta}-\emin - 3    = \expo{\eta}-\mbit-3-\expo{\theta}.
\end{align*}
Next suppose $e_0 \le 0$. Then we have $\expo{\theta} = \emin -e_0 - \mbit +1$ which leads to
\begin{align*}
   \expo{\eta}-\mbit-3-\expo{\theta} &=\expo{\eta}-\mbit-3 -(\emin -e_0 - \mbit + 1 ) \\
   &\ge (\emin+5) -\emin +e_0 - 4 > e_0.
\end{align*}

We can use \cref{lem:contraction2} and consider the following cases.

\paragraph{\fbox{\bf Case 1: $\sigma(\eta)< \sigma(\eta^+)$.}}~

By \cref{lem:contraction2}, there exists an affine transformation $g$ such that
\begin{align}
    g^\sharp\left( \langle -\fmax, \fmax \rangle  \right)\subset\langle -\fmax, \fmax \rangle, \quad  g^\sharp(\mcI) = \langle \eta,\eta \rangle, \quad g^\sharp(\mcI^+) = \langle \eta^+,\eta^+ \rangle, \label{eq:contraction2_thm4}
\end{align}
where
\begin{equation*}
    \mcI\defeq \langle \sigma(\eta)-2^{\expo{\zeta}-\expo{\theta} },\sigma(\eta)\rangle,  \mcI^+\defeq \langle\sigma(\eta^+),\sigma(\eta)+2^{\expo{\zeta}-\expo{\theta} } \rangle_{\fpq}.
\end{equation*}
In addition, if $x -\sigma(\eta) > 2^{\expo{\zeta}-\expo{\theta}}$,
\begin{equation}
    g(x) - \eta^+ \le \lrp{x-\sigma\lrp{\eta^+}}\times 2^{ \expo{\theta}+2}, \label{eq:gxetaplus}
\end{equation}
and if $x < \sigma(\eta)-2^{\expo{\zeta}-\expo{\theta}}$,
\begin{equation}
    \eta - g(x)\le \lrp{\sigma\lrp{\eta}-x}\times  2^{ \expo{\theta}+2}, \label{eq:etaminusgx}
\end{equation}

Since $e_0\ge \max \{ \expo{\sigma(\eta)} , \expo{\sigma(\eta^+)}\}  - M - 1$, we have
\begin{equation*}
    \expo{\theta}\le \max\lrp{\emin-\mbit, \emin - \max \left\{ \expo{\sigma(\eta)} , \expo{\sigma(\eta^+)}\right\} + 2},
\end{equation*}
which leads to
\begin{align*}
    \tilde{\lips} &= \lips\times 2^{\expo{\theta}+2}\le \lips \times 2^{\emin+4} \cdot 2^{\max \{-\mbit-2 , -\max \{ \expo{\sigma(\eta)} , \expo{\sigma(\eta^+)}\} \} }   \\
    &\le 2^{-1} \cdot 2^{\min\{ \max\{ \expo{\sigma(\eta)}, \expo{\sigma(\eta^+)} \}   ,\mbit+2\} }  \cdot 2^{\max \{-\mbit-2 , -\max \{ \expo{\sigma(\eta)} , \expo{\sigma(\eta^+)}\} \} }  \\
    &\le 1/2.
\end{align*}
Define $g_1$ as
\begin{equation*}
    g_1(x)\defeq \sigma(g(x)).
\end{equation*}
Then, if $x\ge \sigma\lrp{\eta^+}$, by \cref{eq:gxetaplus},
\begin{align*}
    g_1(x) -g_1\lrp{\sigma\lrp{\eta^+}} &= \sigma\left(g(x) \right) - \sigma\lrp{\eta^+} \le \lips( g(x)-\eta^+) \\
    &\le \lips{\lrp{x-\sigma\lrp{\eta^+}}}\times 2^{ \expo{\theta}+2}
    \le\tilde{\lips}  {\lrp{x-\sigma\lrp{\eta^+}}}.
\end{align*}
Similarly, if $x\le \sigma(\eta)$, by \cref{eq:etaminusgx},
\begin{equation*}
     g_1 \left(\sigma(\eta) \right) - g_1(x) \le  \tilde{\lips}\left( \sigma(\eta) - x \right).
\end{equation*}
Therefore, we define $n_1 \in \bbZ_{\ge 0}$ such that
\begin{align*}
n_1 \defeq \max \left\{ \ceil*{ \log_{\tilde{\lips}^{-1}}
\left(  \frac{\fmax - \sigma(\eta^+)}{\sigma(\eta) -\sigma(\eta^+) + 2^{\expo{\zeta}-\expo{\theta}}}  \right) }_{\bbZ} , \ceil*{ \log_{\tilde{\lips}^{-1}} \left( \frac{\sigma(\eta)+\fmax}{2^{\expo{\zeta}-\expo{\sigma}}} \right) }_{\bbZ} \right\}.
\end{align*}
Note that since $\sigma(\eta) -\sigma(\eta^+) + 2^{\expo{\zeta}-\expo{\theta}} \ge -2^{e_0}+ 2^{\expo{\zeta}-\expo{\theta}} > 0$, $n$ is well-defined.

We define $h_1(x)$ as $h_1(x) \defeq g_1^{\circ n_1}(x)$. Then we have
\begin{align*}
   h_1(x) - h_1\left( \sigma(\eta^+) \right) \le \tilde{\lips}^n \left(x - \sigma(\eta^+) \right)  \le \sigma(\eta) -\sigma(\eta^+) + 2^{\expo{\zeta}-\expo{\theta}} \;  &\text{ if } \;  \sigma(\eta^+) \le x\le \fmax, \\
 h_1\left(\sigma (\eta) \right) - h_1(x) \le \tilde{\lips}^n \left(\sigma(\eta) - x \right)  \le 2^{\expo{\zeta}-\expo{\sigma}} \;  &\text{ if } \;   -\fmax \le x\le \sigma(\eta),
\end{align*}
leading to
\begin{align*}
   h_1(x) \le \sigma(\eta)  + 2^{\expo{\zeta}-\expo{\theta}} \quad  &\text{ if } \quad  \sigma(\eta^+) \le x\le \fmax, \\
 h_1(x) \ge \sigma(\eta) - 2^{\expo{\zeta}-\expo{\sigma}} \quad  &\text{ if } \quad   -\fmax \le x\le \sigma(\eta).
\end{align*}

Finally, we define $f_1(x)$ as $f_1(x)\defeq g_1 \circ h_1 \circ \sigma(x) = g_1^{\circ (n+1)} \circ \sigma (x)$. Together with \cref{eq:contraction2_thm4}, we have
\begin{align*}
f_1^\sharp\left( \langle -\fmax, \eta \rangle \right) = \langle \sigma(\eta), \sigma(\eta) \rangle , \quad f_1^\sharp\left(  \langle \eta^+, \fmax \rangle \right) = \langle \sigma(\eta^+), \sigma(\eta^+) \rangle.
\end{align*}

\paragraph{\fbox{\bf Case 2: $\sigma(\eta) >  \sigma(\eta^+)$.}}~

By \cref{lem:contraction2}, there exists an affine transformation $g$ such that
\begin{align}
    g^\sharp\left( \langle -\fmax, \fmax \rangle  \right)\subset\langle -\fmax, \fmax \rangle, \quad  g^\sharp(\mcI) = \langle \eta,\eta \rangle, \quad g^\sharp(\mcI^+) = \langle \eta^+,\eta^+ \rangle. \label{eq:contraction2_thm4case2}
\end{align}
where
\begin{equation*}
  \theta \defeq -2^{\expo{\theta}}, \mcI^+ \defeq \langle\sigma(\eta)-2^{\expo{\zeta}-\expo{\theta} },\sigma\lrp{\eta^+}\rangle, \quad  \mcI\defeq \langle \sigma(\eta),\sigma(\eta)+2^{\expo{\zeta}-\expo{\theta} }\rangle,
\end{equation*}
In addition,
\begin{align}
     \eta - g(x) \le \lrp{x-\sigma\lrp{\eta}}\times  2^{ \expo{\theta}+2} \quad \quad &\text{for} \quad \sigma(\eta) +  2^{\expo{\zeta}-\expo{\theta} } \le  x \in \fpq \label{eq:etagx_case2},   \\
     g(x) - \eta^+\le \lrp{\sigma\lrp{\eta^+}-x}\times 2^{ \expo{\theta}+2} \quad &\text{for} \quad  \sigma(\eta)-2^{\expo{\zeta}-\expo{\theta} } \ge x \in \fpq  \label{eq:gxetaplus_case2} ,
\end{align}

Define $g_2$ as
\begin{equation*}
    g_2(x)\defeq \sigma(g(x)).
\end{equation*}
Then, if $ \sigma\lrp{\eta} \le x \in \fpq $, by \cref{eq:etagx_case2} we have
\begin{equation*}
     g_2 \left(\sigma(\eta) \right) - g_2(x) \le \tilde{\lips}\left( x - \sigma(\eta) \right) \le \tilde{\lips}\left( x - \sigma(\eta^+) \right) .
\end{equation*}
If $ \sigma(\eta^+) \ge x \in \fpq $, by \cref{eq:gxetaplus_case2} we have
\begin{align*}
    g_2(x) -g_2\lrp{\sigma\lrp{\eta^+}} \le \tilde{\lips} {\lrp{\sigma\lrp{\eta^+} - x}} \le \tilde{\lips}  {\lrp{\sigma\lrp{\eta} - x}} .
\end{align*}

We define $h$ as $h \defeq g_2 \circ g_2(x)$. Then we have
\begin{align*}
    h_2^\sharp \left( \mcI_2 \right) = \langle \sigma(\eta^+), \sigma(\eta^+) \rangle,  \quad  h_2^\sharp \left( \mcI_2^+ \right) = \langle \sigma(\eta), \sigma(\eta)  \rangle,
\end{align*}
where
\begin{align*}
   \mcI_2 = \langle \sigma(\eta) -  \tilde{\lips}^{-2}\left( 2^{\expo{\zeta}-\expo{\theta}} \right)  ,\sigma(\eta^+)\rangle,  \;  \mcI_2^+ = \langle \sigma(\eta) , \sigma(\eta)+ \tilde{\lips}^{-2}\left( 2^{\expo{\zeta}-\expo{\theta}} \right) \rangle.
\end{align*}
Hence
Therefore, we define $n_2 \in \bbZ_{\ge 0}$ such that
\begin{align*}
n_2 \defeq \max \left\{ \ceil*{ \log_{\tilde{\lips}^{-2}}
\left(  \frac{\fmax - \sigma(\eta^+)}{ 2^{\expo{\zeta}-\expo{\theta}}}  \right) }_{\bbZ} , \ceil*{ \log_{\tilde{\lips}^{-2}} \left( \frac{\sigma(\eta)+\fmax}{2^{\expo{\zeta}-\expo{\sigma}}} \right) }_{\bbZ} \right\},
\end{align*}
and define $f_2$ as $f_2(x) \defeq g_2 \circ h_2^{\circ (n_2)} \circ \sigma (x) = g_2^{\circ (2n_2+1) } \circ \sigma (x)$.
Together with \cref{eq:contraction2_thm4case2}, we have
\begin{align*}
f_2^\sharp\left( \langle -\fmax, \eta \rangle \right) = \langle \sigma(\eta), \sigma(\eta) \rangle , \quad f_2^\sharp\left(  \langle \eta^+, \fmax \rangle \right) = \langle \sigma(\eta^+), \sigma(\eta^+) \rangle,
\end{align*}
and this completes the proof.
\end{proof}

\subsection{Proof of Lemma~\ref{lem:sigmaetatoc1-2}}
\label{sec:pflem:sigmaetatoc1-2}

Since $\sigma$ also satisfies \cref{cond:activation_2r}, we have $|\sigma(\expo{\eta})||\in [(1+2^{-\mbit+1}\times 2^{-\mbit-2},1+2^{-2}-2^{-\mbit}]$ leading to the fact $\sigma(\eta)$ is normal. In addition, we have
$\max \{ \expo{\theta},\expo{\theta'} \} \ge \emin+1$ since $\max \{ |\theta| , |\theta'| \} \ge \tfrac{\fmin}{\feps}$.

By the lemma below (\cref{lemma:sigmaetatoc1}), we have the desired result.
To prove \cref{lemma:sigmaetatoc1}, we need following preliminary technical lemmas: \cref{lem:onebit_difference,lem:distribution_law,lem:residue_control,lem:residue_control_111} (presented in \cref{subsec:techlemma_for_sigmaetatoc1-2}) and \cref{lem:inverse} (presented in \cref{sec:common_techlemma}).
\myqedd

\begin{lemma}\label{lemma:sigmaetatoc1}
Suppose that $\sigma:\fpq\to\fpq$ satisfies \cref{cond:activation_2r}.
Let $\gamma_1,\gamma_2,\kappa_1,\kappa_2 \in \fpq$ with $ |\kappa_1| < |\kappa_2|$ with $\expo{\kappa_2} \ge \emin + 1$.  Suppose there exist $c_2 \in \fpq$ such that $K \defeq \sigma(c_2)$ is normal.
    Then, there exist $n \in \bbN$, $w_1,\alpha_i,z_i,b \in \fpq$ such that
\begin{align*}
    \left(w_1 \otimes \sigma(\gamma_1) \right) \oplus  \bigoplus_{i=1}^n  \left( \alpha_i \otimes \sigma(z_i) \right) \oplus b  &= \kappa_1, \\
    \left(w_1 \otimes \sigma(\gamma_2) \right) \oplus \bigoplus_{i=1}^n \left( \alpha_i \otimes \sigma(z_i) \right) \oplus b  &= \kappa_2.
\end{align*}

\end{lemma}

\begin{proof}
By \cref{lem:inverse}, either $\mant{K}^{\parallel} \in (2^{-1},1]_{\fpq} $ or $\mant{K}^{\dag} \in [2^{-1},1)_{\fpq}$ exists such that
\begin{align*}
    \mant{K}^{\dag} \otimes \mant{K} &= 1^- = 1- 2^{-\mbit-1}=0.\underbrace{1\dots 1}_{\mbit+1 \text{ times }}, \\
    \mant{K}^{\parallel} \otimes \mant{K} &= 1.
\end{align*}
Let define $\tilde{\kappa} \in \fpq$ such that
\begin{align*}
    \tilde{\kappa} \oplus \kappa_1 = \kappa_2.
\end{align*}
Since $\expo{\kappa_2} \ge 1 + \emin$, we have $\expo{\tilde{\kappa}} \ge \emin$.
Let $ \sigma(\bar{\gamma}) = \max \{ |\sigma(\gamma_1)|,| \sigma(\gamma_2)| \}$.

By \cref{lem:onebit_difference}, there exists $w_1 \in \fpq$ such that
\begin{align*}
    w_1 \otimes \left( \sigma(\gamma_2) \ominus \sigma(\gamma_1) \right) + \epsilon_1 = \tilde{\kappa} , \quad \expo{w_1} \le \expo{\tilde{\kappa}} -e_0,
\end{align*}
for some $\epsilon_1 = 0$ or $\pm 2^{-\mbit + \expo{\tilde{\kappa}}}$.
Then by \cref{lem:distribution_law},
\begin{align*}
    w_1 \otimes \sigma(\gamma_2) =  w_1 \otimes \left( \sigma(\gamma_2) \ominus \sigma(\gamma_1) \right) + \left(  w_1 \otimes   \sigma(\gamma_1) \right) + ( C \times 2^{-\mbit-1}),
\end{align*}
where
\begin{align*}
|C| &\le |w_1| \left ( (2+2^{-\mbit-1}) |\sigma(\eta^+)+\sigma(\eta) | + |\sigma(\gamma_2)| + |\sigma(\gamma_1)|\right) \\
&\le (6+2^{-\mbit})  |w_1| |\sigma(\bar{\gamma})| \le 7 \times 2^{\expo{\tilde{\kappa}+1}}.
\end{align*}
Hence we have
\begin{align*}
    w_1 \otimes \sigma(\gamma_2) =  w_1 \otimes \left( \sigma(\gamma_2) \ominus \sigma(\gamma_1) \right) + \left(  w_1 \otimes   \sigma(\gamma_1) \right) + \epsilon_2,
\end{align*}
for some $|\epsilon_2| \le 7 \times 2^{-\mbit+\expo{\tilde{\kappa}}}$.
Therefore we have
\begin{align*}
    w_1 \otimes \sigma(\gamma_2) = \tilde{\kappa}  +  \left(  w_1 \otimes   \sigma(\gamma_1) \right) + \epsilon_3,
\end{align*}
where $ |\epsilon_3| \le 8 \times 2^{-\mbit + \tilde{\kappa}}$. Since $\mbit \ge 3$, the exponents of $w_1 \otimes \sigma(\gamma_1)$ and $w_1 \otimes \sigma(\gamma_2)$ are at most $\expo{\tilde{\kappa}}+2$.

We consider the following cases.

\paragraph{\fbox{\bf Case 1: $\mant{K}^{\parallel}$ exists.}}~

Let $\zeta = w_1 \otimes \sigma(\gamma_1)$. For $i = 0 , \dots , \mbit$, we define $\alpha_i, z_i,b$ as
\begin{align*}
   (\alpha_i,z_i) = \begin{cases}
(0,c_2) \; &\text{if} \;  \mant{\zeta,1} = 0 \\
(- \text{sign}(\zeta) \times \mant{K}^\parallel \times 2^{-i+\expo{\zeta}} , c_2) \; &\text{if} \; \mant{\zeta,1} = 1
\end{cases} , \quad  b = \kappa_1.
\end{align*}
Then since $  \bigoplus_{i=0}^\mbit \left( \alpha_i \otimes \sigma(z_i) \right)  = \zeta$, we have
\begin{align*}
    \left(w_1 \otimes \sigma(\gamma_1) \right) \oplus  \bigoplus_{i=0}^\mbit  \left( \alpha_i \otimes \sigma(z_i) \right) &= 0.
\end{align*}

Since the exponents of $w_1 \otimes \sigma(\gamma_1)$ and $w_1 \otimes \sigma(\gamma_2)$ are at most $\expo{\tilde{\kappa}}+2$, we have
\begin{align*}
    \left(w_1 \otimes \sigma(\gamma_2) \right) \oplus \bigoplus_{i=0}^\mbit \left( \alpha_i \otimes \sigma(z_i) \right)   =  \left(w_1 \otimes \sigma(\gamma_2) \right) + \zeta +\epsilon_4, \quad | \epsilon_4| &\le 2^{\expo{\tilde{\kappa}}+2}, \\
    \left(w_1 \otimes \sigma(\gamma_2) \right) \oplus \bigoplus_{i=0}^\mbit \left( \alpha_i \otimes \sigma(z_i) \right) \oplus b  = \tilde{\kappa} + \epsilon_5, \quad |\epsilon_5| &\le 12 \times 2^{-\mbit + \tilde{\kappa}}.
\end{align*}
By \cref{lem:residue_control}, there exist $\delta_1, \dots , \delta_{36} \in \fpq$ such that
\begin{align*}
    \bigoplus_{i=1}^{36} \delta_i &= 0, \\
    \tilde{\kappa} + \epsilon_4 \oplus \bigoplus_{i=1}^{36} \delta_i &= \tilde{\kappa},
\end{align*}
where $\delta_i = \pm 2^{\expo{\delta_i}}$.
Hence, let $(\alpha_{i+\mbit},z_{i+\mbit}) = (\text{sign}(\delta_i) \times \mant{K}^\parallel \times 2^{\expo{\delta_i}}, c_2)$.
Then,
\begin{align*}
    &\left(w_1 \otimes \sigma(\gamma_1) \right) \oplus  \bigoplus_{i=0}^\mbit  \left( \alpha_i \otimes \sigma(z_i) \right) \oplus \bigoplus_{i=\mbit+1}^{\mbit+37}  \left( \alpha_i \otimes \sigma(z_i) \right) \oplus b  = \kappa_1, \\
    &\left(w_1 \otimes \sigma(\gamma_2) \right) \oplus \bigoplus_{i=0}^\mbit \left( \alpha_i \otimes \sigma(z_i) \right) \oplus \bigoplus_{i=\mbit+1}^{\mbit+37} \oplus b = (\tilde{\kappa} + \epsilon_2) \oplus \bigoplus_{i=\mbit+1}^{\mbit+37} \delta_i \oplus \kappa_1 \\
    &= \tilde{\kappa}\oplus \kappa_1 = \kappa_2.
\end{align*}

\paragraph{\fbox{\bf Case 2: $\mant{K}^{\dag}$ exists.}}~

Let $\zeta_0 = w_1 \otimes \sigma(\gamma_1)$.
We recursively define $\zeta_i,\alpha_i,z_i$ as
\begin{align*}
    (\alpha_i,z_i) &\defeq (\mant{K}^\dag \times 2^{-i+\expo{\zeta_{i-1}}} , c_2), \quad \zeta_{i} \defeq \zeta_{i-1} \ominus 0.\underbrace{1\dots 1}_{\mbit+1 \text{ times }} \times 2^{\expo{\zeta_{i-1}}} , \; b =  \kappa_1.
\end{align*}
Then we have
\begin{align*}
    \alpha_i \otimes \sigma(z_i) = 0.\underbrace{1\dots 1}_{\mbit+1 \text{ times }} \times 2^{\expo{\zeta_{i-1}}}.
\end{align*}

Let $ n_2 = \min \{ n \in \bbN : \expo{\zeta_i} \le -4 -2 \mbit + \expo{\tilde{\kappa}} \}$.

Note that  $  \bigoplus_{i=1}^{n_2} \left( \alpha_i \otimes \sigma(z_i) \right)  = \zeta \pm 2^{\expo{\zeta}-\mbit}$. Then we have
\begin{align*}
    |\left(w_1 \otimes \sigma(\gamma_1) \right) \oplus  \bigoplus_{i=1}^{n_2}  \left( \alpha_i \otimes \sigma(z_i) \right)| \le (2-2^{-\mbit}) \times 2^{-4-2\mbit +\expo{\tilde{\kappa}}}.
\end{align*}

Since the exponents of $w_1 \otimes \sigma(\gamma_1)$ and $w_1 \otimes \sigma(\gamma_2)$ are at most $\expo{\tilde{\kappa}}+2$, we have
\begin{align*}
    \left(w_1 \otimes \sigma(\gamma_2) \right) \oplus \bigoplus_{i=1}^{n_2} \left( \alpha_i \otimes \sigma(z_i) \right)    =  \left(w_1 \otimes \sigma(\gamma_2) \right) + \zeta +\epsilon_6, \quad &| \epsilon_6| \le 2 \times  2^{\expo{\tilde{\kappa}}+2}, \\
    \left(w_1 \otimes \sigma(\gamma_2) \right) \oplus \bigoplus_{i=1}^{n_2} \left( \alpha_i \otimes \sigma(z_i) \right) \oplus b  = \tilde{\kappa} + \epsilon_7, \quad & |\epsilon_7| \le 16 \times 2^{-\mbit + \tilde{\kappa}}.
\end{align*}
By \cref{lem:residue_control_111}, there exist $\delta_1, \dots , \delta_{48} \in \fpq$ such that
\begin{align*}
    |\left(w_1 \otimes \sigma(\gamma_1) \right) \oplus  \bigoplus_{i=1}^{n_2}  \left( \alpha_i \otimes \sigma(z_i) \right)| \oplus \bigoplus_{i=1}^{48} \delta_i &= 0, \\
    \tilde{\kappa} + \epsilon_7 \oplus \bigoplus_{i=1}^{48} \delta_i &= \tilde{\kappa},
\end{align*}
where $\delta_i = \pm 2^{\expo{\delta_i}}$.
Hence let $(\alpha_{i+\mbit},z_{i+\mbit}) = (\text{sign}(\delta_i) \times \mant{K}^\dag \times 2^{\expo{\delta_i}}, c_2)$.
Then,
\begin{align*}
    &\left(w_1 \otimes \sigma(\gamma_1) \right) \oplus  \bigoplus_{i=1}^{n_2}  \left( \alpha_i \otimes \sigma(z_i) \right) \oplus \bigoplus_{i=n_2+1}^{n_2+49}  \left( \alpha_i \otimes \sigma(z_i) \right) \oplus b  = \kappa_1, \\
    &\left(w_1 \otimes \sigma(\gamma_2) \right) \oplus \bigoplus_{i=1}^{n_2} \left( \alpha_i \otimes \sigma(z_i) \right) \oplus \bigoplus_{i=n_2+1}^{n_2+49} \oplus b = (\tilde{\kappa} + \epsilon_2) \oplus \bigoplus_{i=n_2+1}^{n_2+49} \delta_i \oplus \kappa_1 \\
    &= \tilde{\kappa}\oplus \kappa_1 = \kappa_2.
\end{align*}
\end{proof}

\subsection{Proof of Lemma~\ref{lem:eta_etaplus-2}}
\label{sec:pflem:eta_etaplus-2}

To prove \cref{lem:eta_etaplus-2}, we need the following preliminary technical lemma: \cref{lem:special_case} (presented in \cref{subsec:techlemma_for_eta_etaplus-2}) and \cref{lem:inverse} (presented in \cref{sec:common_techlemma})

Define $\mu_z(\cdot)$ as $\mu_z(x) = (w \otimes x)\oplus b$. Since $ \eta\in [-4+8\feps ,4 - 8\feps]_{\fpq}, |\eta| \ge \tfrac{\fmin}{\feps}$ we have $\expo{\eta^-},\expo{\eta},\expo{\eta^+} \le 1$.

We represent $\eta,z$ as  $\eta = \mant{\eta} \times 2^{\expo{\eta}}$, $z = \mant{z} \times 2^{\expo{z}}$ with $ 1+ \emin  \le e_\eta \le 1$.

Let  $ c_z = \max \{  \emin - \expo{z} , 0 \}  \ge 0$. Note that $\expo{z}+c_z \ge \emin. $
By \cref{lem:inverse}, at least one of $\mant{z}^{\parallel}$ or $\mant{z}^{\ddag}$ exists.
We consider the following cases.

\paragraph{\fbox{\bf Case 1: $\eta > 0 \; z \ge 0$.}}~

In this case, (1) in \cref{lem:eta_etaplus-2} holds if
\begin{align*} w > 0, \quad
    \begin{cases}
        \mu_z(-1) &\ge -\fmax, \\
        \mu_z(z) &= \eta, \\
        \mu_z(z^+) &= \eta^+, \\
        \mu_z(1) &\le \fmax.
    \end{cases}
\end{align*}
Note that $z^+ =  (\mant{z} + 2^{-\mbit+c_z}) \times 2^{\expo{z}}$.
Let
\begin{align*}
    w = 2^{-\expo{z}+\expo{\eta}-c_z}, \quad b= (\mant{\eta}- \mant{z} \times 2^{-c_z} ) \times 2^{\expo{\eta}}.
\end{align*}

Since $ 1+ \emin \le \expo{\eta} \le 1$, we have $\emin+\expo{\eta} \le 1 -\emin = \emax$ which leads to $w,b \in \fpq$. Then we have
\begin{align*}
     \mu_z(z)&=(w \otimes z) \oplus b =  \mant{z} \times 2^{\expo{\eta}-c_x}  \oplus b \\
     &= \round{\mant{z} \times 2^{\expo{\eta}-c_z}  + (\mant{\eta}-\mant{z} \times 2^{-c_z}) \times 2^{\expo{\eta}}} = \eta, \\
     \mu_z(z^+) &= (w \otimes z^+) \oplus b  = (\mant{z} \times 2^{\expo{\eta}-c_z} +2^{-\mbit+\expo{\eta}} )  \oplus b \\
     &=  \round{\eta+2^{-\mbit+\expo{\eta}}}  = \eta^+.
\end{align*}
In addition, we have
\begin{align*}
     \mu_z(-1) =  -w \oplus b &\ge -2^{\emax} > -\fmax, \\
     \mu_z(1)  =  w \oplus b  &\le 2^{\emax} \le \fmax.
\end{align*}

\paragraph{\fbox{\bf Case 2: $\eta < 0, \; z < 0$.}}~

In this case, (1) in \cref{lem:eta_etaplus-2} holds if
\begin{align*} w > 0, \quad
    \begin{cases}
        \mu_z(-1) &\ge -\fmax, \\
        \mu_z(z) &= \eta, \\
        \mu_z(z^+) &= \eta^+, \\
        \mu_z(1) &\le \fmax.
    \end{cases}
\end{align*}
Note that we have $z = -\mant{z} \times 2^{\expo{z}}$ and $\eta = -\mant{\eta} \times 2^{\expo{\eta}}$.

\paragraph{\fbox{\bf Case 2-1: $ \eta \ne - 2^{\expo{\eta}}$.}}~

In this case, we have $\eta^+ = - (\mant{\eta} - 2^{-\mbit}) \times 2^{\expo{\eta}}$.

\paragraph{\underline{\bf Case 2-1-1: $  -2^{1+\emin} \le z <0$, or $ z<  -2^{1+\emin},\; z \ne -2^{\expo{z}}$.}}~

In this case, we have $z^+ = - (\mant{z} - 2^{-\mbit+c_z}) \times 2^{\expo{z}}$.
Let
\begin{align*}
    w = 2^{-\expo{z}+\expo{\eta}-c_z}, \quad b= (-\mant{\eta} + \mant{z} \times 2^{-c_z} ) \times 2^{\expo{\eta}}.
\end{align*}
Then we have
\begin{align*}
     \mu_z(z) &= (w \otimes z) \oplus b =  -\mant{z} \times 2^{\expo{\eta}-c_z}  \oplus b  \\
     &= \round{- \mant{z} \times 2^{\expo{\eta}-c_z}  + (-\mant{\eta}+\mant{z} \times 2^{-c_z}) \times 2^{\expo{\eta}}} = \eta, \\
     \mu_z(z^+) &= (w \otimes z^+) \oplus b \\
     &= (-\mant{z} \times 2^{\expo{\eta}-c_z} + 2^{-\mbit+\expo{\eta}} )  \oplus b =  \round{-\eta+2^{-\mbit+\expo{\eta}}}  = \eta^+,\\
     \mu_z(-1) &= -w \oplus b \ge -2^{\emax}  > -\fmax,\\
     \mu_z(1) &= w \oplus b \le 2^{\emax} < \fmax.
\end{align*}

\paragraph{\underline{\bf Case 2-1-2:  $z < - 2^{1+\emin}, \; z = - 2^{\expo{z}}$.}}~

In this case, we have $ 1+\emin \le \expo{z} \le 1 $, $c_z=0$ and $z^+ = - ( 2 - 2^{-\mbit}) \times 2^{-1+\expo{z}}$.
By \cref{lem:special_case}, we have $(1+2^{-\mbit}) \otimes z = -(1+2^{-\mbit}) \times 2^{-\expo{z}} $ and $(1+2^{-\mbit}) \otimes(z^+) = -2^{\expo{z}}$. Let
\begin{align*}
    w = (1+2^{-\mbit}) \times 2^{-\expo{z}+\expo{\eta}}, \quad b= (-\mant{\eta} + \mant{z} + 2^{-\mbit} ) \times 2^{\expo{\eta}}.
\end{align*}
Then we have
\begin{align*}
     \mu_z(z) &= (w \otimes z) \oplus b =  -(1+2^{-\mbit}) \times 2^{\expo{\eta}}  \oplus b \\
     &= \round{- (1+2^{-\mbit}) \times 2^{\expo{\eta}}  + (-\mant{\eta}+1 + 2^{-\mbit}) \times 2^{\expo{\eta}}} = \eta, \\
    \mu_z(z^+) &= (w \otimes z^+) \oplus b  = -2^{\expo{\eta}}   \oplus b =  \round{-\eta+2^{-\mbit+\expo{\eta}}}  = \eta^+, \\
    \mu_z(-1) &= -w \oplus b \ge - 2^{\emax } > -\fmax, \\
    \mu_z(1) &= w \oplus b \le 2^{\emax} < \fmax.
\end{align*}

\paragraph{\fbox{\bf Case 2-2: $ \eta = - 2^{\expo{\eta}}$.}}~

In this case, we have $\eta^+ = - (2 - 2^{-\mbit}) \times 2^{-1+\expo{\eta}}$.

\paragraph{\underline{\bf Case 2-2-1: $  -2^{1+\emin} \le z <0$, or $ z<  -2^{1+\emin},\; z \ne -2^{\expo{z}}$.}}~

In this case, we have $z^+ = - (\mant{z} - 2^{-\mbit+c_z}) \times 2^{\expo{z}}$ and $(-1 + \mant{z} \times 2^{-1-c_z} )$ is exact.
Let
\begin{align*}
    w = 2^{-1-\expo{z}+\expo{\eta}-c_z}, \quad b= (-1 + \mant{z} \times 2^{-1-c_z} ) \times 2^{\expo{\eta}}.
\end{align*}
Then we have
\begin{align*}
     \mu_z(z) &= (w \otimes z) \oplus b =  -\mant{z} \times 2^{-1+\expo{\eta}-c_z}  \oplus b \\
     &= \round{- \mant{z} \times 2^{-1+\expo{\eta}-c_z}  + (-1+\mant{z} \times 2^{-1-c_z}) \times 2^{\expo{\eta}}} = \eta, \\
     \mu_z(z^+) &= (w \otimes z^+) \oplus b  = (-\mant{z} \times 2^{-1+\expo{\eta}-c_z} + 2^{-1-\mbit+\expo{\eta}} )  \oplus b \\
     &=  \round{-\eta+2^{-1-\mbit+\expo{\eta}}}  = \eta^+, \\
     \mu_z(-1) &= -w \oplus b \ge - 2^{\emax } > -\fmax, \\
    \mu_z(1) &= w \oplus b \le 2^{\emax} < \fmax.
\end{align*}

\paragraph{\underline{\bf Case 2-2-2:  $z < - 2^{1+\emin}, \; z = - 2^{\expo{z}}$.}}~

In this case, we have $c_z=0$ and $z^+ = - ( 2 - 2^{-\mbit}) \times 2^{-1+\expo{z}}$.
Let
\begin{align*}
    w = 2^{-\expo{z}+\expo{\eta}-c_z}, \quad b= (-1 +  2^{-c_z} ) \times 2^{\expo{\eta}}.
\end{align*}
Then we have
\begin{align*}
     \mu_z(z) &= (w \otimes z) \oplus b =  - 2^{-\expo{\eta}-c_z}  \oplus b = \round{-  2^{\expo{\eta}-c_z}  + ( -1 + 2^{-c_z}) \times 2^{\expo{\eta}}} = \eta, \\
     \mu_z(z^+) &= (w \otimes z^+) \oplus b  = (- 2^{\expo{\eta}-c_z} + 2^{-1-\mbit+\expo{\eta}} )  \oplus b =  \round{-\eta+2^{-1-\mbit+\expo{\eta}}}  = \eta^+, \\
     \mu_z(-1) &= -w \oplus b \ge - 2^{\emax } > -\fmax, \\
    \mu_z(1) &= w \oplus b \le 2^{\emax} < \fmax.
\end{align*}

\paragraph{\fbox{\bf Case 3: $\eta > 0, \; z < 0$.}}~

In this case, (2) in \cref{lem:eta_etaplus-2} holds if
\begin{align*} w > 0, \quad
    \begin{cases}
        \mu_z(-1) &\le \fmax, \\
        \mu_z(z) &= \eta^+, \\
        \mu_z(z^+) &= \eta, \\
        \mu_z(1) &\ge -\fmax.
    \end{cases}
\end{align*}

In this case, let $\nu := -(\eta^+)$. Using \textbf{Case 2}, there exist $w_\nu , b_\nu \in \fpq$ with $w_\nu > 0$ such that
\begin{align*}
     (w_\nu \otimes z) \oplus b_\nu &= \nu=-(\eta^+), \\
     (w_\nu \otimes z^+) \oplus b_\nu  &= \nu^+ = -\eta.
\end{align*}
Let $w=-w_\nu,b=-b_\nu$. Then we have
\begin{align*}
     \mu_z(z) &= (w \otimes z) \oplus b = \eta^+, \\
     \mu_z(z^+) &= (w \otimes z^+) \oplus b  = \eta, \\
      \mu_z(-1) &= -w \oplus b = w_\nu \oplus (-b_\nu) \le  2^{\emax } < \fmax, \\
    \mu_z(1) &= w \oplus b = (-w_\nu) \oplus (-b_\nu) \ge -2^{\emax} > -\fmax.
\end{align*}

\paragraph{\fbox{\bf Case 4: $\eta < 0, \; z \ge 0$.}}~

Let $\nu := -(\eta^+)$.
Using \textbf{Case 1}, there exist $w_\nu , b_\nu$ with $w_nu > 0 $ such that
\begin{align*}
     (w_\nu \otimes z) \oplus b_\nu &= \nu=-(\eta^+), \\
     (w_\nu \otimes z^+) \oplus b_\nu  &= \nu^+ = -\eta.
\end{align*}
Let $w=-w_\nu,b=-b_\nu$. Then we have
\begin{align*}
     \mu_z(z) &= (w \otimes z) \oplus b = \eta^+, \\
     \mu_z(z^+) &= (w \otimes z^+) \oplus b  = \eta, \\
      \mu_z(-1) &= -w \oplus b = w_\nu \oplus (-b_\nu) \le  2^{\emax } < \fmax, \\
    \mu_z(1) &= w \oplus b = (-w_\nu) \oplus (-b_\nu) \ge -2^{\emax} > -\fmax.
\end{align*}
\myqedd

\clearpage
\section{Proofs of the Results in \S\ref{sec:pflem:indc-to-iua}}
\label{sec:pflem:proof_results}

Throughout this section, we use $\mcR=\{x\in\fpq:|x|\in[\tfrac{\feps}{2}+2\feps^2, \tfrac{5}{4}-2\feps]_{\fpq}\}$ and $K=\sigma(c_2)$, i.e., $K\in\mcR$ by \cref{cond:activation2}.

We present preliminary technical lemmas for the proofs in this section: \cref{lem:sign,lem:endbit_control,lem:approx_sum} (presented in \cref{subsec:techlemma_for_indc-to-iua}).

\subsection{Proof of Lemma~\ref{lem:indc-box}}
\label{sec:pflem:indc-box}

We prove \cref{lem:indc-box} using the mathematical induction on $\lceil\log_{2^M}d\rceil$.
Consider the base case where $\lceil\log_{2^M}d\rceil=1$, i.e., $d\in[2^M]$. Let $\mcB=(\intv{a_1,b_1},\dots,\intv{a_d,b_d}$ and $\mu_{K,\eta,n}:\fpq^n\to\fpq$ be a  network without the first affine layer 
such that for $x_1,\dots,x_n\in\{0,K\}$,
\begin{align*}
\mu_{K,\eta,n}(x_1,\dots,x_d)&>\eta\quad\text{if}~x_i=K~\text{for all}~i\in[d],\\
\mu_{K,\eta,n}(x_1,\dots,x_d)&\le\eta\quad\text{otherwise}.
\end{align*}

We note that such $\mu_{K,\eta,n}$ always exists for all $n\in[2^M]$ by \cref{lem:sign}.
We then construct $\tilde \nu_{\mcB}$ as follows:
\begin{align*}
\tilde \nu_{\mcB}(\bfx)=\psi_{>\eta}\Big(\mu_{K,\eta,d}\big(\tilde \nu_1(x_1),\dots,\tilde \nu_d(x_d)\big)\Big),
\end{align*}
where
\begin{align*}
\tilde \nu_i(x_i)=\psi_{>\eta}\left(\mu_{K,\eta,2}\Big(\phi_{\ge a_i}(x_i),\phi_{\le b_i}(x_i)\Big)\right)\quad\text{for all}~i\in[d].
\end{align*}
Then, one can observe that $\tilde \nu_{\mcB}$ has depth $L_\phi+2 L_\psi-2$ and does not have the last affine layer.

Let $\mcC=(\intv{s_1,t_1},\dots,\intv{s_d,t_d})$ be an abstract box in $[a,b]_\fpq^d$, and let $\bar\mcB=\conc{\mcB}$ and $\bar\mcC=\conc{\mcC}$.
We now show that $\tilde \nu_{\mcB}^\sharp(\mcC)=(K\iota_\mcB)^\sharp(\mcC)$ by considering the following three cases: (1) $\bar\mcB\cap\bar\mcC=\emptyset$, (2) $\bar\mcB\cap\bar\mcC\ne\emptyset$ and $\bar\mcC\not\subset\bar\mcB$, and (3) $\bar\mcC\subset\bar\mcB$.

\paragraph{\fbox{\bf Case 1: $\bar\mcB\cap\bar\mcC=\emptyset$.}}~

In this case, there exists $i^*\in[d]$ such that $[s_{i^*},t_{i^*}]_\fpq\cap[a_{i^*},b_{i^*}]_\fpq=\emptyset$; otherwise, $\bar\mcB\cap\bar\mcC\ne\emptyset$.
This implies that 
\begin{align*}
\tilde \nu_{i^*}^\sharp(\intv{s_{i^*},t_{i^*}})=\intv{0,0}.
\end{align*}
Since $\gamma(g_i^\sharp(\intv{c_{i},d_{i}})\subset[0,K]$ by the definition of $\psi_{>\eta}$, we have
\begin{align*}
\mu_{K,\eta,2}^\sharp\Big(\tilde \nu_1^\sharp(\intv{s_{1},t_{1}}),\dots,\tilde \nu_d^\sharp(\intv{s_{d},t_{d}})\Big)=\intv{u,v},
\end{align*}
for some $u,v\in\fpq$ such that $u\le v<\eta$.
Hence,
$$\tilde \nu^\sharp_{\mcB}(\mcC)=\psi_{>\eta}\big(\intv{u,v}\big)=\intv{0,0}.$$

\paragraph{\fbox{\bf Case 2: $\bar\mcB\cap\bar\mcC\ne\emptyset$ and $\bar\mcC\not\subset\bar\mcB$.}}~

Since $\bar\mcB\cap\bar\mcC\ne\emptyset$, $[a_{i},b_{i}]_{\fpq}\cap[s_{i},t_{i}]_{\fpq}\ne\emptyset$ for all $i\in[d]$. This implies that for each $i\in[d]$,
\begin{align}
\tilde \nu_{i}^\sharp(\intv{s_{i},t_{i}})=\intv{w_i,1},\label{eq:indc-cube1}
\end{align}
for some $w_i\in\{0,1\}$.
Furthermore, since $\bar\mcB\not\subset\bar\mcC$, there exists $i^*\in[n]$ such that $[s_{i^*},t_{i^*}]_{\fpq}\not\subset[a_{i^*},b_{i^*}]_{\fpq}$, i.e., $w_{i^*}=0$. 
Combining \cref{eq:indc-cube1} and $u_{i^*}=0$ implies that
\begin{align*}
\mu_{K,\eta,2}^\sharp\Big(\tilde \nu_1^\sharp(\intv{s_{1},t_{1}},\dots,\tilde \nu_d^\sharp(\intv{s_{d},t_{d}})\Big)=\intv{u,v},
\end{align*}
for some $u<\eta<v$. Hence,
$\tilde \nu^\sharp_{\mcB}(\mcC)=\psi_{>\eta}\big(\intv{u,v}\big)=\intv{0,K}$.

\paragraph{\fbox{\bf Case 3: $\bar\mcC\subset\bar\mcB$.}}~

Since $\bar\mcC\subset\bar\mcB$, $[s_{i},t_{i}]_{\fpq}\subset[a_{i},b_{i}]_{\fpq}$ for all $i\in[n]$. This implies that
\begin{align*}
\tilde \nu_{i}^\sharp(\intv{s_{i},t_{i}})=\intv{1,1},
\end{align*}
for all $i\in[n]$.
Thus, it holds that $\tilde \nu^\sharp_{\mcB}(\mcC)=\intv{K,K}$.
By considering all three cases, one can conclude that when $\lceil\log_{2^M}d\rceil=1$, then a depth-$(L_\phi+2 L_\psi-2)$ $\sigma$-network $\tilde \nu_{\mcB}^\sharp=(K\iota_{\mcB})^\sharp$  on $[a,b]_\fpq^d$.

Now, suppose that $\lceil\log_{2^M}d\rceil=n>1$.
Let $k\in[2^M-1]$ and $r\in[2^{(n-1)M}-1]\cup\{0\}$ such that $d=k2^{(n-1)M}+r$, and let
\begin{align*}
&\mcB_j=\big(\intv{a_{(j-1)2^{(n-1)M}+1},b_{(j-1)2^{(n-1)M}+1}},\dots,\intv{a_{j2^{(n-1)M}},b_{j2^{(n-1)M}}}\big)\quad\text{for all}~j\in[k],\\
&\mcB_{k+1}=\big(\intv{a_{k2^{(n-1)M}+1},b_{k2^{(n-1)M}+1}},\dots,\intv{a_{k2^{(n-1)M}+r},b_{k2^{(n-1)M}+r}}\big).
\end{align*}
Then, by the inductive hypothesis, there exist depth-$(L_\phi+nL_\phi-n)$ $\sigma$-networks $\tilde \nu_{\mcB_1},\dots,\tilde \nu_{\mcB_k}$ and a depth-$(L_\phi+(\lceil\log_{2^M} r\rceil+1) L_\phi-(\lceil\log_{2^M} r\rceil+1))$ $\sigma$-network $\tilde \nu_{\mcB_{k+1}}$ such that $\tilde \nu_{\mcB_j}^\sharp=(K\iota_{\mcB_j})^\sharp$ on $[a,b]_\fpq^d$ for all $j\in[k+1]$.
Let $\hat \nu_{\mcB_{k+1}}$ be the composition of $\tilde \nu_{\mcB_{k+1}}$ and $(n-1-\lceil\log_{2^M}r\rceil)$ times composition of $\psi_{>\eta}\circ\mu_{K,\eta,1}$, i.e.,
$\hat \nu_{\mcB_{k+1}}$ is a depth-$(L_\phi+nL_\phi-n)$ $\sigma$-network satisfying
$\hat \nu_{\mcB_{k+1}}^\sharp=(K\iota_{\mcB_{k+1}})^\sharp$.
Here, we note that $\tilde \nu_{\mcB_1},\dots,\tilde \nu_{\mcB_k},\tilde \nu_{\mcB_1},\dots,\hat \nu_{\mcB_{k+1}}$ have the same depth and they do not have the last affine layers.
Under this observation, we construct $\tilde \nu_{\mcB}$ as
\begin{align*}
    \tilde \nu_{\mcB}^\sharp=\psi_{>\eta}^\sharp\Big(\mu_{K,\eta,k+1}^\sharp\big(\tilde \nu_{\mcB_1}^\sharp,\dots,\tilde \nu_{\mcB_k}^\sharp,\hat \nu_{\mcB_{k+1}}^\sharp\big)\Big).
\end{align*}
Then, $\tilde \nu_{\mcB}^\sharp=(K\iota_{\mcB})^\sharp$ and $\tilde \nu_{\mcB}$ has depth $L_\phi+(n+1)L_\psi-n-1$. This proves the inductive step and completes the proof.
\myqedd

\subsection{Proof of Lemma~\ref{lem:indc-set}}
\label{sec:pflem:indc-set}

Let $\mcT$ be the collection of all abstract boxes in $\mcS$. We construct $\tilde \nu_{\mcS}(\bfx)$ as
\begin{align*}
\tilde \nu_{\mcS}(\bfx)=\psi_{>\eta}\big(g(\bfx)\big),\quad g(\bfx)=\left(\bigoplus_{\mcB\in\mcT}w\otimes \tilde \nu_{\mcB}(\bfx)\right)\oplus\eta,
\end{align*}
where $w\in\fpq$ is chosen so that $w\otimes K\in(2^{\expo{\eta}-\mbit-1},(1+2^{-1})\times2^{\expo{\eta}-\mbit})_{\fpq}$, i.e., $(w\otimes K)\oplus\eta\ge\eta^+$. Such $w$ always exists since $K\in\mcR$ and by \cref{lem:endbit_control}.
Furthermore, from our choice of $w$, we note that $\eta\le g(\bfx)\le2^{\expo{\eta}+1}$ for all $\bfx\in\fpq$, i.e., overflow does not occur in the evaluation of $g$ under $2^{\ebit-1}\ge\mbit\ge3$. Here, note that $\tilde \nu_{\mcS}(\bfx)$ has depth $L+L_\phi-1$.

We now show that $\tilde \nu_{\mcS}^\sharp=(K\iota_{\mcS})^\sharp$ on $[a,b]^d_{\fpq}$.
Let $\mcC=(\intv{s_1,t_1},\dots,\intv{s_d,t_d})$ be an abstract box in $[a,b]_\fpq^d$, and let $\bar\mcC=\conc{\mcC}$ and $\bar\mcB=\conc{\mcB}$ for all $\mcB\in\mcT$.
Here, if $\bar\mcC\cap\mcS=\emptyset$, then $\bar\mcB\cap\bar\mcC=\emptyset$ for all $\mcB\in\mcT$. This implies that
\begin{align*}
g^\sharp(\mcC)=\intv{\eta,\eta}\quad\text{and}\quad\tilde \nu_{\mcS}^\sharp(\mcB)=\intv{0,0},
\end{align*}
by the definition of $\tilde \nu_{\mcB}$ (see \cref{lem:indc-box}).
In addition, if $\bar\mcC\cap\mcS\ne\emptyset$ and $\bar\mcC\not\subset\mcS$, then $\bar\mcC\not\subset\bar\mcB$ for all $\mcB\in\mcT$ and there exists $\mcB^*\in\mcT$ such that $\bar\mcB^*\cap\mcC\ne\emptyset$.
This implies that $\tilde \nu_{\mcB}(\mcC)=\intv{0,u_{\mcB}}$ for some $u_{\mcB}\in\{0,K\}$ for all $\mcB\in\mcT$ and $u_{\mcB^*}=K$. This implies that for some $v\ge\eta^+$, we have
\begin{align*}
g^\sharp(\mcC)=\intv{\eta,v}\quad\text{and}\quad\tilde \nu^\sharp_{\mcB}(\mcC)=\intv{0,K}.
\end{align*}
Lastly, suppose that $\bar\mcC\subset\mcS$. Since $\mcC$ is a box in $\mcS$, this implies $\mcC\in\mcT$ and $\tilde \nu_{\mcC}^\sharp(\mcC)=\intv{K,K}$. Thus, for some $\eta^+\le u\le v$, we have
\begin{align*}
g^\sharp(\mcC)=\intv{u,v}\quad\text{and}\quad\tilde \nu^\sharp_{\mcS}(\mcC)=\intv{K,K}.
\end{align*}
This completes the proof.
\myqedd

\subsection{Proof of Lemma~\ref{lem:iua}}
\label{sec:pflem:iua}

Let $\mcD=[a,b]^d_{\fpq}$ and
define
\begin{align*}
h_+(\bfx)=\max\{0,h(\bfx)\}\quad\text{and}\quad h_-(\bfx)\max\{0,-h(\bfx)\},
\end{align*}
i.e., $h^\sharp=h_+^\sharp\ominus^\sharp h_-^\sharp$.
Let $k_+,k_-\in\bbN\cup\{0\}$ be the numbers such that $\max_{\bfx\in\mcD} h_+(\bfx)$ is the $(k_++1)$-th smallest non-negative number in $\fpq$ and $\max_{\bfx\in\mcD} h_-(\bfx)$ is the $(k_-+1)$-th smallest non-negative number in $\fpq$.
Let $0=z_0<z_1<\cdots<z_{(|\fpq|-1)/2}=\fmax<z_{(|\fpq|+1)/2}=\infty$ be all non-negative numbers in $\fpq\cup\{\infty\}$; here, we have $z_{k_+}=\max_{\bfx\in\mcD} h_+(\bfx)$ and $z_{k_-}=\max_{\bfx\in\mcD} h_-(\bfx)$.

Let $\mcS_{\tau,i}=\{\bfx\in\mcD:h_\tau(\bfx)\ge z_i\}$ 
for $\tau\in\{-,+\}$
for all $i\in[(|\fpq|+1)/2]\cup\{0\}$
and define $f:\mcD\to\fpq$ as
\begin{align*}
\nu(\bfx)=\left(\bigoplus_{i=1}^{k_+}\big(w_i\otimes \tilde \nu_{\mcS_{+,i}}(\bfx)\big)\right)\oplus\bigoplus_{i=1}^{k_-}\big((-w_i)\otimes \tilde \nu_{\mcS_{-,i}}(\bfx)\big),
\end{align*}
where $w_i\in\fpq$ is chosen so that $w_i\otimes K\in(2^{\expo{z_i}-\mbit-1},(1+2^{-1})\times2^{\expo{z_i}-\mbit})_{\fpq}$. Such $w_i$ exists for all $i$ by $K\in\mcR$ and \cref{lem:endbit_control}.
Here, we note that $\nu$ has depth $L$.

Let $\mcB$ be a box in $\mcD$ and choose $i_{\tau,\max},i_{\tau,\min}$ such that $z_{i_{\tau,\max}}=\max_{\bfx\in\conc{\mcB}}h_\tau(\bfx)$ and $z_{i_{\tau,\min}}=\min_{\bfx\in\conc{\mcB}}h_\tau(\bfx)$ for $\tau\in\{-,+\}$.
Then, one can observe that for $\tau\in\{-,+\}$ and $i\in[(|\fpq|+1)/2]\cup\{0\}$,
\begin{align*}
\tilde \nu_{\mcS_{\tau,i}}^\sharp(\mcB)=\begin{cases}
\intv{K,K}~&\text{if}~i\le i_{\tau,\min}\\
\intv{0,K}~&\text{if}~i_{\min}<i\le i_{\tau,\max}\\
\intv{0,0}~&\text{if}~i_{\tau,\max}<i
\end{cases}.
\end{align*}
By the definition of $w_i$ and \cref{lem:approx_sum}, this implies that
\begin{align*}
\sideset{}{^\sharp}\bigoplus_{i=1}^{k_+}\big(w_i\otimes^\sharp \tilde \nu_{\mcS_{+,i}}^\sharp(\mcB)\big)=\intv{z_{i_{+,\min}},z_{i_{+,\max}}}.
\end{align*}
Here, if $z_{i_{+,\min}}>0$, then by the definition $h_-$ and $\mcS_{\tau,i}$, we have $\tilde \nu_{\mcS_{-,i}}(\mcB)=\intv{0,0}$ for all $i$. This implies that \begin{align*}
\nu^\sharp(\mcB)=\sideset{}{^\sharp}\bigoplus_{i=1}^{k_+}\big(w_i\otimes^\sharp \tilde \nu_{\mcS_{+,i}}^\sharp(\mcB)\big)=\intv{z_{i_{+,\min}},z_{i_{+,\max}}}=h^\sharp(\mcB).
\end{align*}
If $z_{i_{+,\min}}=0$, then
\begin{align*}
(-w_i)\otimes^\sharp\tilde \nu_{\mcS_{-,i}^\sharp(\mcB)}=\intv{-u_i,0},
\end{align*}
for some $u_i\in(2^{\expo{z_i}-\mbit-1},(1+2^{-1})\times2^{\expo{z_i}-\mbit})_{\fpq}$.
Hence, by \cref{lem:approx_sum} and the definition of $\oplus^\sharp$, we have $\nu^\sharp(\mcB)=\intv{-z_{i_{-,\max}},z_{i_{+,\max}}}=h^\sharp(\mcB)$.
\myqedd

\clearpage
\section{Technical Lemmas}
\label{sec:techlemma_revised}

\subsection{Common Technical Lemma}
\label{sec:common_techlemma}

We present the technical lemma used in multiple proofs.
\begin{lemma}
\label{lem:inverse}
Suppose $\mbit \ge 3, \ebit \ge 2$. For any $x\in[1,2)_{\fpq}$, at least one of the followings holds:
\begin{itemize}
    \item There exists $y\in (2^{-1},1]_{\fpq}$ such that $x\otimes y=1$. In this case, we denote $y$ as $x^{\parallel}\defeq y$.
    \item There exists $y\in[2^{-1},1)_{\fpq}$ such that $x\otimes y_1=1^- = 1-2^{-1-\mbit}$. In this case, we denote $y$ as $x^{\dag}\defeq y_1$.
\end{itemize}
\end{lemma}
\begin{proof}
If $x = 1$, then $x^{\parallel}=1$, $ x^\dag = 1- 2^{-1-\mbit}$. \\
Now suppose $ 1< x < 2$. Note that $x = 1.\underbrace{m_{x,1}...m_{x,\mbit}}_{\mbit \; \text{times}} \times 2^0 $.
Let $n_x = x \times 2^{\mbit} \in \mathbb{N}$. Note that $2^{\mbit} < n_x <2^{\mbit+1}$.  By dividing $2^{2\mbit}$ by $n_{x}$, we have the following
\begin{equation}
    2^{2 \mbit} = n_x n_u + r \quad ( 0 < r < n_x).
\end{equation}
Note that since $ 2^{\mbit-1} \le n_u < 2^{\mbit}$, $n_u \times 2^{1-\mbit}$ has the form of $1.\underbrace{m_{1}...m_{\mbit}}_{\mbit \; \text{times}} \times 2^{-1}$ whose significand has $\leq \mbit+1$ binary digits.
We consider the following cases.

\begin{itemize}
\item \textbf{Case (1) $\quad$}  $0 < r  \le  2^{-2+\mbit} $. \\
In this case, $x^{\parallel}=  {n_u}\times 2^{-\mbit}$ exists since
\begin{align*}
    1- 2^{-2-\mbit} & \le  x \times (n_u \times 2^{-\mbit}) = (n_x \times n_u) \times 2^{-2\mbit}  = 1 - r \times 2^{-2\mbit} < 1.
\end{align*}

\item \textbf{Case (2) $\quad$} $ 2^{-2+\mbit} < r < n_x - 2^{-1+\mbit}$. \\
In this case, note that  $n_x -2^{-1+\mbit} < 2^{-1+\mbit} < 3 \times 2^{-2+\mbit}$.
Then $x^\dag  = {n_u}\times 2^{-\mbit}$ exists since
\begin{align*}
     1- 3 \times 2^{-2-\mbit}  &< x \times (n_u \times 2^{-\mbit}) = (n_x \times n_u) \times 2^{-2\mbit}  = 1 - r \times 2^{-2\mbit} < 1- 2^{-2-\mbit}.
\end{align*}

\item \textbf{Case (3) $\quad$} $ n_x - 2^{-1+\mbit} \le r < n_x$. \\
In this case, $x^{\parallel}=  (n_u+1)\times 2^{-\mbit}$ exists since
\begin{align*}
      1 &< x \times ((n_u+1) \times 2^{-\mbit}) = 1 + (n_x-r) \times 2^{-2\mbit} \le 1+  2^{-1-\mbit}.
\end{align*}
\end{itemize}

Now, we show $ x^\parallel > 2^{-1}$. If $n_x \le 2^{1+\mbit}-4$, we have $n_u \ge 2^{-1+\mbit}+1$. If $n_x \ge 2^{1+\mbit}-3$, we have $n_u = 2^{-1+\mbit}, r \ge 2^{-1+\mbit}$ which does not belong to \textbf{Case (1)}.
Therefore we have $ x^\parallel > 2^{-1}$.
This completes the proof.
\end{proof}

\subsection{Technical Lemmas for Lemma~\ref{lem:sigmaindicator-2}}
\label{subsec:techlemma_for_sigmaindicator-2}

This subsection presents technical lemmas for the proof of \cref{lem:sigmaindicator-2} (\cref{sec:pflem:sigmaindicator-2}).

\todo{WL. For the next version: The proofs of some lemmas seem to use lemmas appearing later, and this looks a bit strange to me.}

\begin{lemma}\label{lem:subnorm_inverse}
Suppose $\eta \in ( -2^{1+\emin} , 2^{1+\emin})_{\fpq}$. For any normal $x \in (-1-2^{-1},1+2^{-1})_{\fpq}$, there exist $y_1,y_2 \in \fpq$ such that $(y_1 \otimes x) \oplus (y_2 \otimes x) = \eta$.
\end{lemma}
\begin{proof}
Without loss of generality, we assume $\eta,x > 0 $.
Since $\eta \in ( -2^{1+\emin} , 2^{1+\emin})_{\fpq}$, we write $\eta$ as $\eta = n_\eta \times 2^{-\mbit+\emin}$ for $ 1 \le n_\eta < 2^{1+\mbit}, n_\eta \in \bbN$. First define $y_*$ as
\begin{align*}
    y_* \defeq \begin{cases}
    2^{-1-M+\emin-\expo{x}} \quad &\text{if} \quad x < 1, \mant{x} = 1 ,   \\
    2^{-M+\emin-\expo{x}} \quad &\text{if} \quad x < 1, \mant{x} > 1 ,\\
    2^{-M+\emin} \quad &\text{if} \quad  1 \le x < 1+2^{-1} .
\end{cases}
\end{align*}

If $n_\eta =1$, let $y_1 = y_*$.  \\
Then we have $ y_1 \otimes x = \begin{cases}
    \round{2^{-\mbit+\emin} } =  2^{-\mbit + \emin} \quad &\text{if} \quad x < 1, \mant{x} = 1, \\
    \round{\mant{x} \times 2^{-1-\mbit+\emin} } = 2^{-\mbit + \emin}\quad &\text{if} \quad x < 1, \mant{x} > 1, \\
    \round{\mant{x} \times 2^{-\mbit+\emin} } = 2^{-\mbit + \emin} \quad &\text{if} \quad 1 \le x < 1+2^{-1}.\\
\end{cases}$

If $n_\eta > 1$, we have $ 2 \le n_\eta \le 2^{1+\mbit}-1$.
Since $x$ is normal, we can write $x$ as  $ x = n_x \times 2^{-\mbit + \expo{x}}$ for some $2^{\mbit} \le n_x < 2^{1+\mbit}, n_x \in \bbN$.

By dividing $2^{\mbit} \times n_\eta$ by $n_x$, we have
    \begin{align*}
        2^\mbit \times n_\eta =  m \times n_x + r \quad ( 0 \le r < n_x),
    \end{align*}
for some $ 1 \le m \le 2^{1+\mbit}-1$.

Let
\begin{align*}
    (n_{y_1},y_2)= \begin{cases}
        (m,0) \quad &\text{if} \quad r < 2^{\mbit-1}\; \text{or} \; r=2^{\mbit-1}, n_\eta \equiv 0 \Mod{2} \\
        (m,y_*) \quad &\text{if} \quad 2^{\mbit-1} <  r < 3 \times 2^{\mbit-1} \;  \text{or} \; r\in \{2^{\mbit-1} , 3\times 2^{\mbit-1}\}, n_\eta \equiv 1 \Mod{2} \\
        (m+1,0) \quad &\text{if} \quad  r > 3 \times 2^{\mbit-1} \; \text{or} \; r= 3\times 2^{\mbit-1}, n_\eta \equiv 0 \Mod{2}
    \end{cases}
\end{align*} and
$y_1 = n_{y_1} \times 2^{-M+\emin-\expo{x}} \in \fpq$.
We consider the following cases.

\paragraph{\fbox{\bf Case 1: $r < 2^{\mbit-1}\; \text{or} \; r=2^{\mbit-1}, n_\eta \equiv 0 \Mod{2}$.}}~

In this case, we have
\begin{align*}
    y_1 \otimes x =  \round{ ( n_{y_1} \times n_x \times 2^{-M}) \times 2^{-M+\emin} } =
      \round{ ( n_\eta - r \times 2^{-M}) \times 2^{-M+\emin} } = \eta.
\end{align*}

\paragraph{\fbox{\bf Case 2: $2^{\mbit-1} <  r < 3 \times 2^{\mbit-1} \;  \text{or} \; r\in \{2^{\mbit-1} , 3\times 2^{\mbit-1}\}, n_\eta \equiv 1 \Mod{2}$.}}~

In this case, we have
\begin{align*}
    (y_1 \otimes x) \oplus (y_2 \otimes x) &=
      \round{ ( n_\eta  - r \times 2^{-M}) \times 2^{-M+\emin} } \oplus 2^{-M+\emin} \\
      &= (n_\eta-1) \times 2^{-M+\emin} \oplus 2^{-M+\emin} = \eta.
\end{align*}

\paragraph{\fbox{\bf Case 3: $ r > 3 \times 2^{\mbit-1} \; \text{or} \; r= 3\times 2^{\mbit-1}, n_\eta \equiv 0 \Mod{2}$.}}~

In this case, since $ 3 \times 2^{\mbit-1} \le r,n_x < 2^{\mbit+1}$, we have $ 0 \le n_x -r < 2^{\mbit-1} $. Therefore,
\begin{align*}
    (y_1 \otimes x)  &=
      \round{ ( n_\eta  + (n_x- r) \times 2^{-M}) \times 2^{-M+\emin} }  = \eta.
\end{align*}
\end{proof}

\begin{lemma}\label{lem:telescoping}
    Let $K\in [(1+2^{-\mbit+1})\times2^{-\mbit-2},1+2^{-2}-2^{-\mbit}]_{\fpq}$.
    Consider $\expo{\zeta}\in \bbZ$ such that $\emin +1 \le \expo{\zeta}\le \emax-\mbit -1$.
    For $n\in \bbN$, $i\in [n]$, and $0<\alpha_i\in \fpq$, define $f:\fpq\rightarrow\fpq$ as
    \begin{equation*}
        f(x)\defeq x\oplus \bigoplus_{i=1}^n\lrp{\alpha_i\otimes K}.
    \end{equation*}
    Then, there exists $n$ and $\alpha_i$s such that one of the following statements holds:
    \begin{equation}\label{eq:pseudo_inverse}
        f^\sharp\lrp{ \langle -1.1\times 2^{\expo{\zeta}},0\rangle }\subset \langle -(2^{\expo{\zeta}})^-,(2^{\expo{\zeta}})^-\rangle, \;
        f^\sharp \left( \langle\fmin, 1.1\times 2^{\expo{\zeta}}\rangle \right) \subset \langle 2^{\expo{\zeta}}, (1.1\times 2^{\expo{\zeta}+1})^-\rangle,
    \end{equation}
    or
    \begin{equation}\label{eq:true_inverse}
        f^\sharp\lrp{ \langle -1.1\times 2^{\expo{\zeta}},0\rangle}\subset \langle-(2^{\expo{\zeta}})^-,2^{\expo{\zeta}}\rangle, \;
        f^\sharp \left( \langle \fmin, 1.1\times 2^{\expo{\zeta}}  \rangle \right) \subset \langle (2^{\expo{\zeta}})^+, (1.1\times 2^{\expo{\zeta}+1})^- \rangle .
    \end{equation}
If $\expo{\zeta}\le \emin$, for each statement of \cref{eq:pseudo_inverse} or $\cref{eq:true_inverse}$, there exist $n$ and $\alpha_i$s such that  satisfy each equation, respectively. \\
Additionally, if $2^{\expo{\zeta}} \le x\in \fpq$, then we have
\begin{equation}
    f(x)\le  x\oplus \lrp{2^{\expo{\zeta}}}^+. \label{eq:telescoping_further}
\end{equation}
\end{lemma}
\begin{proof}
Since $f$ is a sum of increasing functions, we only need to consider the endpoints of $ \langle -1.1\times 2^{\expo{\zeta}},0 \rangle $ and $\langle \fmin, 1.1\times 2^{\expo{\zeta}} \rangle$, namely $x =0, -1.1\times 2^{\expo{\zeta}}, 1.1\times 2^{\expo{\zeta}}$, and $\fmin$. In other words, we suffice to show
\begin{align}
    \begin{cases}
    f( -1.1\times 2^{\expo{\zeta}}) > - 2^{\expo{\zeta}} \\
    f(0) < 2^{\expo{\zeta}} \\
    f(\fmin) \ge 2^{\expo{\zeta}} \\
    f(1.1\times 2^{\expo{\zeta}}) < 1.1 \times 2^{\expo{\zeta}+1} .
    \end{cases} \label{eq:condforcase1}
\end{align}
or
\begin{align}
\begin{cases}
    f( -1.1\times 2^{\expo{\zeta}}) > - 2^{\expo{\zeta}} \\
    f(0) \le 2^{\expo{\zeta}} \\
    f(\fmin) > 2^{\expo{\zeta}} \\
    f(1.1\times 2^{\expo{\zeta}}) < 1.1 \times 2^{\expo{\zeta}+1} .
    \end{cases} \label{eq:condforcase2}
\end{align}
Let $K$ be represented as
\begin{equation*}
    K = \mant{K}\times 2^{\expo{K}} , \quad -\mbit-2 \le \expo{K} \le 0.
\end{equation*}
By \cref{lem:inverse}, at least one of $\mant{K}^{\parallel} \in (2^{-1},1]_{\fpq} $ or $\mant{K}^{\dag} \in [2^{-1},1)_{\fpq}$ exists such that
\begin{align*}
    \mant{K}^{\dag} \otimes \mant{K} &= 1^- = 1- 2^{-\mbit-1}=0.\underbrace{1\dots 1}_{\mbit+1 \text{ times }}, \\
    \mant{K}^{\parallel} \otimes \mant{K} &= 1.
\end{align*}
We consider the following cases.

\paragraph{\fbox{\bf Case 1: $\mant{K}^{\dag}$ exists.}}~

In this case, we will show \cref{eq:condforcase1}.
First, define $m_1 \in \bbN_{\ge 0}$ as
\begin{align*}
    m_1 \defeq  \ceilZ{\frac{ -1   + \expo{\zeta}-\emin}{\mbit+2}}  \in \bbN_{\ge 0}  .
\end{align*}
Then $m_1$ is the unique non-negative natural number satisfying
    \begin{equation*}
       \expo{\zeta} - \lrp{\mbit+2}(m_1+1)  \le \emin -\mbit -1  < \expo{\zeta} - \lrp{\mbit+2}m_1 ,
    \end{equation*}
or equivalently,
\begin{equation}
     \emin - \mbit  \le \expo{\zeta} - \lrp{\mbit+2}m_1 \le \emin + 1, \label{eq:expo_zeta}
\end{equation}
where the first inequality is due to $\emin - \mbit - 1 < \expo{\zeta} - \lrp{\mbit+2}m_1$ and $\expo{\zeta} - \lrp{\mbit+2}m_1 \in \bbZ$.

    For $i\in [m_1]$, we define $\beta_i$ as
\begin{equation*}
    \beta_i\defeq (2-2^{-\mbit}) \times  2^{\expo{\zeta}- \lrp{\mbit+2}\lrp{m_1 - i}-1} = 0.\underbrace{1\dots 1}_{\mbit+1 \text{ times }}\times 2^{\expo{\zeta}- \lrp{\mbit+2}\lrp{m_1 - i}}.
    \end{equation*}
As
\begin{align*}
  \emin  &\le \expo{\beta_i}= \expo{\zeta}- \lrp{\mbit+2}\lrp{m_1 - 1} - 1  = \expo{\zeta}- \lrp{\mbit+2}m_1 + \mbit + 1  \\
  &\le \expo{\zeta} -\lrp{\mbit+2}\lrp{m_1 - i} - 1 \le \expo{\zeta} - 1 \le \emax - \mbit -2,
\end{align*}
    we have $\beta_i \in \fpq$.
Define $\alpha'_i$ for $i\in [m_1]$ as
\begin{equation*}
    \alpha'_i\defeq \mant{K}^{\dag} \times 2^{\expo{\zeta}- \lrp{\mbit+2}\lrp{m_1 - i} - \expo{K}}.
\end{equation*}
Since $\mant{K}^{\dag} \in \left[\frac{1}{2},1\right)_{\fpq}$, we have $\mant{\alpha'_i}= 2\mant{K}^{\dag}$ and $\expo{\alpha'_i} = \expo{\zeta}- \lrp{\mbit+2}\lrp{m_1 - i}  -1 $.  Since $0\le - \expo{K}\le \mbit +2$, we have
    \begin{equation*}
        \emin\le \expo{\alpha'_i} = \expo{\zeta}- \lrp{\mbit+2}\lrp{m_1 - i} - \expo{K} - 1 \le \emax.
    \end{equation*}
Therefore, we have  $\alpha'_i\in \fpq$, and
\begin{equation}
    \alpha'_i\otimes K = \beta_i. \label{eq:alphapK}
\end{equation}
 Consider $ 2^{\expo{\zeta} - \lrp{\mbit +2}m_1}-\fmin$.
By \cref{eq:expo_zeta},  $2^{\expo{\zeta} - \lrp{\mbit +2}m_1}$ is subnormal and
    \begin{equation*}
       2^{\expo{\zeta} - \lrp{\mbit +2}m_1}- \fmin =  1.\underbrace{1\dots 1}_{\expo{\zeta} - \lrp{\mbit +2}m_1-1- \emin +\mbit \text{ times }} \times 2^{\expo{\zeta} - \lrp{\mbit +2}m_1-1}\in \fpq.
    \end{equation*}
Since $ 2^{\expo{\zeta} - \lrp{\mbit +2}m_1}-\fmin \in (-2^{1+\emin},2^{1+\emin})_\fpq$, by \cref{lem:subnorm_inverse}, there exist $\widetilde{\beta}_1, \widetilde{\beta}_2\in \fpq$ such that
    \begin{equation}
      2^{\expo{\zeta} - \lrp{\mbit +2}m_1}- \fmin =  \left(  \widetilde{\beta}_1 \otimes K \right) \oplus \left(  \widetilde{\beta}_2 \otimes K\right). \label{eq:two_betas}
    \end{equation}
We define $f_1(x)$ as
    \begin{equation*}
        f_1(x) \defeq \begin{cases}
        x \oplus \left(  \widetilde{\beta}_1 \otimes K \right) \oplus \left(  \widetilde{\beta}_2 \otimes K \right) \quad &\text{if} \quad m_1 =0  \\
        x \oplus \left(  \widetilde{\beta}_1 \otimes K \right) \oplus \left(  \widetilde{\beta}_2 \otimes K \right) \oplus \bigoplus_{i=1}^{m_1}   \alpha'_i \otimes K \quad &\text{if} \quad m_1 \ge 1
        \end{cases}
    \end{equation*}
By \cref{eq:alphapK} we have
    \begin{equation*}
        f_1(x) =
        \begin{cases}
            x \oplus \left(  \widetilde{\beta}_1 \otimes K \right) \oplus \left(  \widetilde{\beta}_2 \otimes K \right) &\text{if} \quad m_1 =0  \\
            x \oplus \left(  \widetilde{\beta}_1 \otimes K \right) \oplus \left(  \widetilde{\beta}_2 \otimes K \right) \oplus \bigoplus_{i=1}^{m_1} \beta_i &\text{if} \quad m_1 \ge 1
        \end{cases}
    \end{equation*}

Next, we present the following claim.

\paragraph{\underline{\bf Claim 1-1:}}
For any $k \in [m_1]$, we have
\begin{equation*}
  \left(  \widetilde{\beta}_1 \otimes K \right) \oplus \left(  \widetilde{\beta}_2 \otimes K \right) \oplus \bigoplus_{i=1}^k\beta_i = \beta_k.
\end{equation*}
We show the claim using the induction on $k$.

{\bf Base step ($k=1$): } \\
By \cref{eq:two_betas}, we have
\begin{align*}
&\left(  \widetilde{\beta}_1 \otimes K \right) \oplus \left(  \widetilde{\beta}_2 \otimes K \right) \oplus \beta_1
    =   \lrp{2^{\expo{\zeta} - \lrp{\mbit +2}m_1}- \fmin}\oplus 0.\underbrace{1\dots 1}_{\mbit+1 \; \text{times}} \times 2^{\expo{\zeta} -\lrp{\mbit+2}\lrp{m_1 - 1}} \\
    =& \round{ 0.\underbrace{1\dots 1}_{\mbit+1 \; \text{times}} 0 \underbrace{1\dots 1}_{\expo{\zeta} - \lrp{\mbit +2}m_1-1- \emin +\mbit \text{ times } } \times 2^{\expo{\zeta} -\lrp{\mbit+2}\lrp{m_1 - 1}} } \\
    &=
    0.\underbrace{1\dots 1}_{\mbit+1 \; \text{times}} \times 2^{\expo{\zeta} -\lrp{\mbit+2}\lrp{m_1 - 1}} = \beta_1.
\end{align*}

{\bf Induction step: } \\
Assume that the induction hypothesis is satisfied for $k$. Then we have
\begin{align*}
     \bigoplus_{i=1}^{k+1}\beta_i
     &= \beta_k \oplus \beta_{k+1} = \round{0.\underbrace{1\dots 1}_{{\mbit+1} \; \text{times}}0\underbrace{1\dots 1}_{\mbit+1 \; \text{times}}\times 2^{\expo{\zeta} -\lrp{\mbit+2}\lrp{m_1 - k-1}}}
         \\
         &=0.\underbrace{1\dots 1}_{\mbit+1 \; \text{times}} \times 2^{\expo{\zeta} -\lrp{\mbit+2}\lrp{m_1 - k-1}}=\beta_{k+1}.
\end{align*}
Therefore, we prove the claim for any $k \in [m_1]$. \\

Thus, we have
\begin{equation}
    f_1(0) = \begin{cases}
            \left(  \widetilde{\beta}_1 \otimes K \right) \oplus \left(  \widetilde{\beta}_2 \otimes K \right) = 2^{\expo{\zeta}}- \fmin < 2^{\expo{\zeta} } &\text{if} \quad m_1 =0 , \\
            \beta_{m_1} = 0.\underbrace{1\dots 1}_{\mbit+1 \; \text{times}} \times 2^{\expo{\zeta}} < 2^{\expo{\zeta}} &\text{if} \quad m_1 \ge 1.
        \end{cases}
    \label{eq:condforcase1_1}
\end{equation}

Now, we present the following claim.

\paragraph{\underline{\bf Claim 1-2:}}
For any $k \in [m_1]$, we have
\begin{equation}
      \fmin\oplus  \left(  \widetilde{\beta}_1 \otimes K \right) \oplus \left(  \widetilde{\beta}_2 \otimes K \right) \oplus \bigoplus_{i=1}^{k}\beta_i\ge  2^{\expo{\zeta} - (\mbit+2)(m_1-k)} \label{eq:claim1-2}.
\end{equation}
We show the claim using the induction on $k$.

{\bf Base step ($k=1$): } \\
As $\widetilde{\beta}_1 \otimes K  < 2^{\emin +1}$, the summation $\fmin \oplus (K\otimes \widetilde{\beta}_1)$ is exact. Thus,
\begin{align*}
    &\fmin \oplus \left(  \widetilde{\beta}_1 \otimes K \right) \oplus \left(  \widetilde{\beta}_2 \otimes K \right)
     = \lrp{\fmin + K\otimes \widetilde{\beta}_1}\oplus \left( K\otimes \widetilde{\beta}_2 \right)
     = \round{\fmin + K\otimes \widetilde{\beta}_1 + K\otimes \widetilde{\beta}_2}_{\fpq}
   \\ &=  \round{\fmin + \lrp{K\otimes \widetilde{\beta}_1 \oplus K\otimes \widetilde{\beta}_2 }}_{\fpq}
      = 2^{\expo{\zeta} - \lrp{\mbit +2}m_1}.
\end{align*}
Therefore, we have
\begin{align*}
\fmin \oplus \left(  \widetilde{\beta}_1 \otimes K \right) \oplus \left(  \widetilde{\beta}_2 \otimes K \right) \oplus \beta_1
&= \round{ 0.\underbrace{1\dots 1}_{\mbit+2 \; \text{times}}  \times 2^{\expo{\zeta} -\lrp{\mbit+2}\lrp{m_1 - 1}} } \\
&= 2^{\expo{\zeta} - (\mbit+2)(m_1-1)}.
\end{align*}

{\bf Induction step: } \\
Assume that the induction hypothesis is satisfied for $k$ and consider the case of $k+1$. By the induction hypothesis,
\begin{align*}
& \fmin \oplus \left(  \widetilde{\beta}_1 \otimes K \right) \oplus \left(  \widetilde{\beta}_2 \otimes K \right) \oplus \bigoplus_{i=1}^{k+1}\beta_i
    \ge 2^{\expo{\zeta} - (\mbit+2)(m_1-k)} \oplus \beta_{k+1}
     \\&= \round{0.\underbrace{1\dots 1}_{\mbit+2 \text{ times}} \times 2^{\expo{\zeta} - (\mbit+2)(m_1-k-1)} } = 2^{\expo{\zeta} - (\mbit+2)(m_1-k-1)}.
\end{align*}
Thus, the induction hypothesis holds for any $k\in [m_1]$, which proves the claim. \\
If $k=m_1$ in \cref{eq:claim1-2}, we have
\begin{equation}
    f_1\lrp{\fmin} =
     \begin{cases}
            \fmin \oplus \left(  \widetilde{\beta}_1 \otimes K \right) \oplus \left(  \widetilde{\beta}_2 \otimes K \right) = 2^{\expo{\zeta}} \ge 2^{\expo{\zeta} } &\text{if} \quad m_1 =0,  \\
            \fmin\oplus \left(  \widetilde{\beta}_1 \otimes K \right) \oplus \left(  \widetilde{\beta}_2 \otimes K \right) \oplus \bigoplus_{i=1}^{m_1}\beta_i \ge 2^{\expo{\zeta}} &\text{if} \quad m_1 \ge 1.
        \end{cases}
    \label{eq:condforcase1_2}
\end{equation}

For $x = 1.1\times 2^{\expo{\zeta}}$, we consider three cases with respect to $\expo{\zeta}$: $\expo{\zeta} \ge \emin + 2$, $\expo{\zeta} = \emin + 1$, and $\expo{\zeta} \le \emin$.

If $\expo{\zeta} \ge \emin + 2$, we have $m_1 \ge 1$. Since  $ \widetilde{\beta}_1 \otimes K , \widetilde{\beta}_2 \otimes K, \beta_{i} < 2^{\expo{\zeta}}\times 2^{-\mbit-1}$ for $i\in [m_1-1]$, we have
\begin{align}
      &f_1\lrp{ 1.1\times 2^{\expo{\zeta}}} =  1.1\times 2^{\expo{\zeta}} \oplus \left(  \widetilde{\beta}_1 \otimes K \right) \oplus \left(  \widetilde{\beta}_2 \otimes K \right) \oplus \bigoplus_{i=1}^{m_1}\beta_i
    =  1.1\times 2^{\expo{\zeta}}\oplus  \beta_{m_1} \nonumber
    \\& = 1.1\times 2^{\expo{\zeta}}
   \oplus 0.\underbrace{1\dots 1}_{\mbit+1 \text{ times}} \times 2^{\expo{\zeta}}
 \le 1.01 \times 2^{\expo{\zeta} +1} < 1.1\times 2^{\expo{\zeta} +1}.  \label{eq:condforcase1_31}
\end{align}
If $\expo{\zeta} = \emin + 1$, we have $m_1=0$, which leads to $\left(  \widetilde{\beta}_1 \otimes K \right) \oplus \left(  \widetilde{\beta}_2 \otimes K \right) = 2^{\expo{\zeta}}-\fmin$ by \cref{eq:two_betas}. Therefore we have,
\begin{align}
    &f_1\lrp{ 1.1\times 2^{\expo{\zeta}}} = 1.1\times 2^{\expo{\zeta}} \oplus  \left(  \widetilde{\beta}_1 \otimes K \right) \oplus \left(  \widetilde{\beta}_2 \otimes K \right) \nonumber
    \\
    &\le \round{1.1\times 2^{\expo{\zeta}} +  \left(  \widetilde{\beta}_1 \otimes K \right) + \left(  \widetilde{\beta}_2 \otimes K \right)+ 2\fmin }_{\fpq} \nonumber
\\ &\le \round{1.01\times 2^{\expo{\zeta}+1} + 2\fmin }_{\fpq}
    < 1.1\times 2^{\expo{\zeta}+1}. \label{eq:condforcase1_32}
\end{align}
If $\expo{\zeta} \le \emin$, we also have $m_1=0$ and $\left(  \widetilde{\beta}_1 \otimes K \right) \oplus \left(  \widetilde{\beta}_2 \otimes K \right) = 2^{\expo{\zeta}}-\fmin$.
Since the summation is exact, we have
\begin{align}
    f_1\lrp{ 1.1\times 2^{\expo{\zeta}}} \le 1.01 \times 2^{\expo{\zeta}+1} < 1.1\times 2^{\expo{\zeta} +1}.\label{eq:condforcase1_33}
\end{align}

For $x = -1.1\times 2^{\expo{\zeta}}$, we consider three cases with respect to $\expo{\zeta}$, that is, $\expo{\zeta} \ge \emin + 2$, $\expo{\zeta} = \emin + 1$, and $\expo{\zeta} \le \emin$. \\

If $\expo{\zeta} \ge \emin+2$, we have $m_1 \ge 1$. Hence
\begin{align}
     f_1\lrp{ -1.1\times 2^{\expo{\zeta}}}
     &\ge - 1.1\times 2^{\expo{\zeta}} \oplus \left(  \widetilde{\beta}_1 \otimes K \right) \oplus \left(  \widetilde{\beta}_2 \otimes K \right) \oplus \bigoplus_{i=1}^{m_1}\beta_i \label{eq:condforcase1_41}
     \\
     &\ge - 1.1\times 2^{\expo{\zeta}} \oplus \beta_m
     = - 1.1\times 2^{\expo{\zeta}} \oplus 0.\underbrace{1\dots 1}_{\mbit+1 \text{ times}} \times 2^{\expo{\zeta}}
     = -2^{\expo{\zeta}-1}. \nonumber
\end{align}
If $\expo{\zeta} = \emin+1$, we have $m_1 =0$. Therefore,
\begin{align}
     f_1\lrp{ -1.1\times 2^{\expo{\zeta}}} &\ge -1.1\times 2^{\expo{\zeta}} \oplus \left(  \widetilde{\beta}_1 \otimes K \right) \oplus \left(  \widetilde{\beta}_2 \otimes K \right) \label{eq:condforcase1_42} \\
     &\ge
     \round{ -1.1\times 2^{\expo{\zeta}} + \left(  \widetilde{\beta}_1 \otimes K \right) + \left(  \widetilde{\beta}_2 \otimes K \right) - 2\fmin }_{\fpq}
     > - 2^{\expo{\zeta}}. \nonumber
\end{align}
If $\expo{\zeta}\le \emin$, we also have $m_1=0$. Since the summation is exact, we have
\begin{align}
  f_1\lrp{ -1.1\times 2^{\expo{\zeta}}} = - 2^{\expo{\zeta}-1} - \fmin >  - 2^{\expo{\zeta}-1}  \label{eq:condforcase1_43}.
\end{align}

By \cref{eq:condforcase1_1,eq:condforcase1_2,eq:condforcase1_31,eq:condforcase1_32,eq:condforcase1_33,eq:condforcase1_41,eq:condforcase1_42,eq:condforcase1_43}, we show \cref{eq:condforcase1}. Therefore we conclude that
    \begin{equation*}
        f^\sharp\lrp{ \langle -1.1\times 2^{\expo{\zeta}},0\rangle }\subset \langle -(2^{\expo{\zeta}})^-,(2^{\expo{\zeta}})^-\rangle,
        \;
        f^\sharp \left( \langle\fmin, 1.1\times 2^{\expo{\zeta}}\rangle \right) \subset \langle 2^{\expo{\zeta}}, (1.1\times 2^{\expo{\zeta}+1})^-\rangle,
    \end{equation*}
To show \cref{eq:telescoping_further}, we can apply
by similar argument to $x=1.1\times 2^{\expo{\zeta}}$.
If $\expo{\zeta} \ge \emin+2$, we have
\begin{equation*}
    f_1(x) =  x\oplus \beta_m \le x\oplus 2^{\expo{\zeta}} \le  x\oplus (2^{\expo{\zeta}})^+.
\end{equation*}
If $\expo{\zeta} \le \emin+1$, we have
\begin{equation*}
    f_1(x) \le  \round{ x \oplus \beta_m + 2\fmin }\le x\oplus (2^{\expo{\zeta}})^+.
\end{equation*}


\paragraph{\fbox{\bf Case 2: $\mant{K}^{\parallel}$ exists.}}~

In this case, we show \cref{eq:condforcase2}.
Define $m_2 \in \bbN_{\ge 0}$ as
\begin{align*}
    m_2 \defeq  \floorZ{\frac{ -\emin + \mbit   + \expo{\zeta}}{\mbit+1}}  \in \bbN_{\ge 1}  .
\end{align*}
Then $m_2$ is the unique non-negative natural number satisfying
    \begin{equation*}
       \expo{\zeta} - \lrp{\mbit+1}m_2 < \emin + 1   \le \expo{\zeta} - \lrp{\mbit+1}(m_2-1) ,
    \end{equation*}
or equivalently,
\begin{equation}
     \emin -\mbit  \le \expo{\zeta} - \lrp{\mbit+1}m_2 < \emin + 1.\label{eq:expo_zeta_case2}
\end{equation}


   For $i\in [m_2]$, define $\beta_{i}$ as
    \begin{equation*}
        \beta_{i}\defeq 2^{\expo{\zeta} -\lrp{M+1}\lrp{m_2-i}}.
    \end{equation*}
    Then, as
    \begin{equation*}
       \emin +1\le \expo{\zeta} -\lrp{\mbit+1}\lrp{m-1} \le \expo{\beta_i} =
 \expo{\zeta} -\lrp{\mbit+1}\lrp{m-i} \le  \expo{\zeta}\le \emax-\mbit-1,
    \end{equation*}
we have $\beta_i\in \fpq$.
For $i\in [n]$, define $\alpha_i$ as
    \begin{equation*}
        \alpha'_i\defeq \begin{cases}
            \mant{K}^{\parallel} \times 2^{\expo{\zeta} -\lrp{M+1}\lrp{m_2-i}-\expo{K}} \quad &\text{if} \quad 2^{-1} \le \mant{K}^{\parallel} < 1 \\
            2^{\expo{\zeta} -\lrp{M+1}\lrp{m_2-i}-\expo{K}-1} \quad &\text{if} \quad \mant{K}^{\parallel} =1
        \end{cases}
\end{equation*}
As $0\le -\expo{K}\le \mbit+2$, we have
\begin{equation*}
   \emin \le \expo{\alpha'_i} =   \expo{\zeta} -\lrp{M+1}\lrp{m_2-i} -\expo{K} -1 \le \emax.
\end{equation*}
Hence we have $\alpha'_i\in \fpq$.

Consider $ 2^{\expo{\zeta} - \lrp{\mbit +1}m_2}$.
Since $ \emin -\mbit \le \expo{\zeta} - \lrp{\mbit +1}m_2\le \emin$, we have  $ 2^{\expo{\zeta} - \lrp{\mbit +1}m_2}\in \fpq$ and $2^{\expo{\zeta} - \lrp{\mbit +1}m_2} \in (-2^{1+\emin},2^{1+\emin})_\fpq$.
By \cref{lem:subnorm_inverse}, there exist $\widetilde{\beta}_1, \widetilde{\beta}_2\in \fpq$ such that
\begin{equation*}
   2^{\expo{\zeta} - \lrp{\mbit +1}m_2} = \left( \widetilde{\beta}_1 \otimes K \right) \oplus \left( \widetilde{\beta}_2 \otimes K \right).
\end{equation*}
Define $f_2(x)$ as
   \begin{align*}
        f_2(x)
        &\defeq x \oplus \left( \widetilde{\beta}_1 \otimes K \right) \oplus \left( \widetilde{\beta}_2 \otimes K \right) \oplus \bigoplus_{i=1}^{m_2} {\alpha'_i\otimes K}
        \\
        &= x \oplus \left( \widetilde{\beta}_1 \otimes K \right) \oplus \left( \widetilde{\beta}_2 \otimes K \right) \oplus \bigoplus_{i=1}^{m_2} {\beta_i}.
    \end{align*}
We present the following claim.

\paragraph{\underline{\bf Claim 2-1:}}
For any $k \in [m_2]$, we have
\begin{equation*}
 \left( \widetilde{\beta}_1 \otimes K \right) \oplus \left( \widetilde{\beta}_2 \otimes K \right) \oplus \bigoplus_{i=1}^k\beta_i = \beta_k.
\end{equation*}
We show the claim using the induction on $k$.

{\bf Base step ($k=1$): } \\
\begin{align*}
     \left( \widetilde{\beta}_1 \otimes K \right) \oplus \left( \widetilde{\beta}_2 \otimes K \right) \oplus\beta_1
     &= 2^{\expo{\zeta} - \lrp{\mbit +1}m_2} \oplus 2^{\expo{\zeta} - \lrp{\mbit +1}\lrp{m_2-1}} \\
     &=  2^{\expo{\zeta} - \lrp{\mbit +1}\lrp{m_2-1}} = \beta_1.
\end{align*}

{\bf Induction step: } \\
Assume that the induction hypothesis is satisfied for $k$. Then we have
\begin{align*}
     &\left( \widetilde{\beta}_1 \otimes K \right) \oplus \left( \widetilde{\beta}_2 \otimes K \right) \oplus \bigoplus_{i=1}^{k+1}\beta_i
     = \beta_k \oplus \beta_{k+1} \\
     &= \round{1.\underbrace{0\dots 0}_{{\mbit} \; \text{times}}1\times 2^{\expo{\zeta} -\lrp{\mbit+1}\lrp{m_2 - k-1}}}
         = 2^{\expo{\zeta} -\lrp{\mbit+1}\lrp{m_2 - k-1}}=\beta_{k+1}.
\end{align*}
Therefore the induction hypothesis is satisfied for any $k\le m$, and we prove the claim.
Thus,
\begin{equation}
    f_2(0) = \beta_{m_2} = 2^{\expo{\zeta}}. \label{eq:condforcase2_1}
\end{equation}

Now we will show that for $f_2(\fmin)> 2^{\expo{\zeta}}$. We present the following claim.

\paragraph{\underline{\bf Claim 2-2:}} For any $k \in [m_2]$, we have
\begin{align*}
    {\fmin} \oplus  \left( \widetilde{\beta}_1 \otimes K \right) \oplus \left( \widetilde{\beta}_2 \otimes K \right) \oplus\bigoplus_{i=1}^k\beta_i \ge \beta_k^+.
\end{align*}
We show the claim using the induction on $k$.

{\bf Base step ($k=1$): } \\
Since $\fmin \oplus  \left( \widetilde{\beta}_1 \otimes K \right)$ is exact, we have
\begin{align*}
    &\fmin \oplus  \left( \widetilde{\beta}_1 \otimes K \right) \oplus \left( \widetilde{\beta}_2 \otimes K \right)
     = \lrp{\fmin +K\otimes \widetilde{\beta}_1} \oplus \left( K\otimes \widetilde{\beta}_2 \right) \\
  =& \round{\fmin +K\otimes \widetilde{\beta}_1 + K\otimes \widetilde{\beta}_2}_{\fpq}
  = 2^{\expo{\zeta}-\lrp{M+1}m_2} + \fmin > 2^{\expo{\zeta}-\lrp{M+1}m_2}.
\end{align*}
Thus,
\begin{align*}
   &\fmin \oplus  \left( \widetilde{\beta}_1 \otimes K \right) \oplus \left( \widetilde{\beta}_2 \otimes K \right) \oplus \beta_1
     \ge \lrp{2^{\expo{\zeta}-\lrp{M+1}m_2}+\fmin}\oplus 2^{\expo{\zeta} -\lrp{M+1}\lrp{m_2-1}}
     \\
     &= \lrp{2^{\expo{\zeta}-\lrp{M+1}\lrp{m_2-1}}}^+=\beta_1^+.
\end{align*}

{\bf Induction step: }
Assume that the induction hypothesis is satisfied for $k$ and consider the case of $k+1$.
\begin{align*}
     &\fmin \oplus  \left( \widetilde{\beta}_1 \otimes K \right) \oplus \left( \widetilde{\beta}_2 \otimes K \right)  \oplus\bigoplus_{i=1}^{k+1}\beta_i \ge \beta_k^+ \oplus \beta_{k+1} \\
    &=  1.\underbrace{0\dots 0}_{\mbit-1 \text{ times}}1 \times 2^{\expo{\zeta} -\lrp{M+1}\lrp{m-k}} \oplus 2^{\expo{\zeta} -\lrp{M+1}\lrp{m-k-1}}
    \\&= \round{1.\underbrace{0\dots 0}_{\mbit \text{ times}}1 \underbrace{0\dots 0}_{\mbit-1 \text{ times}} 1 \times 2^{\expo{\zeta} -\lrp{M+1}\lrp{m-k-1} } } \\
    &=   (1+2^{-\mbit} ) \times 2^{\expo{\zeta} -\lrp{M+1}\lrp{m-k-1}}
    = \beta_{k+1}^+.
\end{align*}
Thus, the induction hypothesis is satisfied for any $k\in [m_2]$ and we prove the claim. Therefore, we have
\begin{equation}
 f_2(\fmin) = {\fmin} \oplus  \left( \widetilde{\beta}_1 \otimes K \right) \oplus \left( \widetilde{\beta}_2 \otimes K \right) \oplus\bigoplus_{i=1}^m\beta_i \ge  \beta_m^+ > \beta_m =2^{\expo{\zeta}}. \label{eq:condforcase2_2}
\end{equation}

Now, we show \cref{eq:condforcase2} for  $x = 1.1\times 2^{\expo{\zeta}}$ and $x = -1.1\times 2^{\expo{\zeta}}$. We consider two cases with respect to $\expo{\zeta}$: $\expo{\zeta} \le \emin$ and $\expo{\zeta} \ge \emin + 1$.
If $\expo{\zeta}\le \emin$, since the summation exact, we have the desired results. \\
If $\expo{\zeta}\ge \emin +1 $ and $x = 1.1\times 2^{\expo{\zeta}}$, since  $\beta_{i} < 2^{\expo{\zeta}}\times 2^{-\mbit-1}$ for $i\in [m_2-2]$, we have
\begin{align}
      f_2\lrp{1.1\times 2^{\expo{\zeta}}}
    &= \lrp{1.1\times 2^{\expo{\zeta}}}\oplus  \beta_{n-1}\oplus  \beta_n
    =  1.1\times 2^{\expo{\zeta}}
    \oplus 2^{\expo{\zeta}-M-1}\oplus 2^{\expo{\zeta}} \nonumber
    \\
    &=\lrp{1.1\times 2^{\expo{\zeta}}}^+ \oplus 2^{\expo{\zeta}}
<1.1 \times 2^{\expo{\zeta}+1} . \label{eq:condforcase2_3}
\end{align}
If $\expo{\zeta}\ge \emin +1 $ and $x = - 1.1\times 2^{\expo{\zeta}}$, we have
\begin{equation}
     f_2\lrp{ - 1.1\times 2^{\expo{\zeta}}} \ge -  1.1\times 2^{\expo{\zeta}}\oplus \bigoplus_{i=1}^{m_2}\beta_i
     \ge -1.1\times 2^{\expo{\zeta}} \oplus \beta_m = - 2^{\expo{\zeta}-1} > - 2^{\expo{\zeta}}. \label{eq:condforcase2_4}
\end{equation}

Due to \cref{eq:condforcase2_1,eq:condforcase2_2,eq:condforcase2_3,eq:condforcase2_4}, we show \cref{eq:condforcase2}. Therefore we conclude that
\begin{equation*}
   f^\sharp\lrp{ \langle -1.1\times 2^{\expo{\zeta}},0\rangle}\subset \langle-(2^{\expo{\zeta}})^-,2^{\expo{\zeta}}\rangle,
\;
f^\sharp \left( \langle \fmin, 1.1\times 2^{\expo{\zeta}}  \rangle \right) \subset \langle (2^{\expo{\zeta}})^+, (1.1\times 2^{\expo{\zeta}+1})^- \rangle .
\end{equation*}
To show \cref{eq:telescoping_further}, we consider two cases: $ 2^{\expo{\zeta}} \le x < 2^{\expo{\zeta}+1}$ and $x\ge 2^{\expo{\zeta}+1}$.

If $  2^{\expo{\zeta}} \le x < 2^{\expo{\zeta}+1}$, we have
\begin{equation*}
    f_2(x)= x\oplus \beta_{m_2-1} \oplus \beta_{m_2}
   \le x^+\oplus 2^{\expo{\zeta}}
    = x\oplus \lrp{2^{\expo{\zeta}}}^+.
\end{equation*}
If $x\ge 2^{\expo{\zeta}+1}$, we have
\begin{equation*}
    f_2(x)= x\oplus \beta_{m_2} \le x\oplus 2^{\expo{\zeta}} \le x\oplus \lrp{2^{\expo{\zeta}}}^+.
\end{equation*}
Hence we show \cref{eq:telescoping_further}.

Finally, note that in both cases, if $\expo{\zeta}\le \emin$, then, $m_1=0$ and $m_2=0$, which implies that we do not need the existence of $\mant{K}^{\dag}$ and $\mant{K}^{\parallel}$ in the definition of $f_1(x)$ and $f_2(x)$.
Therefore, if $\expo{\zeta}\le \emin$, for both statements \cref{eq:pseudo_inverse} and \cref{eq:true_inverse}, there exists $n$ and $\alpha_i$ such that satisfying the statements.

This completes the proof.
\end{proof}

\begin{lemma}\label{lem:contraction2}
Suppose that $\sigma:\fpq\to\fpq$ satisfies \cref{cond:activation_2r}.
Assume that there exists an integer $e_0\in \bbZ$ such that $  2^{e_0} \le \left|\sigma\lrp{\eta^+}- \sigma\lrp{\eta}\right| <  2^{e_0+1} $
and define $\expo{\theta}$ as
\begin{equation*}
    \expo{\theta}\defeq \max\lrp{\emin-\mbit, -e_0+\emin-\mbit+1}.
\end{equation*}

Suppose that $\expo{\theta} \le  -3$, $e_0\le \expo{\eta}-\mbit-3-\expo{\theta}$.



If $\sigma(\eta)< \sigma(\eta^+)$, define $\theta, \mcI, \mcI^+$ as
\begin{equation*}
    \theta \defeq 2^{\expo{\theta}}, \mcI\defeq \langle \sigma(\eta)-2^{\expo{\zeta}-\expo{\theta} },\sigma(\eta)\rangle, \text{ and } \mcI^+\defeq \langle \sigma(\eta^+),\sigma(\eta)+2^{\expo{\zeta}-\expo{\theta} } \rangle,
\end{equation*}
and if $\sigma(\eta)> \sigma(\eta^+)$, define $\theta, \mcI, \mcI^+$ as
\begin{equation*}
  \theta \defeq -2^{\expo{\theta}},\mcI^+ \defeq \langle\sigma(\eta)-2^{\expo{\zeta}-\expo{\theta} },\sigma\lrp{\eta^+}\rangle, \text{ and } \mcI\defeq \langle \sigma(\eta),\sigma(\eta)+2^{\expo{\zeta}-\expo{\theta} }\rangle,
\end{equation*}
where
\begin{equation*}
    \expo{\zeta}\defeq \begin{cases}
        \expo{\eta} - \mbit - 1 &\text{ if } \eta >0\; \text{or} \; \eta <0, \eta \neq -2^{\expo{\eta}},
        \\ \expo{\eta} - \mbit - 2 &\text{ if } \eta <0,\eta = -2^{\expo{\eta}}.
    \end{cases}
\end{equation*}
Then, there exists $k\in\bbN$ and $\alpha_1,\dots,\alpha_k,z_1,\dots,z_k\in\fpq$ such that for
\begin{align*}
f(x)=(\theta\otimes x)\oplus\bigoplus_{i=1}^k(\alpha_i\otimes\sigma(z_i))\oplus\eta,
\end{align*}
the followings hold:
\begin{align}
    f^\sharp(\langle -\fmax , \fmax \rangle ) \subset\langle -\fmax , \fmax \rangle , \quad  f^\sharp(\mcI) = \langle \eta,\eta \rangle, \quad f^\sharp(\mcI^+) = \langle \eta^+,\eta^+ \rangle . \label{eq:contraction2_main}
\end{align}
In addition, if $\sigma(\eta)< \sigma(\eta^+)$
\begin{align}
     f(x) - \eta^+\le \lrp{x-\sigma\lrp{\eta^+}}\times 2^{ \expo{\theta}+2} \quad &\text{for} \quad  \sigma(\eta) +  2^{\expo{\zeta}-\expo{\theta} } \le x  \in \fpq ,  \label{eq:contraction2_further1} \\
     \eta - f(x) \le \lrp{\sigma\lrp{\eta}-x}\times  2^{ \expo{\theta}+2} \quad &\text{for} \quad  \sigma(\eta)-2^{\expo{\zeta}-\expo{\theta} }  \ge x \in \fpq  \label{eq:contraction2_further2},
\end{align}
and if
$\sigma(\eta) >  \sigma(\eta^+)$, we have
\begin{align}
     \eta - f(x) \le \lrp{x-\sigma\lrp{\eta}}\times  2^{ \expo{\theta}+2} \quad \quad &\text{for} \quad \sigma(\eta) +  2^{\expo{\zeta}-\expo{\theta} } \le  x \in \fpq ,  \label{eq:contraction2_further3} \\
     f(x) - \eta^+\le \lrp{\sigma\lrp{\eta^+}-x}\times 2^{ \expo{\theta}+2} \quad &\text{for} \quad  \sigma(\eta)-2^{\expo{\zeta}-\expo{\theta} } \ge x \in \fpq   \label{eq:contraction2_further4},
\end{align}
\end{lemma}
\begin{proof}
First, $\mcI^+$ is not empty
because $|\sigma(\eta^+)-\sigma(\eta)| < 2^{e_0+1} \le 2^{\expo{\eta}-\mbit -2 -\expo{\theta}}
\le  2^{\expo{\zeta}-\expo{\theta}}$.
Since $f$ is monotone, we only need to consider the endpoints of $\mcI$ and $\mcI^+$.



Note that $2^{\expo{\zeta}},2^{\expo{\zeta}+1} \in \fpq$ (since $\expo{\eta} \ge 5 + \emin$)
and the following hold:
\begin{align}
    2^{\expo{\zeta}+1} &= \eta^+ - \eta, \; \lrp{ 1.1 \times 2^{\expo{\zeta}+1} }^- \oplus \eta = \eta^+ \label{eq:11ezetaplus1}.
\end{align}
Now, represent $\eta$ as
\begin{equation*}
    \eta = \mant{\eta}\times 2^{\expo{\eta}} = 1.\mant{\eta,1}\dots \mant{\eta,\mbit} \times 2^{\expo{\eta}}.
\end{equation*}
Then we have if $\mant{\eta,\mbit}= 0$, we have
\begin{align}
    \eta \oplus 2^{\expo{\zeta}} = \eta, \; \eta \oplus (-2^{\expo{\zeta}}) = \eta, \;  \eta \oplus \lrp{2^{\expo{\zeta}}}^+ = \eta^+, \; \eta \oplus \left(- \lrp{2^{\expo{\zeta}}}^+ \right) = \eta^-, \label{eq:ezetaeta1}
\end{align}
and if $\mant{\eta,\mbit}= 1$, we have
\begin{align}
     \eta \oplus \lrp{2^{\expo{\zeta}}}^- = \eta, \; \eta \oplus \left(-\lrp{2^{\expo{\zeta}}}^-\right) = \eta, \;  \eta \oplus 2^{\expo{\zeta}} = \eta^+, \; \eta \oplus (-2^{\expo{\zeta}}) = \eta^-. \label{eq:ezetaeta2}
\end{align}

By \cref{lem:telescoping}, there exists $n\in \bbN$, $i\in[n]$, $\widetilde{\alpha_i}\in \fpq$ such that for $g:\fpq\rightarrow\fpq$ defined as
    \begin{equation*}
        g(x)\defeq x\oplus \bigoplus_{i=1}^n\lrp{\widetilde{\alpha_i}\otimes K_{\sigma}},
    \end{equation*}
  one of the following statements holds:
    \begin{equation}\label{eq:case_psudo_inverse}
           g^\sharp\lrp{ \langle -1.1\times 2^{\expo{\zeta}},0\rangle }\subset \langle -(2^{\expo{\zeta}})^-,(2^{\expo{\zeta}})^-\rangle, \;
        g^\sharp \left( \langle\fmin, 1.1\times 2^{\expo{\zeta}}\rangle \right) \subset \langle 2^{\expo{\zeta}}, (1.1\times 2^{\expo{\zeta}+1})^-\rangle,
    \end{equation}
    or \begin{equation}\label{eq:case_true_inverse}
          g^\sharp\lrp{ \langle -1.1\times 2^{\expo{\zeta}},0\rangle}\subset \langle-(2^{\expo{\zeta}})^-,2^{\expo{\zeta}}\rangle, \;
        g^\sharp \left( \langle \fmin, 1.1\times 2^{\expo{\zeta}}  \rangle \right) \subset \langle (2^{\expo{\zeta}})^+, (1.1\times 2^{\expo{\zeta}+1})^- \rangle .
\end{equation}
with
\begin{align}
    g(x)\le  x\oplus \lrp{2^{\expo{\zeta}}}^+ \quad \text{if} \quad  2^{\expo{\zeta}} \le x \in \fpq. \label{eq:gxxop2ezp}
\end{align}

    Furthermore, if $\expo{\zeta}\le \emin$, there exist $g_1$ and $g_2$ such that they satisfy  \cref{eq:case_psudo_inverse} and \cref{eq:case_true_inverse}, respectively. Hence  we pick $g$ as $g = g_2$ if $\mant{\eta,\mbit}=0$ and $g=g_1$ if $\mant{\eta,\mbit}=1$.

   If $\expo{\zeta} \ge \emin$,  by \cref{lem:endbit_control}, there exists $\beta\in \fpq$ such that the following inequality holds:
    \begin{equation*}
      \frac{1}{2}\times 2^{\expo{\zeta}-\mbit}   < \beta\otimes K_{\sigma}< \frac{5}{4}\times 2^{\expo{\zeta}-\mbit}.
    \end{equation*}
Therefore we have
\begin{align}
     2^{\expo{\zeta}}\le 2^{\expo{\zeta}} \oplus \left(\beta\otimes K_{\sigma} \right) \le (2^{\expo{\zeta}})^+, \; -\left((2^{\expo{\zeta}})^+ \right) \le -2^{\expo{\zeta}} \oplus \left(\beta\otimes K_{\sigma} \right) \le -2^{\expo{\zeta}}. \label{eq:betaKsigma}
\end{align}
    Define $\tilde{\beta} \in \fpq$ as
    \begin{equation}
        \tilde{\beta} \defeq \begin{cases}
        0 &\text{ if } \expo{\zeta}\le \emin.
       \\     \beta &\text{ if } \expo{\zeta}\ge \emin+1, \mant{\eta,\mbit} = 0 \text{ and \cref{eq:case_psudo_inverse} holds},
            \\0 &\text{ if } \expo{\zeta}\ge \emin+1,\mant{\eta,\mbit} = 1 \text{ and \cref{eq:case_psudo_inverse} holds},
            \\0 &\text{ if } \expo{\zeta}\ge \emin+1, \mant{\eta,\mbit} = 0 \text{ and \cref{eq:case_true_inverse} holds},
            \\-\beta &\text{ if } \expo{\zeta}\ge \emin+1, \mant{\eta,\mbit} = 1 \text{ and \cref{eq:case_true_inverse}
            holds}.
        \end{cases} \label{eq:tildebeta}
    \end{equation}
    Define $f(x):\fpq\rightarrow\fpq$ as
    \begin{equation*}
        f(x) \defeq \lrp{\theta \otimes x} \oplus \lrp{-\theta \otimes \sigma\lrp{\eta}}\oplus \bigoplus_{i=1}^n\lrp{\widetilde{\alpha}_i\otimes K_{\sigma}} \oplus \lrp{\tilde{\beta} \otimes K_{\sigma}} \oplus \eta .
    \end{equation*}
    Now we analyze the abstract interval arithmetic of $f$.
    First, consider the function $\tilde{g}$ defined as
    \begin{equation}
        \tilde{g}(x)\defeq g(x)\oplus \lrp{\tilde{\beta}\otimes K_{\sigma}} \oplus \eta = x\oplus \bigoplus_{i=1}^n\lrp{\widetilde{\alpha_i}\otimes K_{\sigma}} \oplus \lrp{\tilde{\beta}\otimes K_{\sigma}} \oplus \eta. \label{eq:gtilde}
    \end{equation}
    Together with \cref{eq:ezetaeta1,eq:ezetaeta2,eq:betaKsigma,eq:tildebeta}, we have
      \begin{multline*}
         \tilde{g}^\sharp \langle -1.1\times 2^{\expo{\zeta}},0 \rangle\subset
         \\  \begin{cases}
            \langle -(2^{\expo{\zeta}})^-,2^{\expo{\zeta}} \rangle  \oplus^\sharp \eta &\subset \langle \eta, \eta \rangle\text{ if }  \expo{\zeta}\le \emin \text{ and }\mant{\eta,\mbit} = 0,
            \\  \langle -(2^{\expo{\zeta}})^-,(2^{\expo{\zeta}})^- \rangle  \oplus^\sharp \eta &\subset \langle \eta, \eta \rangle\text{ if }  \expo{\zeta}\le \emin \text{ and } \mant{\eta,\mbit} = 1,
            \\  \langle -(2^{\expo{\zeta}})^-,(2^{\expo{\zeta}})^- \rangle \oplus^\sharp \lrp{\tilde{\beta} \otimes K_{\sigma}} \oplus^\sharp \eta &\subset \langle \eta, \eta \rangle\text{ if }  \expo{\zeta}\ge \emin+1,\mant{\eta,\mbit} = 0 \text{ and \eqref{eq:case_psudo_inverse} holds},
            \\   \langle -(2^{\expo{\zeta}})^-,(2^{\expo{\zeta}})^- \rangle \oplus^\sharp \eta &\subset \langle \eta, \eta \rangle \text{ if }  \expo{\zeta}\ge \emin+1,\mant{\eta,\mbit} = 1 \text{ and \eqref{eq:case_psudo_inverse} holds},
            \\  \langle -(2^{\expo{\zeta}})^-,2^{\expo{\zeta}} \rangle \oplus^\sharp \eta  &\subset \langle \eta, \eta \rangle \text{ if }  \expo{\zeta}\ge \emin+1,\mant{\eta,\mbit} = 0 \text{ and \eqref{eq:case_true_inverse} holds},
            \\ \langle -(2^{\expo{\zeta}})^-,2^{\expo{\zeta}} \rangle \oplus^\sharp \lrp{-\tilde{\beta}\otimes K_{\sigma}} \oplus^\sharp \eta &\subset \langle \eta, \eta \rangle \text{ if }  \expo{\zeta}\ge \emin+1,\mant{\eta,\mbit} = 1 \text{ and \eqref{eq:case_true_inverse} holds},
        \end{cases}
    \end{multline*}
    and thus, we have $g'\langle -1.1\times 2^{\expo{\zeta}},0 \rangle = \langle \eta, \eta \rangle$.
    Similarly, together with \cref{eq:11ezetaplus1,eq:ezetaeta1,eq:ezetaeta2,eq:betaKsigma,eq:tildebeta}, we have  $ \tilde{g}^\sharp \lrp{\langle \fmin, 1.1\times 2^{\expo{\zeta}} \rangle} = \langle \eta^+, \eta^+ \rangle $ by the following argument:
        \begin{multline*}
        \!\!\!\!\!\!\!\!\!\!\!\!\!\!\!\!\!\!\!\!\!\!\!\!
        \tilde{g}^\sharp \lrp{\langle \fmin, 1.1\times 2^{\expo{\zeta}} \rangle}
        \\
        \!\!\!\!\!\!\!\!\!\!\!\!\!\!\!\!\!\!\!\!\!\!\!\!
        \subset
        \begin{cases}
          \langle (2^{\expo{\zeta}})^+, (1.1\times 2^{\expo{\zeta}+1})^- \rangle \oplus^\sharp \eta &\subset  \langle \eta^+, \eta^+  \rangle \text{ if }\expo{\zeta}\le \emin \text{ and } \mant{\eta,\mbit} = 0,
            \\\langle 2^{\expo{\zeta}}, (1.1\times 2^{\expo{\zeta}+1})^- \rangle \oplus^\sharp \eta  &\subset  \langle \eta^+, \eta^+  \rangle \text{ if }\expo{\zeta}\le \emin \text{ and } \mant{\eta,\mbit} = 1,
           \\\langle 2^{\expo{\zeta}}, (1.1\times 2^{\expo{\zeta}+1})^- \rangle \oplus^\sharp \lrp{ \tilde{\beta}\otimes K_{\sigma}} \oplus^\sharp \eta &\subset  \langle \eta^+, \eta^+  \rangle\text{ if }\expo{\zeta}\ge \emin+1, \mant{\eta,\mbit} = 0 \text{ and \eqref{eq:case_psudo_inverse} holds},
            \\ \langle 2^{\expo{\zeta}}, (1.1\times 2^{\expo{\zeta}+1})^- \rangle \oplus^\sharp \eta &\subset  \langle \eta^+, \eta^+  \rangle \text{ if }\expo{\zeta}\ge \emin+1, \mant{\eta,\mbit} = 1 \text{ and \eqref{eq:case_psudo_inverse} holds},
            \\ \langle (2^{\expo{\zeta}})^+, (1.1\times 2^{\expo{\zeta}+1})^- \rangle \oplus^\sharp \eta  &\subset  \langle \eta^+, \eta^+  \rangle \text{ if }\expo{\zeta}\ge \emin+1, \mant{\eta,\mbit} = 0 \text{ and \eqref{eq:case_true_inverse} holds},
            \\ \langle (2^{\expo{\zeta}})^+, (1.1\times 2^{\expo{\zeta}+1})^- \rangle \oplus^\sharp \lrp{-\tilde{\beta}\otimes K_{\sigma}} \oplus^\sharp \eta &\subset \langle \eta^+, \eta^+  \rangle \text{ if } \expo{\zeta}\ge \emin+1,\mant{\eta,\mbit} = 1 \text{ and \eqref{eq:case_true_inverse} holds}.
        \end{cases}
    \end{multline*}
Now we define $f(x)$ and $h(x)$ as
\begin{align*}
    h(x) &\defeq \begin{cases}
        ( 2^{\expo{\theta}} \otimes x) \oplus (-2^{\expo{\theta}} \otimes \sigma(\eta)) \quad &\text{if} \quad \sigma(\eta) < \sigma(\eta^+), \\
        ( -2^{\expo{\theta}} \otimes x) \oplus (2^{\expo{\theta}} \otimes \sigma(\eta)) \quad &\text{if} \quad \sigma(\eta) > \sigma(\eta^+),
    \end{cases} \\
    f(x) &\defeq h(x)  \oplus \bigoplus_{i=1}^n\lrp{\widetilde{\alpha}_i\otimes K_{\sigma}} \oplus \lrp{\tilde{\beta} \otimes K_{\sigma}} \oplus \eta,
\end{align*}
We need to show the followings to show \cref{eq:contraction2_main}.
\begin{align}\label{eq:temp_interval1}
     h^\sharp \left( \mcI \right)
     &\subset  \langle -1.1\times 2^{\expo{\zeta}},0\rangle , \\
     \label{eq:temp_interval2}
     h^\sharp \left( \mcI^+\right)
     &\subset \langle \fmin, 1.1\times 2^{\expo{\zeta}} \rangle.
\end{align}
To show this, we consider the following cases.

\paragraph{\fbox{\bf Case 1: $\sigma(\eta)< \sigma(\eta^+)$.}}~

In this case, we need to show
\begin{align}
\begin{cases}
    h \left( \sigma(\eta) - 2^{\expo{\zeta}-\expo{\theta}} \right) \ge - 1.1 \times 2^{\expo{\zeta}} \\
    h\left( \sigma(\eta) \right)\le 0 \\
   h\left( \sigma(\eta^+)  \right) \ge \fmin \\
    h\left( \sigma(\eta^+) + 2^{\expo{\zeta}-\expo{\theta}} \right) \le 1.1 \times 2^{\expo{\zeta}}.
    \end{cases}. \label{eq:hxcase1}
\end{align}

First, note that for any $\gamma \in \fpq$, unless $2^{\expo{\theta}} \otimes \gamma$ is subnormal, $2^{\expo{\theta}} \otimes \gamma$
is exact. Hence we have
\begin{align*}
    |2^{\expo{\theta}} \otimes \gamma - 2^{\expo{\theta}} \times \gamma| \le \frac{1}{2}\fmin.
\end{align*}

For $x = \sigma(\eta) - 2^{\expo{\zeta}-\expo{\theta}}$, we have
\begin{align*}
&\left| \round{ 2^{\expo{\theta}} \otimes \sigma(\eta ) - 2^{\expo{\theta}} \otimes x  }_\fpq \right| \le  \left| \round{ 2^{\expo{\theta}} \times \sigma(\eta ) - 2^{\expo{\theta}} \times x +\fmin   }_\fpq \right|  \\
    & \le \left|\round{\lrp{x-\sigma\lrp{\eta}}\times 2^{\expo{\theta}} + \fmin}_{\fpq}\right|
    \le \round{2^{\expo{\zeta}} +\fmin}_{\fpq}\le 1.1\times 2^{\expo{\zeta}}.
\end{align*}
Therefore,
\begin{align*}
h(x) =  \left( 2^{\expo{\theta}} \otimes  x  \right) \oplus \left(  -2^{\expo{\theta}} \otimes \sigma(\eta)  \right) \ge -1.1 \times 2^{\expo{\zeta}}.
\end{align*}

For $x = \sigma\lrp{\eta}$, we have
\begin{equation*}
    h\left( \sigma(\eta) \right) =\left( 2^{\expo{\theta}} \otimes  \sigma\lrp{\eta}  \right) \oplus \left(  -2^{\expo{\theta}} \otimes \sigma(\eta)  \right) = 0.
\end{equation*}
For $x = \sigma\lrp{\eta^+}$, by \cref{lem:determine_theta}, since $\left|\sigma\lrp{\eta^+}- \sigma\lrp{\eta}\right|\ge 2^{e_0}$ and $-e_0-\mbit+\emin+1=\expo{\theta} \in \fpq$, we have
\begin{equation*}
    h\left( \sigma(\eta^+) \right) = \left( 2^{\expo{\theta}}  \otimes \sigma\lrp{\eta^+} \right) \oplus \left(  2^{\expo{\theta}} \otimes - \sigma\lrp{\eta} \right) \ge \fmin.
\end{equation*}

For $x = \sigma(\eta^+) + 2^{\expo{\zeta}-\expo{\theta}}$, similar to the case of $x = \sigma(\eta) - 2^{\expo{\zeta}-\expo{\theta}}$, we have
\begin{align*}
h(x) =  \left( 2^{\expo{\theta}} \otimes  x  \right) \oplus \left(  -2^{\expo{\theta}} \otimes \sigma(\eta)  \right) \le 1.1 \times 2^{\expo{\zeta}}.
\end{align*}
Therefore, due to \cref{eq:hxcase1}, \cref{eq:temp_interval1,eq:temp_interval2}  hold. \\




Additionally, we need to show \cref{eq:contraction2_further1,eq:contraction2_further2}. \\
For $x > \sigma(\eta)+2^{\expo{\zeta}-\expo{\theta}}$, we have
\begin{align*}
    \left(\theta \otimes x  \right) \oplus \left( - 2^{\expo{\theta}} \otimes  \sigma\lrp{\eta} \right)
    &= \round{ \theta \otimes x  - 2^{\expo{\theta}} \otimes \sigma\lrp{\eta} }_{\fpq}
    \ge \round{ x\times 2^{\expo{\theta}} -\sigma\lrp{\eta}\times 2^{\expo{\theta}} -\fmin}_{\fpq}
    \\
    &= \round{\lrp{x-\sigma\lrp{\eta}}\times 2^{\expo{\theta}} - \fmin }_{\fpq} \\
    &\ge \round{\lrp{\left( \sigma(\eta)+2^{\expo{\zeta}-\expo{\theta}}\right)^+ -\sigma\lrp{\eta}}\times 2^{\expo{\theta}} - \fmin }_{\fpq} \ge 2^{\expo{\zeta}}.
\end{align*}
Hence, by \cref{eq:gxxop2ezp}, we have
\begin{equation*}
    g \left( \left( 2^{\expo{\theta}} \otimes x \right) \oplus \left( 2^{\expo{\theta}} \otimes \sigma(\eta) \right) \right)
    \le \round{\lrp{x-\sigma\lrp{\eta}}\times 2^{\expo{\theta}} + \fmin }_{\fpq}  \oplus \lrp{2^{\expo{\zeta}}}^+.
\end{equation*}
Therefore,
\begin{align*}
    f(x) &\le \round{\lrp{x-\sigma\lrp{\eta}}\times 2^{\expo{\theta}} + \fmin }_{\fpq} \oplus \lrp{2^{\expo{\zeta}}}^+ \oplus
    \left( \beta\otimes K_{\sigma} \right) \oplus \eta \\
    &\le \round{\lrp{x-\sigma\lrp{\eta}}\times 2^{\expo{\theta}} + \fmin }_{\fpq} \oplus \lrp{2^{\expo{\zeta}}}^+\oplus \left( \frac{5}{4}\times 2^{\expo{\zeta}-\mbit} \right)\oplus \eta
    \\ &\le \round{\lrp{x-\sigma\lrp{\eta^+}}\times 2^{\expo{\theta}} + 2^{\expo{\theta}+e_0+1} +\fmin }_{\fpq} \oplus \lrp{2^{\expo{\zeta}}}^+\oplus \left( \frac{5}{4}\times 2^{\expo{\zeta}-\mbit} \right) \oplus \eta \\
    &\le \round{\lrp{x-\sigma\lrp{\eta^+}}\times 2^{\expo{\theta}} + 2^{\expo{\zeta}} }_{\fpq} \oplus \lrp{2^{\expo{\zeta}}}^+\oplus \left( \frac{5}{4}\times 2^{\expo{\zeta}-\mbit} \right) \oplus \eta.
\end{align*}
where we use $2^{\expo{\zeta}} \ge 8\fmin$.

If $\lrp{x-\sigma\lrp{\eta^+}}\times 2^{\expo{\theta}}\le 2^{\expo{\zeta}-1}$, since
\begin{align*}
    \round{\lrp{x-\sigma\lrp{\eta^+}}\times 2^{\expo{\theta}} + 2^{\expo{\zeta}}+ \fmin }_{\fpq} \oplus \lrp{2^{\expo{\zeta}}}^+\oplus \left( \frac{5}{4}\times 2^{\expo{\zeta}-\mbit} \right)  < 1.1\times 2^{\expo{\zeta}+1},
\end{align*}
we have $f(x)\le \eta^+$. \\
If $\lrp{x-\sigma\lrp{\eta^+}}\times 2^{\expo{\theta}}> 2^{\expo{\zeta}-1}$, there exist $k\in \bbN$ such that  $k\times 2^{\expo{\zeta}-1}\le \lrp{x-\sigma\lrp{\eta^+}}\times 2^{\expo{\theta}}< (k+1)\times 2^{\expo{\zeta}-1}$. Then,
\begin{align*}
    f(x)    &\le \round{\lrp{x-\sigma\lrp{\eta^+}}\times 2^{\expo{\theta}} + 2^{\expo{\zeta}} }_{\fpq} \oplus \lrp{2^{\expo{\zeta}}}^+\oplus \left( \frac{5}{4}\times 2^{\expo{\zeta}-\mbit} \right)\oplus \eta \\
     &\le \round{\lrp{x-\sigma\lrp{\eta^+}}\times 2^{\expo{\theta}} + 2^{\expo{\zeta+1}} + 2^{\expo{\zeta}-\mbit} }_{\fpq} \oplus \left( \frac{5}{4}\times 2^{\expo{\zeta}-\mbit} \right)\oplus \eta \\
    &\le \eta^+ + \ceilZ{\frac{k+1}{4}} \times 2^{\expo{\zeta}+1}
    \le \eta^+ + 4\times \lrp{ x - \sigma\lrp{\eta^+}}\times 2^{\expo{\theta}}.
\end{align*}
Similarly, for $x <  \sigma(\eta)-2^{\expo{\zeta}-\expo{\theta}}$ we have
\begin{equation*}
    \eta - f(x)\le \lrp{x-\sigma\lrp{\eta}}\times 4 \theta.
\end{equation*}

\paragraph{\fbox{\bf Case 2: $\sigma(\eta)> \sigma(\eta^+)$.}}~

In this case, we need to show
\begin{align}
\begin{cases}
    h \left( \sigma(\eta) - 2^{\expo{\zeta}-\expo{\theta}} \right) \le 1.1 \times 2^{\expo{\zeta}} \\
    h\left( \sigma(\eta^+) \right)\ge \fmin \\
   h\left( \sigma(\eta)  \right) \le 0 \\
    h\left( \sigma(\eta^+) + 2^{\expo{\zeta}-\expo{\theta}} \right) \ge -1.1 \times 2^{\expo{\zeta}}.
    \end{cases}. \label{eq:hxcase2}
\end{align}
By similar arguments to \textbf{Case 1}, we can show \cref{eq:hxcase2} and \cref{eq:contraction2_further3,eq:contraction2_further4}.
\end{proof}

\subsection{Technical Lemmas for Lemma~\ref{lem:sigmaetatoc1-2}}
\label{subsec:techlemma_for_sigmaetatoc1-2}

This subsection presents technical lemmas for the proof of \cref{lem:sigmaetatoc1-2} (\cref{sec:pflem:sigmaetatoc1-2}).

\begin{lemma}\label{lem:onebit_difference}
    Let $x \in \fpq$ be normal. Then for any $c \in \fpq$ such that  $\expo{c} \ge \emin$ and $ |\expo{x}-\expo{c}| \le \emax$, there exists $w \in \fpq$ such that
    \begin{align*}
        w \otimes x = c \pm 2^{-\mbit+\expo{c}}.
    \end{align*}
\end{lemma}
\begin{proof}
    Without loss of generality, we assume $x,c > 0$ and $\expo{x}=1$. For any normal $a = 1.\mant{a,1}\dots\mant{a,\mbit}$, we have
    \begin{align*}
        (a^+ \otimes x ) - (a \otimes x )=  \round{a^+ \otimes x} - \round{a \otimes x} =  2^{-\mbit} \; \text{or} \; 2^{1-\mbit} .
    \end{align*}
since $ (a^+ \times x) - (a \times x) = 2^{-\mbit} \times x < 2^{1-\mbit}$. Since the gap is $2^{-\mbit}$ or $2^{1-\mbit}$, we can pick $w \in \fpq$ such that $w \otimes x = c \pm 2^{-\mbit+\expo{c}}$.
\end{proof}

\begin{lemma}[Floating-point distributive  law]\label{lem:distribution_law}
Let $a,b,c \in \fpq$ be normal. Then we have
    \begin{align*}
        a \otimes ( b \oplus c) = (a \otimes b) + (a \otimes c)  + (C \times 2^{-\mbit-1}),
    \end{align*}
where $|C| \le  |a| \left( (2+2^{-\mbit-1})|b+c|+|b|+|c|  \right)$.
\end{lemma}
\begin{proof}
Since
\begin{align*}
    b \oplus c  = (b+c)(1 + \delta_1), \quad |\delta_1| \le  \feps = 2^{-\mbit-1} .
\end{align*}
    We have
\begin{align*}
   a \otimes (b \oplus c) &=   \round{a \times \round{b+c}} = \left( a \times (b+c) \right)(1 + \delta_1)(1 + \delta_2) \\
    &= \left( ab+ac \right)(1 + \delta_1)(1 + \delta_2),
   \quad |\delta_1|,|\delta_2| \le  \feps.
\end{align*}
Since
\begin{align*}
    ( a \otimes b) + (a \otimes c) = ab(1+\delta_3) + ac(1+\delta_4) = a(b+c) + a(b \delta_3 + c \delta_4), \quad |\delta_3|,|\delta_4| \le  \feps
\end{align*}
the difference between $a \otimes ( b \oplus c)$ and $(a \otimes b) + (a \otimes c)$ is
\begin{align*}
   a \otimes ( b \oplus c) - \left( ( a \otimes b) + (a \otimes c) \right) &= (ab+ac) \left( \delta_1 + \delta_2 + \delta_1\delta_2\right) - a(b \delta_3 + c \delta_4)
\end{align*}
Therefore,
\begin{align*}
    | a \otimes ( b \oplus c) - \left( ( a \otimes b) + (a \otimes c) \right) | \le |a| \left( |b + c|(2 \feps + \feps^2) + (|b| + |c|)\feps  \right).
\end{align*}
\end{proof}

\begin{lemma}\label{lem:residue_control}
Let $ x, y \in \fpq$. Suppose $x=0$ or  $\expo{x}\le -3 -2\mbit + \expo{y}$ where $\expo{y} \ge 1+\emin$. Then
for $a \in \{ -2^{-\mbit+\expo{y}} , 2^{-\mbit+\expo{y}} \}$, there exist $\alpha_1,\dots,\alpha_3 \in \fpq$ such that
\begin{align*}
    x \oplus \bigoplus_{i=1}^3 \alpha_i &= 0, \\
    y \oplus \bigoplus_{i=1}^3 \alpha_i &= y+a.
\end{align*}
\end{lemma}
\begin{proof}
Without loss of generality, we assume $y>0$. \\
\textbf{Case (1) $\quad$}  $\mant{y,\mbit}=0$. \\
If $a =-2^{-\mbit+\expo{y}}$, let $\alpha_1  =\alpha_2 = 2^{-1-\mbit+\expo{y}}$,  $\alpha_3 = -2^{-\mbit+\expo{y}}$. \\
If $a =2^{-\mbit+\expo{y}}$, let $\alpha_1  = \alpha_2= -2^{-1-\mbit+\expo{y}}$, $\alpha_3 = 2^{-\mbit+\expo{y}}$. \\
\textbf{Case (2) $\quad$}  $\mant{y,\mbit}=1$. \\
If $a =-2^{-\mbit+\expo{y}}$, let $\alpha_1 = -2^{-\mbit+\expo{y}}$, $\alpha_2 =\alpha_3 = 2^{-1-\mbit+\expo{y}}$. \\
If $a =2^{-\mbit+\expo{y}}$, let $\alpha_1 = 2^{-\mbit+\expo{y}}$, $\alpha_2 =\alpha_3   = -2^{-1-\mbit+\expo{y}}$.

\end{proof}

\begin{lemma}\label{lem:residue_control_111}
Let $ x, y \in \fpq$. Suppose $x=0$ or  $\expo{x}\le -4 -2\mbit + \expo{y}$ where $\expo{y} \ge 1+\emin$. Then
for $a \in \{ -2^{-\mbit+\expo{y}} , 2^{-\mbit+\expo{y}} \}$, there exist $\alpha_1,\dots,\alpha_5 \in \fpq$ of the form $\alpha_i = 1.\underbrace{1\dots 1}_{\mbit \text{ times }} \times 2^{\expo{\alpha_i}} $ such that
\begin{align*}
    x \oplus \bigoplus_{i=1}^5 \alpha_i &= 0, \\
    y \oplus \bigoplus_{i=1}^5 \alpha_i &= y+a.
\end{align*}
\end{lemma}
\begin{proof}
Without loss of generality, we assume $y>0$. \\
First, note that
\begin{align*}
    1.\underbrace{1\dots 1}_{\mbit \text{ times }} \oplus 1.\underbrace{1\dots 1}_{\mbit \text{ times }} &= 11.\underbrace{1\dots 1}_{\mbit -1 \text{ times }}, \\
    11.\underbrace{1\dots 1}_{\mbit -1 \text{ times }} \oplus 1.\underbrace{1\dots 1}_{\mbit \text{ times }} &=
    \round{101.\underbrace{1\dots 1}_{\mbit -2 \text{ times }}01 }=
    101.\underbrace{1\dots 1}_{\mbit -2 \text{ times }}  , \\
    101.\underbrace{1\dots 1}_{\mbit -2 \text{ times }} \oplus 1.\underbrace{1\dots 1}_{\mbit \text{ times }} &= \round{111.\underbrace{1\dots 1}_{\mbit -3 \text{ times }}011} =111.\underbrace{1\dots 1}_{\mbit -2 \text{ times }}.
\end{align*}
Hence $\bigoplus_{i=1}^4  1.\underbrace{1\dots 1}_{\mbit \text{ times }} = 1.\underbrace{1\dots 1}_{\mbit \text{ times }} \times 2^2$.
\\
Therefore, if $a = \pm 2^{-\mbit +\expo{y}}$, let
\begin{align*}
\alpha_1 =  \alpha_2= \mp 1.\underbrace{1\dots 1}_{\mbit \text{ times }} \times 2^{-2-\mbit+\expo{y}}, \alpha_3 = \pm 1.\underbrace{1\dots 1}_{\mbit \text{ times }} \times 2^{-1-\mbit+\expo{y}}.
\end{align*}
Then we have
\begin{align*}
    x \oplus \bigoplus_{i=1}^3 \alpha_i &= 0, \quad y \oplus \bigoplus_{i=1}^3 \alpha_i = y+a.
\end{align*}
\end{proof}



\begin{lemma}\label{lem:determine_theta}
    Consider normal floating-point numbers $\gamma_1, \gamma_2\in \fpq$ with $\expo{\gamma_1} \ge \expo{\gamma_2} \ge \emin+1$. Suppose an integer $ e_0 \in \bbZ$ satisfies the following:
    \begin{equation*}
         2^{e_0}  \le |\gamma_1 - \gamma_2| < 2^{e_0+1}.
    \end{equation*}
Define $\expo{\theta}$ as
\begin{equation*}
    \expo{\theta}\defeq \max\lrp{\emin-\mbit, -e_0+\emin-\mbit+1}.
\end{equation*}
Then, we have
    \begin{equation*}
        \left( 2^{\expo{\theta}} \otimes \gamma_1 \right) \oplus \left( - 2^{\expo{\theta}} \otimes \gamma_2 \right) \neq 0.
    \end{equation*}
\end{lemma}
\begin{proof}
For any $\gamma \in \fpq$, unless $2^{\expo{\theta}} \otimes \gamma$ is subnormal , $2^{\expo{\theta}} \otimes \gamma$ is exact. Hence,
 \begin{align*}
        \left| 2^{\expo{\theta}} \otimes \gamma_1 -  2^{\expo{\theta}}\times \gamma_1  \right| & \le  \frac{1}{2} \fmin, \\
        \left| 2^{\expo{\theta}} \otimes \gamma_2 -  2^{\expo{\theta}}\times \gamma_2  \right|  &\le \frac{1}{2} \fmin.
    \end{align*}
Since
$e_0 \ge \expo{\gamma_1} - \mbit -1 $ and $e_0 \ge \expo{\gamma_2} - \mbit -1 $, we have
\begin{align*}
     \expo{\theta} + \expo{\gamma_1} = -e_0 + \emin - \mbit + 1 + \expo{\gamma_1} \le \emin
\end{align*}

Therefore following inequalities hold:
\begin{align*}
    &        \left| 2^{\expo{\theta}} \otimes \gamma_1 -  2^{\expo{\theta}} \otimes \gamma_2  \right| \\
    &\ge \left| 2^{\expo{\theta}} \times \gamma_1 -  2^{\expo{\theta}} \times \gamma_2  \right| -\left| 2^{\expo{\theta}} \otimes \gamma_1 -  2^{\expo{\theta}}\times \gamma_1  \right| - \left| 2^{\expo{\theta}} \otimes \gamma_2 -  2^{\expo{\theta}}\times \gamma_2  \right| \\
    &\ge 2^{\emin-\mbit+1} - \fmin = \fmin.
\end{align*}
Since $|  2^{\expo{\theta}} \otimes \gamma_1|, | 2^{\expo{\theta}} \otimes \gamma_2| \le 2^{1+\emin}$, their gap to adjacent number is $\fmin$. Therefore they are distinct.
\end{proof}

\subsection{Technical Lemma for Lemma~\ref{lem:eta_etaplus-2}}
\label{subsec:techlemma_for_eta_etaplus-2}

This subsection presents technical lemma for the proof of \cref{lem:eta_etaplus-2} (\cref{sec:pflem:eta_etaplus-2}).
\begin{lemma}\label{lem:special_case}
For $1 \le x < 1+2^{-1}$, we have $(1+2^{-\mbit} )\otimes x= x^+$. For  $ 1+2^{-1} \le x \le 2-2^{-1-\mbit}$, we have $(1+2^{-\mbit} )\otimes x= x^{++}$. For $x=2-2^{-\mbit}$ we have $(1+2^{-\mbit}) \otimes x= x^+$.
\end{lemma}
\begin{proof}
    $ (2^{-1}+2^{-1-\mbit}) \otimes x = \round{ 1 +2^{-1-\mbit} - 2^{-1-2\mbit}} = 1.$
\end{proof}

\subsection{Technical Lemmas for \S\ref{sec:pflem:proof_results}}
\label{subsec:techlemma_for_indc-to-iua}

This subsection presents technical lemma for the proofs in \cref{sec:pflem:proof_results}.

\begin{lemma}\label{lem:endbit_control}
    Let $K\in\fpq$ with $|K|\in [(1+2^{-\mbit+1})\times2^{-\mbit-2},1+2^{-2}-2^{-\mbit}]_{\fpq}$.
    Consider $\expo{\zeta}\in \bbZ$ such that $\emin -\mbit \le \expo{\zeta}\le \emax-\mbit$.
    Then, there exists $\gamma\in \fpq$ such that the following inequality holds:
    \begin{equation*}
      \frac{1}{2}\times 2^{\expo{\zeta}}   < \gamma\otimes K\le \frac{5}{4}\times 2^{\expo{\zeta}}.
    \end{equation*}
\end{lemma}
\begin{proof}
    Let $K$ be represented as
    \begin{equation*}
        K = \mant{K}\times 2^{\expo{K}}, \quad -M-2 \le \expo{K} \le 0.
    \end{equation*}
We consider the following cases.

\paragraph{\fbox{\bf Case 1: $\emax -\mbit-1 \le \expo{\zeta} \le  \emax -\mbit $.}}~

If $\mant{K} \in [1, 1+ 2^{-\mbit}]_\fpq$, we have $ -\mbit-1 \le \expo{K} \le 0$.
We define $\gamma$ as
\begin{align*}
    \gamma \defeq 2^{\expo{\zeta}-\expo{K}-1}.
\end{align*}
Since $\expo{\zeta}-\expo{K}-1 \le \emax -\mbit -(\mbit-1)-1 = \emax$, we have $\gamma \in \fpq$.
Then,
\begin{align*}
   2^{\expo{\zeta}-1} \le \gamma \times K \le (1+2^{-\mbit}) \times 2^{\expo{\zeta}-1}.
\end{align*}
If $\mant{K} \in [1+2^{-\mbit+1}, 2)_\fpq$,
we define $\gamma$ as
\begin{equation*}
    \gamma \defeq (2- 2^{-\mbit}) \times  2^{\expo{\zeta} - \expo{K}-2}.
\end{equation*}
Since $ \expo{\zeta} - \expo{K} - 2\le \emax -\mbit -(-\mbit-2) - 2= \emax$, we have $\gamma \in \fpq$. \\
Because
\begin{align*}
   (2-2^{-\mbit}) \times (1 + 2^{-\mbit + 1 }) = 2 +3 \times 2^{-\mbit} -2^{-2\mbit+1} &> 2, \\
(2-2^{-\mbit}) \times (2-2^{-\mbit})  = 2^2 - 2^{2-\mbit}+2^{-2\mbit}  &< 4,
\end{align*}
we have
\begin{align*}
     \frac{1}{2} \times 2^{\expo{\zeta}} < \gamma \otimes K <  2^{\expo{\zeta}}.
\end{align*}

\paragraph{\fbox{\bf Case 2: $\emin -\mbit + 2 \le \expo{\zeta} \le \emax -\mbit -2 $.}}~

    If $\mant{K}\in \left[1, 1 + 2^{-2}\right]_{\fpq}$, define $\gamma$ as
    \begin{equation*}
        \gamma \defeq 2^{\expo{\zeta} - \expo{K}}.
    \end{equation*}
As $ \emin -\mbit + 2  \le \expo{\zeta}-\expo{K} \le \emax $, we have $\gamma \in \fpq$.
Then, we have
\begin{equation*}
    2^{\expo{\zeta}}\le
    \gamma \otimes K = \round{\mant{K}\times 2^{\expo{\zeta}}}_{\fpq}
    \le \round{\frac{5}{4}\times 2^{\expo{\zeta}}}_{\fpq} = \frac{5}{4}\times 2^{\expo{\zeta}},
\end{equation*}
where the last equality is followed by $\expo{\zeta}\ge \emin-\mbit+2$. \\
    If $\mant{K}\in \left(1+2^{-2},2 \right)_{\fpq}$, define $\gamma$ as
    \begin{equation*}
        \gamma \defeq 2^{\expo{\zeta} - \expo{K}-1}.
    \end{equation*}
As $  \expo{\zeta}-\expo{K} - 1 \ge \emin -\mbit + 1$, we have $\gamma \in \fpq$.
Then, we have
\begin{equation*}
  \frac{1}{2}\times 2^{\expo{\zeta}} <  \gamma \otimes K = \round{\mant{K}\times 2^{\expo{\zeta} -1}}_{\fpq} \le 2^{\expo{\zeta} },
\end{equation*}
 where the first inequality is followed from $\expo{\zeta}-1\ge \emin-\mbit+1$ and $\mant{K} > 1+ 2^{-2}$.

\paragraph{\fbox{\bf Case 3: $\expo{\zeta} = \emin -\mbit +1$.}}~

 If $\mant{K}\in [1, \frac{5}{4} )$, define $\gamma$ as
 \begin{equation*}
     \gamma\defeq 2^{\expo{\zeta}-\expo{K}}.
 \end{equation*}
Since $  \expo{\zeta}-\expo{K} \ge \emin -\mbit + 1$, we have $\gamma \in \fpq$.
     Then,
     \begin{equation*}
         \gamma\otimes K =  \round{\mant{K}\times 2^{\emin -\mbit +1}}_{\fpq} = 2\fmin = 2^{\expo{\zeta}},
     \end{equation*}
since $ 2 \fmin \le \mant{K}\times 2^{\emin -\mbit +1} < \frac{5}{2} \fmin $. \\
 If  $\mant{K} \in \left[\frac{5}{4}, \frac{5}{3}\right) $, define $\gamma$ as
     \begin{equation*}
        \gamma\defeq  (1+2^{-1})\times 2^{\expo{\zeta}-\expo{K}-1}.
     \end{equation*}
As $\expo{K}\le -1$, $  \expo{\zeta}-\expo{K} -1 \ge \emin -\mbit +1$, and we have $\gamma \in \fpq$.
Then,
     \begin{equation*}
         \gamma\otimes K =  \round{\mant{K}\times (1+2^{-1})\times 2^{\emin-\mbit}}_{\fpq}
         =  \round{\mant{K}\times (1+2^{-1}) \times \fmin}_{\fpq} = 2\fmin = 2^{\expo{\zeta}},
     \end{equation*}
since $ \frac{3}{2} \fmin < \frac{15}{8}\fmin \le \mant{K}\times (1+2^{-1}) \times \fmin < \frac{5}{2} \fmin $. \\

If  $\mant{K} \in \lrp{\frac{5}{3}, 2} $, define $\gamma$ as
 \begin{equation*}
    \gamma\defeq  2^{\expo{\zeta}-\expo{K}-1}.
 \end{equation*}
Since $  \expo{\zeta}-\expo{K} \ge \emin -\mbit $, we have $\gamma \in \fpq$.
 Then,
     \begin{equation*}
         \gamma\otimes K =  \round{\mant{K}\times  2^{\emin-\mbit}}_{\fpq}
         =  \round{\mant{K}\times  \fmin}_{\fpq} = 2\fmin = 2^{\expo{\zeta}},
     \end{equation*}
since $ \frac{3}{2} \fmin < \frac{5}{3}\fmin \le \mant{K} \times \fmin < 2 \fmin $. \\

\paragraph{\fbox{\bf Case 4: $\expo{\zeta} = \emin -\mbit$.}}~

 If $\mant{K}\in [1, \frac{3}{2} )$, define $\gamma$ as
    \begin{equation*}
        \gamma\defeq  2^{\expo{\zeta}-\expo{K}}.
    \end{equation*}
Since $  \expo{\zeta}-\expo{K} \ge \emin -\mbit $, we have $\gamma \in \fpq$.
      Then,
     \begin{equation*}
         \gamma\otimes K =  \round{\mant{K}\times  2^{\emin-\mbit}}_{\fpq}
         =  \round{\mant{K}\times \fmin}_{\fpq} = \fmin= 2^{\expo{\zeta}},
     \end{equation*}
since $ \fmin \le \mant{K} \times \fmin < \frac{3}{2}\fmin $. \\
If  $\mant{K} \in (\frac{3}{2}, 2) $,  we have $\expo{K}\le -1$ by the assumption.
Define $\gamma$ as
    \begin{equation*}
        \gamma\defeq  2^{\expo{\zeta}-\expo{K}-1}.
    \end{equation*}
Since $  \expo{\zeta}-\expo{K} -1 \ge \emin -\mbit $, we have $\gamma \in \fpq$.
  Then,
     \begin{equation*}
         \gamma\otimes K =  \round{\mant{K}\times  2^{\emin-\mbit-1}}_{\fpq}
         =  \round{\frac{1}{2} \times \mant{K}\times  \fmin}_{\fpq} = \fmin = 2^{\expo{\zeta}},
     \end{equation*}
since $ \frac{1}{2}\fmin < \frac{3}{4}\fmin \le \frac{1}{2} \times \mant{K} \times \fmin < \fmin $. This completes the proof.
\end{proof}

\begin{lemma}\label{lem:approx_sum}
Let $0=z_0<z_1<\cdots<z_{(|\fpq|-1)/2}=\fmax<z_{(|\fpq|+1)/2}=\infty$ be all non-negative floats in $\fpq$.
Then, for any $j\in[(|\fpq|+1)/2]\cup\{0\}$ and $x_i\in(2^{-1}\times(z_i-z_{i-1}),(1+2^{-1})\times(z_i-z_{i-1}))_{\fpq}$, it holds that
\begin{align*}
\bigoplus_{i=1}^jx_i=z_j.
\end{align*}
\end{lemma}
\begin{proof}
Since it is obvious for $j=0$ we consider $j \ge 1$, and we use the mathematical induction on $j$.

\paragraph{\fbox{\bf Base step: $j=1$.}}~

For the base step, note that $z_1 = \fmin$. Then we have
\begin{equation*}
    x_1 \in(2^{-1}\times(z_1-z_{0}),(1+2^{-1})\times(z_1-z_{0}))_{\fpq}=
     (\frac{1}{2} \times \fmin , \frac{3}{2} \times \fmin)_\fpq = \{ \fmin \} = \{z_1\}.
\end{equation*}

\paragraph{\fbox{\bf Induction step.}}~

Assume that $\bigoplus_{i=1}^j x_i =z_j$  (inductive hypothesis). We write $z_j$ as $ \mant{z_j} \times 2^{\expo{z_j}}$.
We consider the following cases.

\paragraph{\underline{\bf Case 1: $2^{\expo{z_j}} \le \emin + 1$.}}~

In this case, $z_{j+1} - z_j = \fmin$.
Consider the case of $j+1$ as follows:
\begin{align*}
    \bigoplus_{i=1}^{j+1}x_i
    &= \lrp{\bigoplus_{i=1}^{j}x_i} \oplus x_{j+1} = z_j \oplus x_{j+1}= z_{j+1} ,
\end{align*}
where the last equality follows from
\begin{align*}
    x_{j+1} &\in \left(2^{-1}\times(z_{j+1}-z_{j}),(1+2^{-1})\times(z_{j+1}-z_{j}) \right)_{\fpq} = (\frac{1}{2} \times \fmin , \frac{3}{2} \times \fmin)_\fpq = \{ \fmin \} .
\end{align*}

\paragraph{\underline{\bf Case 2: $2^{\expo{z_j}} \ge \emin + 2$.}}~

In this case, $z_{j+1} - z_j = 2^{-\mbit + \expo{z_j}}$ Consider the case of $j+1$ as follows:
\begin{align*}
    \bigoplus_{i=1}^{j+1}x_i
    &= \lrp{\bigoplus_{i=1}^{j}x_i} \oplus x_{j+1} = z_j \oplus x_{j+1}= z_{j+1} ,
\end{align*}
where the last equality follows from
\begin{align*}
    x_{j+1} &\in \left(2^{-1}\times(z_{j+1}-z_{j}),(1+2^{-1})\times(z_{j+1}-z_{j}) \right)_{\fpq} \\
    &= \left(2^{-\mbit-1 + \expo{x_j}},(1+2^{-1}) \times 2^{-\mbit-1 + \expo{x_j} }\right)_{\fpq} .
\end{align*}
This completes the proof.
\end{proof}

\begin{lemma}\label{lem:sign}
For any $\eta\in[-2^3,2^3)_{\fpq}$, $x\in[(1+2^{-p+1})\times2^{-p-2},1+2^{-1}-2^{-p}]_{\fpq}$, and $n\in[2^p]$, there exist $k\in\bbN$ and $\alpha,\beta,z_1,\dots,z_k,\theta_1,\dots,\theta_k\in\fpq$ such that
\begin{align*}
&\lrp{\bigoplus_{i=1}^n \alpha\otimes x}\oplus (\theta_1\otimes\sigma(z_1))\oplus\cdots\oplus(\theta_k\otimes\sigma(z_k))\oplus\beta\in[\eta^+,\infty)_{\fpq},\\
&\lrp{\bigoplus_{i=1}^{n-1} \alpha\otimes x}\oplus (\theta_1\otimes\sigma(z_1))\oplus\cdots\oplus(\theta_k\otimes\sigma(z_k))\oplus\beta\in(-\infty,\eta]_{\fpq}.
\end{align*}
\end{lemma}
\begin{proof}
    If $\eta$ is normal, by \cref{lem:endbit_control}, there exists $\alpha\in \fpq$ such that the following inequality holds:
if $\eta>0 $,
    \begin{equation*}
      \frac{1}{2}\times 2^{\expo{\eta}-\mbit}   \le \alpha\otimes x< \frac{5}{4}\times 2^{\expo{\eta}-\mbit}.
    \end{equation*}
if $\eta<0 $,
    \begin{equation*}
      \frac{1}{2}\times 2^{\expo{\eta}-\mbit-1}   \le \alpha\otimes x< \frac{5}{4}\times 2^{\expo{\eta}-\mbit-1}.
    \end{equation*}
Then, for $n=1$,
\begin{equation}
   \eta \oplus \lrp{ \sum_{i=1}^n \alpha \otimes x } = \eta \oplus \lrp{  \alpha \otimes x } = \eta^+.
\end{equation}
If $\eta$ is subnormal, consider $\alpha\in \fpq$ such that $\alpha\otimes x = \fmin$.
Then, for $n=1$,
\begin{equation}
   \eta \oplus \lrp{ \sum_{i=1}^n \alpha \otimes x } = \eta \oplus \fmin = \eta^+.
\end{equation}
This completes the proof.
\end{proof}

\end{document}